\pdfoutput=1
\documentclass[twoside,11pt]{article}

\usepackage[utf8]{inputenc}
\usepackage{amsmath,amsthm,amssymb}
\let\amsproof\proof
\let\endamsproof\endproof
\usepackage[preprint]{jmlr2e}
\let\proof\amsproof
\let\endproof\endamsproof

\setkeys{Gin}{width=\linewidth}
\usepackage{mathtools}
\usepackage{bm}
\usepackage{booktabs}
\usepackage{multirow}
\usepackage{xcolor}
\usepackage{graphicx}
\usepackage{algorithm}
\usepackage{algpseudocode}
\usepackage{url}
\usepackage[capitalise,noabbrev]{cleveref}

\usepackage{aliascnt}
\makeatletter
\@for\p@thmname:=theorem,lemma,proposition,corollary,definition,remark,conjecture\do{%
  \expandafter\let\csname\p@thmname\endcsname\relax
  \expandafter\let\csname end\p@thmname\endcsname\relax
  \expandafter\let\csname c@\p@thmname\endcsname\relax}
\makeatother
\newcommand{\sharedtheorem}[2]{%
  \newaliascnt{#1}{theorem}\newtheorem{#1}[#1]{#2}\aliascntresetthe{#1}%
  \crefname{#1}{#2}{#2s}\Crefname{#1}{#2}{#2s}}
\theoremstyle{plain}
\newtheorem{theorem}{Theorem}[section]
\sharedtheorem{lemma}{Lemma}
\sharedtheorem{proposition}{Proposition}
\sharedtheorem{corollary}{Corollary}
\theoremstyle{definition}
\sharedtheorem{definition}{Definition}
\sharedtheorem{assumption}{Assumption}
\sharedtheorem{remark}{Remark}
\sharedtheorem{conjecture}{Conjecture}

\newcommand{\R}{\mathbb{R}}
\newcommand{\E}{\mathbb{E}}

\newcommand{\Var}{\mathrm{Var}}
\newcommand{\1}{\mathbf{1}}
\newcommand{\KL}{\mathrm{KL}}
\newcommand{\kl}{\mathrm{kl}}
\newcommand{\Kstar}{K^{\star}}
\newcommand{\Kscal}{K_{\mathrm{scal}}}
\newcommand{\Kloc}{K_{\mathrm{loc}}}
\newcommand{\Qs}{\mathcal{Q}}
\newcommand{\Qn}{\mathcal{Q}_0}
\newcommand{\Ys}{\mathcal{Y}}
\newcommand{\Es}{\mathcal{E}}
\newcommand{\J}{\mathcal{J}}
\newcommand{\supp}{\operatorname{supp}}
\newcommand{\range}{\operatorname{range}}
\newcommand{\rank}{\operatorname{rank}}
\newcommand{\mbar}{\bar m}
\newcommand{\sigeff}{\sigma^2_{\mathrm{eff}}}
\newcommand{\sigis}{\sigma^2_{\mathrm{IS}}}
\DeclareMathOperator*{\argmin}{arg\,min}
\DeclareMathOperator*{\argmax}{arg\,max}

\title{The Certification Exponent: Instance-Optimal Sequential Certification of Policy Improvement from History-Dependent Logs}

\author{\name Rui Wang \email wangrui@microsoft.com \\
       \addr Microsoft \\
       Vancouver, BC, Canada}

\ShortHeadings{The Certification Exponent}{Wang}
\firstpageno{1}

\begin{document}

\maketitle

\begin{abstract}
Before a new policy replaces a deployed one, we often want a certificate, computed from
logged episodes, that it is better. We ask how many episodes such a certificate needs
when the logging policy is known but history-dependent, and the error probability must
stay below $\delta$ in every environment consistent with the logs. The answer is an
exponent $\Kstar$: the smallest KL divergence from the logged law to a law with the same
logged action probabilities under which the target does not improve. No sound certifier
stops in fewer than $\log(1/\delta)/\Kstar$ expected episodes, at any $\delta$, and
whenever $0<\Kstar<\infty$ a plug-in Kelly certifier attains this rate to first order as
$\delta\to0$. Coverage, computed by one backward induction,
decides whether certification is impossible ($\Kstar=0$), possible at a finite rate, or
possible in a constant number of episodes ($\Kstar=\infty$). By convex duality, the
optimal certificate bets on the importance-weighted advantage and also places fair bets
on the logged actions, so doubly robust control variates act as hedges rather than
estimators. As a consequence, no choice of off-policy estimator changes
$\Kstar$, a coverage ceiling explains the curse of memory, and structural assumptions can
be priced in exponent units. A cushion certifier extends the approach to trees too large
to enumerate. In experiments the plug-in's median stopping time is $1.13$--$1.43\times$
the lower bound on memory tasks, $1.13$--$1.51\times$ on Tiger ($H\le5$) and
$1.17$--$1.64\times$ on random trees. The best scalar betting baselines need up to
$2.65\times$ (memory) and $4.1\times$ (Tiger) the plug-in's episodes; only on a few
shallow random trees is a doubly robust bet faster, by at most 18\%.
\end{abstract}

\begin{keywords}
  off-policy evaluation, sequential testing, testing by betting, e-values,
  instance optimality, history-dependent policies
\end{keywords}

\section{Introduction}
\label{sec:intro}

Suppose a team has trained a new dialogue policy $\pi$ offline and holds a log of
conversations produced by the deployed policy $\mu$. The decision to replace $\mu$ by
$\pi$ is binary, and a wrong replacement is costly. The team therefore wants a
\emph{sound} rule: if $\pi$ is not better than $\mu$, it is deployed with probability at
most $\delta$. Logged episodes arrive one at a time, so the rule should be sequential and
stop as soon as the evidence suffices. This is the goal of high-confidence off-policy
evaluation and safe policy improvement
\citep{thomas2015hcope,thomas2015hcpi,laroche2019spibb}, and of gated deployment with
anytime-valid confidence sequences
\citep{karampatziakis2021offpolicy,waudbysmith2024anytime}. The practical question is how
many episodes such a rule must wait for.

Memory makes this question hard. Recurrent policies, dialogue agents and controllers for
partially observable environments choose actions based on the whole history. The
importance weight that reweights a logged episode towards $\pi$ is then a product of
per-step ratios, and it can be enormous on the few episodes that $\pi$ would have
produced. Recent lower bounds show that this cost is sometimes unavoidable: under
history-dependent logging, evaluation can require exponentially many episodes in the
horizon \citep{zhang2025tractability,jajoo2026hardness,luo2026revealing}. Worst-case
bounds, however, say little about a particular log. Some instances can be certified from
a handful of episodes, and others cannot be certified at all.

Existing certifiers compress each episode to a scalar, such as an importance-sampling
(IS), per-decision IS or doubly robust (DR) estimate of the advantage, and then apply a
concentration inequality \citep{thomas2015hcope} or a betting confidence sequence
\citep{waudbysmith2024estimating} to the resulting stream. Better estimators and tighter
inequalities shorten the wait. It is unclear, though, whether these refinements approach
a limit, and whether a certifier that looks at more than one scalar per episode could do
fundamentally better. This leaves a basic question open:
\emph{for a given instance, how many logged episodes does the best possible sound
certifier need, and what does it look like?} We answer it in the expected-stopping-time
sense of sequential testing \citep{garivier2016optimal}.

Our approach views certification as a betting game \citep{shafer2021testing,ramdas2023gametheoretic}.
The certifier starts with unit wealth, places a bet on each logged episode and certifies
once its wealth reaches $1/\delta$. The bets must be fair or unfavourable under every
environment in the null, so that by Ville's inequality wealth rarely reaches $1/\delta$
when $\pi$ is not better. The null is large: it contains every environment in which
$\pi$ does not improve on $\mu$, including environments that differ from the true one on
histories the log never reaches. Because $\mu$ is known, however, no environment can
change how the logged actions were chosen. The bettor therefore has two kinds of admissible
bets: a stake on the importance-weighted advantage $X$, which pays off when $\pi$ is
better, and fair bets on the logged actions themselves, at the odds $\mu$ assigned to
them. Fair bets earn nothing on average under any null
environment, but they can offset the losses of the advantage bet on bad outcomes, which
lets the bettor stake more while keeping its wealth nonnegative. We show that the optimal bettor combines both, and that its wealth grows at rate
$\Kstar$, the KL divergence from the logged law to the nearest null law that leaves the
logged action probabilities unchanged. The control variates of DR estimators are fair
bets of this kind, chosen heuristically. They therefore act as hedges rather than as
better estimators, which explains both why they help and why no estimator can beat the
optimal combination. How far the bettor can lever its stake is limited by coverage, that
is, by how much of $\pi$'s behaviour the log reveals. The largest safe stake is a
superhedging price, computed by backward induction over the tree. \Cref{fig:overview}
illustrates the setting and the geometry of $\Kstar$.

Informally, our main result is as follows. For every logged law $P$ under which $\pi$ improves on $\mu$, every $\delta$-sound
certifier needs at least $\log(1/\delta)/\Kstar(P)$ episodes in expectation, at every
$\delta$. A plug-in Kelly certifier attains this bound to first order as $\delta\to0$
when $0<\Kstar(P)<\infty$. Coverage decides whether $\Kstar(P)$ is zero, finite or
infinite.

\paragraph{Contributions.}
We model an episode as a path through a finite outcome tree whose action nodes are
controlled by the known logger $\mu$. The null hypothesis contains every environment on
the tree under which $\pi$ does not improve on $\mu$.
\begin{enumerate}
\item In \cref{sec:exponent} we define the certification exponent
  $\Kstar(P)=\inf\{\KL(P\|q):\allowbreak q \text{ action-neutral},\allowbreak \E_q X\le 0\}$, where
  $X=(W-1)R$ is the importance-weighted advantage of an episode. A trade-off inequality
  shows that any stopping time with finite mean under $P$ stops under a null $q$ with
  probability at least $\exp(-\E_P[\tau]\KL(P\|q))$. Hence every $\delta$-sound
  sequential certifier has $\E_P[\tau_\delta]\ge\log(1/\delta)/\Kstar$,
  non-asymptotically and without a power condition (\cref{thm:lower}). By duality,
  $\Kstar$ is the optimal growth rate of a bettor holding a claim $\lambda X$ who may
  also place fair bets on every logged action at odds $\mu$ (\cref{thm:duality}). The
  maximal leverage is a superhedging price computed by backward induction
  (\cref{prop:superhedge}).
\item \Cref{sec:achieve} gives matching certifiers. A predictable Kelly certifier
  attains $\E_P[\tau_\delta]=(1+o(1))\log(1/\delta)/\Kstar$ (\cref{thm:plugin}). When
  $\Kstar<\infty$ it also satisfies the non-asymptotic bound
  $\E_P[\tau_\delta]\le\log(1/\delta)/\Kstar+O(\sqrt{\log(1/\delta)\log\log(1/\delta)})$,
  with constants given in terms of $\Kstar$, the support of $P$ and the least-favourable
  law (\cref{thm:plugin-explicit}). The plug-in fits its bet to the maximum-likelihood
  \emph{logged-law projection} of the empirical law, which exploits the known $\mu$;
  without the projection it overfits and is $15\times$ slower in our experiments. When
  the null contains a full-support law, a universal-portfolio mixture has finite-sample
  tail and expectation bounds (\cref{thm:mixture}). Its exact form sharpens the
  expected-time remainder to $(d/\Kstar)\log\log(1/\delta)+C_P$
  (\cref{thm:mixture}(v)); with \cref{thm:lower}, this pins the first-order constant from
  both sides. Averaging the two wealths keeps both
  guarantees at a cost of $\log 2$ (\cref{cor:hedge}).
\item We show that coverage decides the regime (\cref{sec:coverage}). The advantage $\Delta$
  and the superhedging price $m$ place each instance in one of three regimes. If
  $\Kstar=0$, no $\delta$-sound certifier certifies with probability above $\delta$. The
  case $0<\Kstar<\infty$ is equivalent to the null containing a full-support law. If
  $\Kstar=\infty$, an arbitrage certifier is exactly sound and stops after $1/P(G>0)$
  episodes in expectation (\cref{prop:trichotomy}). Without positivity of $\mu$,
  certification is possible if and only if $\pi$ beats $\mu$ after every unlogged action
  of $\pi$ is assigned the worst reward, even when $V(\pi)>V(\mu)$
  (\cref{prop:positivity}). Non-asymptotically, $\Kstar$ lies between two multiples of
  $\Delta^2/2\sigeff$ (\cref{prop:sandwich}).
\item We derive several consequences of the exponent (\cref{sec:consequences}). Replacing $X$ by any
  DR or per-decision IS payoff leaves $\Kstar$ unchanged, and every sound certificate
  built on a per-episode statistic $S$ needs at least $\log(1/\delta)/K_S$ episodes in expectation,
  with $K_S\le\Kstar$ (\cref{thm:invariance}). As the advantage vanishes, $\Kstar$ tends
  to the semiparametric efficiency bound $\Delta^2/2\sigeff$, under regularity
  conditions on $P$ alone (\cref{thm:local}). Away from that limit the ratio
  $\Kstar/(\Delta^2/2\sigeff)$ ranges from $0.014$ to $19.4$ on random trees, so the
  efficiency bound can be off by more than an order of magnitude in either direction.
  The coverage ceiling $\Kstar\le P(A)\log\frac{V(\pi)-r_A}{V(\mu)-r_A}$, where $A$ is
  the event that $\pi$ could have produced the logged actions (\cref{thm:ceiling}),
  recovers the exponential rates of \citet{jajoo2026hardness} and
  \citet{luo2026revealing} for one-sided certification. Structural assumptions are
  linear constraints whose value $\Kstar_F/\Kstar$ is computable. Any certifier that
  turns an assumption into a speed-up certifies falsely at the least-favourable law with
  probability about $\delta^{\Kstar/\Kstar_F}\gg\delta$ (\cref{thm:structure}).
\item We extend certification to trees that cannot be enumerated (\cref{sec:local}). Feasibility
  of a bet reduces to a sequence of local conditions driven by the superhedging price
  (\cref{thm:cushion}). Any supersolution of the price recursion, including a closed
  form that needs only per-step ratio bounds (\cref{lem:ratio}), gives a sound certifier
  that costs $O(T|A|)$ per episode for arbitrary observation spaces
  (\cref{prop:scalable}).
\item We test the theory (\cref{sec:experiments}) on memory tasks,
  the Tiger POMDP \citep{kaelbling1998planning}, random trees and the published hard
  instances. The strongest baselines are scalar betting certificates: scalar Kelly and
  the betting confidence sequence of \citet{waudbysmith2024estimating} on the PDIS and
  DR payoffs. The best of them needs $1.05$--$2.65\times$ the plug-in's episodes on
  memory trees and $2.8$--$4.1\times$ on Tiger ($H\le5$). On shallow random trees a
  fitted DR bet can be faster than the plug-in (\cref{sec:random-comparators}).
  Doubling-grid certificates in the style of HCOPE, logarithmic smoothing and empirical
  Bernstein need $9$--$52\times$ more episodes, or fail. We check soundness at the
  least-favourable null on 21 instances, report the wall-clock and memory cost of every
  certifier (\cref{sec:cost}), and split the scalable certifier's gap to $\Kstar$ into
  supersolution, gambler-class and online-learning factors (\cref{tab:matched}).
\end{enumerate}

The rest of the paper is organized as follows. \Cref{sec:related} reviews related work,
and \cref{sec:setting} introduces logged trees, action-neutral laws and soundness.
\Cref{sec:exponent} proves the lower bound and the duality, classifies instances by
coverage (\cref{sec:coverage}) and relates $\Kstar$ to the efficiency bound.
\Cref{sec:achieve} gives matching certifiers.
\Cref{sec:consequences} derives estimator invariance, the coverage
ceiling and the price of structural assumptions. \Cref{sec:local} develops the scalable
cushion certifier, \cref{sec:experiments} reports experiments, and \cref{sec:conclusion}
concludes. Proofs and additional experiments are in the appendix.

\begin{figure}[t]
\centering
\includegraphics{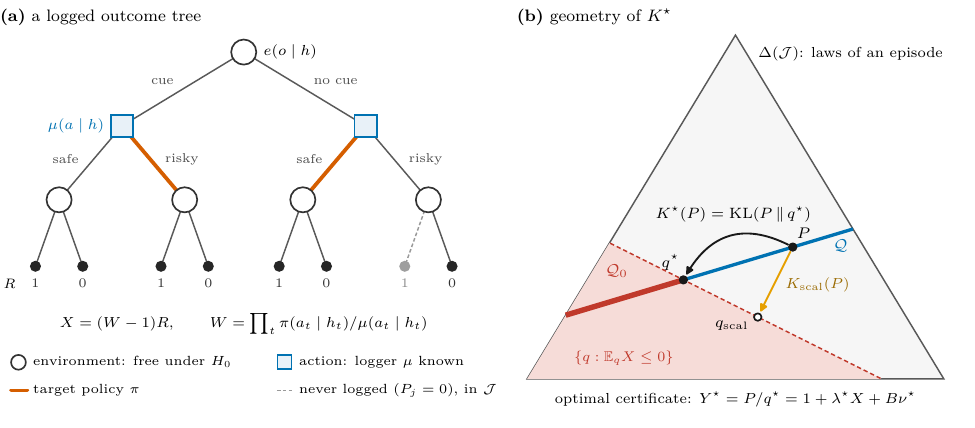}
\caption{Setting and main object. (a)~A logged outcome tree (the memory
task with $d=0$). Environment nodes (circles) have kernels left free by the null; at
action nodes (squares) the logger $\mu$ is known. Thick edges are the actions of the
target $\pi$. The grey leaf has zero logged probability but remains in the outcome space
$\J$, so every certificate must stay nonnegative on it. (b)~Schematic, not to scale, of
\cref{thm:duality}. $\Qs$ is the set of action-neutral laws and $\Qn$ its part with
$\E_qX\le0$. $\Kstar(P)$ is the divergence from $P$ to $\Qn$, attained at $q^\star$, and
the optimal certificate $P/q^\star$ is affine in the payoff $X$ and the fair bets $B$.
The scalar exponent $\Kscal$ is the divergence to the larger set $\{q:\E_qX\le0\}$, so
$\Kscal\le\Kstar$.}
\label{fig:overview}
\end{figure}

\section{Related work}
\label{sec:related}

Our work draws on sequential testing, testing by betting and off-policy evaluation. Relative
to these literatures, the new ingredient is the null class: with a known logger, the
action-neutral laws form a polytope, and this structure produces the fair bets $B$ and the
superhedging price $m$ from which our results follow. In particular, HCOPE
\citep{thomas2015hcope} applies a scalar concentration bound to $X$ and, in our terms,
targets the scalar exponent $\Kscal$ rather than $\Kstar$.

\paragraph{Sequential testing and instance-optimal lower bounds.} \Cref{thm:lower} uses the
change-of-measure argument of \citet{garivier2016optimal} and
\citet{kaufmann2016complexity}. $\Kscal$ is the $\KL_{\inf}$ of
\citet{honda2010asymptotically} applied to $X$. \citet{sethi2026markovian} characterize
optimal sequential tests with Markovian data under the same KL direction, but without an
action layer or a value inequality. Our contribution is the alternative class of
action-neutral laws, which may change any environment kernel, including kernels at
unvisited histories, but never the logged action conditionals, and its consequences.
The explicit plug-in bound of \cref{thm:plugin-explicit} follows the good-event plus
concentration template of track-and-stop analyses \citep{garivier2016optimal} and of
$\KL_{\inf}$-based identification with nonparametric arms \citep{agrawal2021optimal}; what
is specific here is that the plug-in is a logged-law projection and its loss is controlled
by a sandwich in $\ell_1$ that holds without any interior assumption.

\paragraph{Coverage and partial identification.} Without positivity the value of the
target is only partially identified, and the worst case sets uncovered rewards to the
bottom of their range, as in the bounds of \citet{manski1990nonparametric}.
\Cref{prop:positivity} shows that this worst case is exactly what the sequential problem
sees: the exponent is that of the identified lower value, and certification is impossible
when that value does not exceed the logger's. \Cref{prop:trichotomy} separates the
complementary regime, in which coverage is so strong that a single episode can be a
certificate.

\paragraph{E-values, growth-rate optimality and betting.} For an arbitrary composite null
and a simple alternative, \citet{larsson2025numeraire} prove that a log-optimal
(numeraire) e-variable exists and characterize it through the reverse information
projection; \citet{lardy2024reverse} extend this theory to cases where the infimal KL is
infinite. \Cref{thm:duality} gives the
closed-form instance of that theory for the polytope null $\Qn$ induced by a known
history-dependent logger. In this instance the optimal e-variable is affine,
$1+\lambda X+B\nu$, with $\nu$ a vector of fair bets on logged actions, and its feasible
set is the superhedging set of \cref{prop:superhedge}. This
explicit form is what the plug-in, the invariance theorem and the cushion certifier use.
\Cref{thm:duality} also specializes the growth-rate-optimal e-variables of
\citet{grunwald2024safe}, and
\cref{rem:bipolar} explains why the bipolar gap of \citet{ram2026optimal} and the
counterexamples of \citet{kitamura2012asymptotic} do not arise. The specialization adds the
explicit family of e-variables (\cref{lem:evariables}), the superhedging recursion
(\cref{prop:superhedge}) and the strong duality analysis including the infinite case
(\cref{app:duality}). The recursion is the discrete-time superhedging price of an
incomplete market \citep{elkaroui1995dynamic}, with the logged actions as the traded
assets. Testing by betting is surveyed by \citet{shafer2021testing} and
\citet{ramdas2023gametheoretic}. Split and running-MLE likelihood-ratio tests
\citep{wasserman2020universal} are also valid for composite nulls; on a finite tree
their denominator is a likelihood maximization over $\Qn$, the same convex program the
plug-in solves. The universal
mixture is Cover's universal portfolio
\citep{cover1991universal}, used as in \citet{waudbysmith2024estimating} and
\citet{orabona2024tight}; see \citet{kaufmann2021mixture} for mixture martingales in
sequential testing.

\paragraph{Anytime-valid off-policy evaluation.} \citet{karampatziakis2021offpolicy} and
\citet{waudbysmith2024anytime} construct betting confidence sequences for off-policy values
in contextual bandits, including doubly robust corrections. \citet{karampatziakis2020empirical}
use empirical likelihood, which is dual to a KL tilt, for bandit OPE.
\citet{sandoval2026multiarmed} prove Kelly optimality for multi-armed sequential testing.
Our setting is multi-step and history-dependent, with a known logger. Our certifier's rate
is asymptotically instance-optimal against the action-neutral class, and the logged-law projection is
what makes the plug-in practical.

\paragraph{Scalar high-confidence OPE and safe policy improvement.} HCOPE and HCPI
\citep{thomas2015hcope,thomas2015hcpi}, logarithmic smoothing \citep{sakhi2024logarithmic},
universal OPE \citep{chandak2021universal} and SPIBB \citep{laroche2019spibb} bound a
scalar estimator or restrict policy updates to well-covered regions. \Cref{thm:invariance}
shows that any certificate based on a per-episode statistic of the same log has exponent at
most $\Kstar$, and \cref{sec:experiments} measures the gap.

\paragraph{Semiparametric efficiency.} \citet{jiang2016doubly} and \citet{kallus2020double}
give the local efficiency theory of DR estimators. \Cref{thm:local} identifies the
efficiency bound as the small-advantage limit of $\Kstar$. Empirical likelihood has a
similar large-deviation optimality for moment restrictions \citep{kitamura2001asymptotic},
subject to the caveats of \citet{kitamura2012asymptotic}.

\paragraph{History-dependent OPE and the curse of memory.} \citet{tennenholtz2020off}
study off-policy evaluation in partially observable environments.
\citet{zhang2025tractability}
identify the curse of memory for history-dependent off-policy evaluation in POMDPs.
\citet{jajoo2026hardness} and \citet{luo2026revealing} prove worst-case exponential
evaluation lower bounds under a known history-dependent logger. These works establish the
phenomenon. \Cref{thm:ceiling} complements them with an instance-dependent ceiling for
one-sided certification, and \cref{tab:hard} computes $\Kstar$ on their constructions.

\section{Setting}
\label{sec:setting}

We first fix the objects: the tree of possible episodes, the class of environments that
could have generated the log, and what it means for a certifier to be sound.

\paragraph{Logged trees.} An episode alternates environment outcomes and logged actions,
\[
\tau=(o_1,a_1,o_2,a_2,\dots,o_T,a_T,o_{T+1}),\qquad h_t=(o_1,a_1,\dots,o_t).
\]
Rewards are part of the outcomes, and the total reward $R(\tau)$ lies in $[0,b]$. All
possible episodes form a finite rooted tree. Its non-leaf nodes are \emph{environment
nodes}, where an outcome is drawn, and \emph{action nodes}, where the logger acts. The tree
includes every outcome the environment could produce, including outcomes that never
appear in the logs. We write $\J$ for the set of leaves (complete episodes), $ha$ for the
child of an action node $h$ reached by action $a$, and $h\preceq j$ if node $h$ lies on
the path to leaf $j$. The logger $\mu(a\mid h)>0$ is known and its propensities are
recorded. The target $\pi(a\mid h)$ is fixed. Both may depend on the whole history. Let
\[
W(\tau)=\prod_{t=1}^{T}\frac{\pi(a_t\mid h_t)}{\mu(a_t\mid h_t)},\qquad
X(\tau)=\bigl(W(\tau)-1\bigr)R(\tau).
\]

\paragraph{Environments and logged laws.} An environment $e$ assigns a distribution over
outcomes to every environment node. The class $\Es$ contains \emph{all} such assignments
on the tree, so environments may be history-dependent and partially observable. Each
$e\in\Es$ and the logger $\mu$ together induce a law $P_e$ on $\J$. Since $\E_{P_e}[WR]$
is the value of $\pi$,
\[
\Delta(P_e) := V_e(\pi)-V_e(\mu)=\E_{P_e}[WR]-\E_{P_e}[R]=\E_{P_e}[X].
\]
Episodes are i.i.d.\ draws from $P:=P_{e_0}$ for an unknown true environment $e_0$.

\begin{lemma}[Logged laws form a polytope]
\label{lem:polytope}
Let $B\in\R^{\J\times N}$ have one column for each pair $(h,a)$ of an action node and an
action, with entries
$B_{j,(h,a)}=\1\{ha\preceq j\}-\mu(a\mid h)\1\{h\preceq j\}$. A distribution $q$ on
$\J$ equals $P_e$ for some $e\in\Es$ if and only if $B^\top q=0$. We call such laws
\emph{action-neutral} and write $\Qs=\{q\in\Delta(\J):B^\top q=0\}$.
\end{lemma}

\begin{proof}
Write $q(h)=\sum_{j\succeq h}q_j$. The condition $B^\top q=0$ states
$q(ha)=\mu(a\mid h)q(h)$ at every action node, that is, actions are drawn from $\mu$.
This is necessary. Conversely, if it holds, define $e$ at each environment node $h$ by
$e(o\mid h)=q(ho)/q(h)$, chosen arbitrarily when $q(h)=0$. Then $P_e=q$ by induction
along each path.
\end{proof}

The columns of $B$ have a direct reading. For $\nu\in\R^N$,
\begin{equation}
\label{eq:fairbets}
(B\nu)(\tau)=\sum_{t=1}^{T}\Bigl[\nu(h_t,a_t)-\sum_{a}\mu(a\mid h_t)\,\nu(h_t,a)\Bigr]
\end{equation}
is the total payoff of a sequence of \emph{fair bets on the logged action at odds $\mu$}.
Every control variate or doubly robust correction with known propensities has this form,
and $\E_q[B\nu]=0$ for every $q\in\Qs$.

\paragraph{Certifiers.} The null set is
$\Qn=\{q\in\Qs:\E_q X\le 0\}$. By \cref{lem:polytope} it is exactly the set of logged
laws of environments under which $\pi$ does not improve on $\mu$.

\begin{definition}[Soundness]
A test $\phi_n:\J^n\to\{0,1\}$ is \emph{$\delta$-sound} if $q^{\otimes n}(\phi_n=1)\le\delta$
for every $q\in\Qn$. A \emph{sequential certifier} is a stopping time $\tau_\delta$ for
the filtration generated by the episodes, possibly enlarged by independent
randomization. It is $\delta$-sound if $q^{\otimes\infty}(\tau_\delta<\infty)\le\delta$ for
every $q\in\Qn$.
\end{definition}

\begin{definition}[Certification exponent]
\label{def:exponent}
For $P\in\Qs$, let
\[
\Kstar(P)=\inf\{\KL(P\,\|\,q):q\in\Qn\},
\]
with $\inf\emptyset=\infty$. For comparison, the \emph{scalar exponent}
$\Kscal(P)=\sup\{\E_P\log(1+\lambda X):\lambda\ge0,\ 1+\lambda X\ge0 \text{ on } \J\}$ is the
optimal growth rate of betting on $X$ alone. By the duality of
\citet{honda2010asymptotically}, it is the $\KL_{\inf}$ of the law of $X$ against all
laws on the range of $X$ with non-positive mean.
\end{definition}

The direction of the divergence ($P$ first) is the one that governs test power. The
alternative $q$ may change any environment kernel, including kernels at histories that
$P$ never visits, but it can never change the logged action conditionals. This is what
separates $\Kstar$ from both the bandit $\KL_{\inf}$ \citep{honda2010asymptotically} and
an unconstrained divergence between episode laws.

\section{The certification exponent}
\label{sec:exponent}

\subsection{Lower bound}

\begin{theorem}[Lower bound]
\label{thm:lower}
Let $\tau$ be any stopping time for the episode filtration, possibly enlarged by
independent randomization, with $\E_P[\tau]<\infty$. Then for every law $q$ on $\J$,
\begin{equation}
\label{eq:tradeoff}
q^{\otimes\infty}(\tau<\infty)\ \ge\ \exp\bigl(-\E_P[\tau]\,\KL(P\,\|\,q)\bigr).
\end{equation}
Consequently:
\begin{enumerate}
\item[(i)] \emph{Sequential certifiers.} Every $\delta$-sound sequential certifier
  satisfies, for every $\delta\in(0,1)$ and every logged law $P\in\Qs$ under which $\pi$
  improves on $\mu$ (with the convention $\log(1/\delta)/0=\infty$ when $\Kstar(P)=0$),
  \[
  \E_P[\tau_\delta]\ \ge\ \frac{\log(1/\delta)}{\Kstar(P)}.
  \]
  The bound is non-asymptotic, has leading constant one and needs no power condition.
  If $\E_P[\tau_\delta]=\infty$ it holds trivially; otherwise $P(\tau_\delta<\infty)=1$.
\item[(ii)] \emph{Fixed-sample tests.} If $\phi_n$ is a $\delta$-sound test with
  $P^{\otimes n}(\phi_n=1)\ge1-\beta>\delta$, then
  $n\,\Kstar(P)\ge\kl(1-\beta,\delta)\ge(1-\beta)\log(1/\delta)-\log2$, where $\kl$ is the
  binary relative entropy; the first inequality uses that $y\mapsto\kl(x,y)$ is
  nonincreasing on $y\le x$.
\item[(iii)] \emph{Speed costs soundness.} If $\Kstar(P)<\infty$ (so that the minimizer
  $q^\star$ exists) and $q^\star$ is the
  least-favourable law of \cref{thm:duality}, then every stopping time with
  $\E_P[\tau]<\infty$, sound or not, satisfies
  $q^{\star\otimes\infty}(\tau<\infty)\ge\exp(-\E_P[\tau]\,\Kstar(P))$.
\end{enumerate}
\end{theorem}

\begin{proof}[Proof sketch]
Wald's identity and the data-processing inequality applied to the event $\{\tau\le N\}$, as in
\citet[Lemma~1]{garivier2016optimal}, give \eqref{eq:tradeoff} as $N\to\infty$. Parts (i) and
(ii) follow by taking the infimum over $q\in\Qn$, and (iii) by taking $q=q^\star$. Full
proof in \cref{app:lower}.
\end{proof}

Part (i) removes any power or asymptotic qualification. A certifier that fails to
certify with positive probability under $P$ has $\E_P[\tau_\delta]=\infty$, so finite
expected stopping time already means certifying almost surely. Part (iii) is the sharper
statement behind (i). A certifier that stops faster than $\log(1/\delta)/\Kstar$ in
expectation must certify falsely with probability greater than $\delta$, and it does so at
one identifiable null, the least-favourable environment $q^\star$. \Cref{thm:structure}
uses this to show that every certifier exploiting a structural assumption pays for its
speed with soundness at $q^\star$.

\subsection{Duality: the optimal certificate}

\begin{lemma}[Every e-variable is dominated by a linear bet]
\label{lem:evariables}
Assume $\Qn\neq\emptyset$, and let $Y:\J\to[0,\infty)$ satisfy
$\sup_{q\in\Qn}\E_q Y\le1$. Then there are $\lambda\ge0$ and $\nu\in\R^N$ with
$Y\le1+\lambda X+B\nu$ on every leaf.
\end{lemma}

\begin{proof}
$\sup_{q\in\Qn}\E_qY$ is the linear program
$\max\{Y^\top q: q\ge0,\ \1^\top q=1,\ B^\top q=0,\ X^\top q\le0\}$. It is feasible and
bounded, so by strong LP duality its value equals
$\min\{c: c\1+\lambda X+B\nu\ge Y,\ \lambda\ge0\}$, and the optimal $c$ is at most $1$.
\end{proof}

Let $\Ys=\{Y=1+\lambda X+B\nu:\lambda\ge0,\ \nu\in\R^N,\ Y\ge0\text{ on every leaf of }\J\}$
be the set of \emph{feasible bets}. Every $Y\in\Ys$ satisfies
$\E_qY=1+\lambda\E_qX\le1$ for $q\in\Qn$, so it is an e-variable for $\Qn$. By
\cref{lem:evariables}, these are all of them, up to domination.

\begin{theorem}[Duality]
\label{thm:duality}
For every $P\in\Qs$,
\[
\Kstar(P)=\sup_{Y\in\Ys}\E_P\log Y
=\sup\Bigl\{\E_P\log(1+\lambda X+B\nu):\lambda\ge0,\ 1+\lambda X+B\nu\ge0\text{ on }\J\Bigr\}.
\]
If $\Kstar(P)<\infty$, the supremum is attained at some $Y^\star=1+\lambda^\star X+B\nu^\star$.
The infimum defining $\Kstar$ is attained at a law $q^\star$ with
$q^\star_j=P_j/Y^\star_j$ on $\supp P$. Any further mass of $q^\star$ lies on leaves
where $Y^\star=0$. Moreover $\lambda^\star\E_{q^\star}X=0$ and $\KL(P\|q^\star)=\E_P\log Y^\star$.
\end{theorem}

\begin{proof}[Proof sketch]
\emph{Weak duality.} For $Y\in\Ys$ and $q\in\Qn$, Jensen's inequality gives
$\E_P\log Y-\KL(P\|q)=\E_P\log(Yq/P)\le\log\E_qY\le0$. \emph{Strong duality.} If
$\Kstar<\infty$, some $q\in\Qn$ is positive on $\supp P$. This is the refined Slater
condition for the convex program $\min_{q\in\Qn}\KL(P\|q)$, whose constraints are affine.
Computing the Lagrangian dual and optimizing out its scale gives the stated supremum. If
$\Kstar=\infty$, a Farkas-type alternative produces feasible bets whose growth rate is
unbounded. The full proof is in \cref{app:duality}.
\end{proof}

\begin{proposition}[Maximal leverage is a superhedging price]
\label{prop:superhedge}
Define $m$ on the nodes of the tree by $m(j)=-X(j)$ at leaves,
$m(h)=\sum_a\mu(a\mid h)\,m(ha)$ at action nodes, and $m(h)=\max_o m(ho)$ at environment
nodes. Then $m(\mathrm{root})=\max_{q\in\Qs}\E_q[-X]$. A multiplier $\lambda\ge0$ is
feasible (some $\nu$ makes $1+\lambda X+B\nu\ge0$) if and only if
$\lambda\,m(\mathrm{root})\le1$. In particular, if $m(\mathrm{root})<0$ then $\Qn=\emptyset$
and $\Kstar=\infty$.
\end{proposition}

\begin{proof}
By LP duality, $\max_\nu\min_j(1+\lambda X+B\nu)_j=\min_{q\in\Qs}\E_q[1+\lambda X]$. The
maximum of $\E_q[-X]$ over action-neutral laws is computed by backward induction, since
the environment chooses outcomes freely while actions follow $\mu$.
\end{proof}

\begin{remark}[Interpretation]
$Y=1+\lambda X+B\nu$ is the terminal wealth of a gambler who starts with one unit,
places fair bets at odds $\mu$ on each logged action, and holds a claim $\lambda X$ on
the importance-weighted advantage. Action risk can be hedged with fair bets.
Environment risk cannot, because the environment is adversarial within $\Es$. This is
Kelly betting with a contingent claim in an incomplete market, and $m$ is the
superhedging price of the short claim $-X$. Doubly robust corrections are special fair
bets, but in the optimal certificate the bets are chosen to maximize growth, not to
minimize variance.
\end{remark}

\begin{remark}[No bipolar gap]
\label{rem:bipolar}
\citet{ram2026optimal} and \citet{kitamura2012asymptotic} show that the optimal growth
rate of a sequential test can fall strictly below the naive infimum of $\KL$ over the
null, unless $\KL(P\|\cdot)$ restricted to the null is lower semicontinuous at the right
points. Here $\Qn$ is a compact polytope in a finite-dimensional simplex, so
\cref{thm:lower,thm:duality} give the same quantity. The duality agrees with the
growth-rate-optimal e-variable of \citet{grunwald2024safe}. What \cref{lem:evariables}
adds is the explicit finite-dimensional form of all e-variables for this null.
\end{remark}

\begin{remark}[Continuous rewards]
\label{rem:continuous}
If terminal rewards take values in a compact set within $[0,b]$, the leaf constraint
$1+\lambda(W-1)R+B\nu\ge0$ is affine in $R$. It therefore suffices to check it at
$R\in\{0,b\}$, so a compact reward support reduces to a finite tree without loss of
generality.
\end{remark}

\begin{remark}[KL direction]
Entropic-risk and soft-Bellman dualities use $\KL(q\|P)$, with the alternative first
\citep{todorov2007linearly}. $\Kstar$ is a reverse information projection with $P$
first, and does not reduce to those recursions.
\end{remark}

\subsection{Coverage decides the regime}
\label{sec:coverage}

Let $S=\supp P$ and $m=m(\mathrm{root})$ from \cref{prop:superhedge}.

\begin{proposition}[Coverage trichotomy]
\label{prop:trichotomy}
Let $\mu>0$ on every action node.
\begin{enumerate}
\item[(a)] If $\Delta\le0$ then $\Kstar(P)=0$, and every $\delta$-sound certifier has
$P(\tau<\infty)\le\delta$; in particular $\E_P\tau=\infty$.
\item[(b)] If $\Delta>0$ then $\Kstar(P)>0$, and the infimum defining $\Kstar$ is attained
whenever it is finite.
\item[(c)] If $\Delta>0$, the following are equivalent: \cref{ass:fs}; $m>0$;
$\Kstar(P)<\infty$.
\item[(d)] If $\Delta>0$ and $m\le0$, there is $\nu$ with $G=X+B\nu\ge0$ on $\J$ and
$\E_PG=\Delta$. The rule $\tau=\inf\{t:G(j_t)>0\}$ is $\delta$-sound for every $\delta$
($q(\tau<\infty)=0$ for all $q\in\Qn$) and $\E_P\tau=1/P(G>0)$, independent of $\delta$.
\end{enumerate}
\end{proposition}

\begin{proof}[Proof sketch]
(a) $P\in\Qn$, so \cref{thm:lower} applies with $q=P$. (b) $P\notin\Qn$ and $\Qn$ is
closed; lower semicontinuity of $\KL(P\|\cdot)$ on the compact set gives attainment.
(c) and (d) follow from the LP dual $m=\min\{c:c+B\nu\ge-X\text{ on }\J\}$ of
\cref{prop:superhedge}: if $m\le0$ the dual solution gives $G\ge0$, and every
$q\in\Qn$ has $0\ge\E_qX=\E_qG$, so $q(G>0)=0$ while $P(G>0)>0$. If $m>0$, mixing a
maximizer of $\E_q[-X]$ with a full-support element of $\Qs$ gives a full-support law in
$\Qn$. Full proof in \cref{app:coverage}.
\end{proof}

Part (a) says that no amount of data certifies a policy that does not beat the logger;
part (d) says that a dominance certificate, when it exists, is free. The exponent is
informative exactly in the middle regime $\Delta>0<m$.

\paragraph{Partial coverage.} Now allow $\mu(a\mid h)=0$ at some action nodes. Actions
outside the logger's support are never observed, so their rewards are unconstrained in
$[0,b]$ and the null must contain every reward assignment there. Let $W^0$ be the
importance weight restricted to covered trajectories, $X^0=(W^0-1)R$, and
$\Delta^0=\E_PX^0=V(\pi)-V^{\mathrm{unc}}(\pi)-V(\mu)$, where $V^{\mathrm{unc}}(\pi)$ is the
value $\pi$ collects through uncovered actions.

\begin{proposition}[Positivity under partial coverage]
\label{prop:positivity}
With uncovered rewards adversarial in $[0,b]$, $\Kstar(P)$ equals the exponent of the
covered tree with claim $X^0$. In particular certification is impossible ($\Kstar=0$ and
part (a) applies) if and only if $\Delta^0\le0$, which can happen even when
$V(\pi)>V(\mu)$.
\end{proposition}

\begin{proof}[Proof sketch]
The adversary sets every uncovered reward to $0$, which is the worst case for the
inequality $\E_qX\le0$ and is invisible to the data; the remaining null is the covered
tree's. Full proof in \cref{app:coverage}.
\end{proof}

This is the sequential, exponent-level form of the partial-identification bounds of
\citet{manski1990nonparametric}: the identified set for $V(\pi)$ has lower end
$V(\pi)-V^{\mathrm{unc}}(\pi)$, and the exponent is that of the lower end. On 40 random
instances with uncovered actions, 26 had $\Delta^0\le0$, 18 of them with
$V(\pi)>V(\mu)$; brute-force optimization over uncovered rewards matched the reduction to
$7.4\times10^{-8}$.

\subsection{Local expansion: the efficiency bound is a limit of the exponent}

Let $\sigeff(P)=\min_\nu\Var_P(X+B\nu)$ be the residual variance of $X$ after
$L^2(P)$ projection onto the span of fair bets. This is the semiparametric efficiency
bound for $\Delta$ with known propensities
\citep{jiang2016doubly,kallus2020double}. Let $\sigis(P)=\Var_P(X)$.

\begin{theorem}[Local expansion]
\label{thm:local}
Along a sequence of instances with $\Delta\to0$ and $\sigeff$ bounded away from zero,
suppose $X$ is uniformly bounded on $\J$ and the efficient residual
$Z=X+B\theta^\star$, $\theta^\star\in\argmin_\nu\Var_P(X+B\nu)$, is uniformly bounded
on $\supp P$. Then
\[
\Kstar(P)=\frac{\Delta^2}{2\sigeff}\bigl(1+O(\Delta)\bigr),\qquad
\Kscal(P)=\frac{\Delta^2}{2\sigis}\bigl(1+O(\Delta)\bigr).
\]
Hence $\Kstar/\Kscal\to\sigis/\sigeff\ge1$.
\end{theorem}

Both hypotheses concern only $X$ and the law $P$, which is the usual regularity for an
efficient influence function. The dual bet must also be nonnegative on leaves outside
$\supp P$, where $P$ does not pin down $\theta^\star$. \Cref{app:local} shows that this
costs nothing. Choose the minimizer whose fair bets vanish at action nodes that $P$ never
reaches. Its residual then satisfies $\sup_\J|Z|\le2\sup_\J|X|+\sup_{\supp P}|Z|$ on
every leaf of the tree.

The proof (\cref{app:local}) evaluates the dual at a shrunken efficient bet and the
primal at a linear tilt of $P$. \Cref{thm:local} identifies the efficiency bound as the
small-advantage limit of the large-deviation exponent. Away from that limit the two
quantities separate in \emph{both} directions. Over 189 random trees with $\Delta>0.01$
(\cref{fig:random}), the ratio $\Kstar/(\Delta^2/2\sigeff)$ had median $1.03$,
10--90\% range $0.44$--$3.0$ and extremes $0.014$ and $19.4$. It was below $1$ on 47\% of
trees. Where it exceeds $1$, a certificate calibrated by a central limit theorem or
Bernstein bound for the best DR estimator wastes that factor. Where it is below $1$, such
a certificate claims a rate that no sound procedure attains, so it cannot be uniformly
valid at small $\delta$. $\Kstar$ must be computed, not approximated.

The expansion also has a non-asymptotic form. Let $Z=X+B\theta^\circ$ be the efficient
residual of \cref{app:local} (so $\sup_\J|Z|<\infty$ and $\E_PZ=\Delta$), and write
$\sigma=\sigeff$.

\begin{proposition}[Explicit local sandwich]
\label{prop:sandwich}
Let $\Delta>0$.
\begin{enumerate}
\item[(i)] If $\rho_+=(\Delta/\sigma)\max\{0,\max_S(Z-\Delta)\}<1$ then
$\Kstar(P)\le\Delta^2/\bigl(2\sigma(1-\rho_+)\bigr)$.
\item[(ii)] Let $s=\Delta/(\sigma+\Delta^2)$. If $s\max\{0,-\min_\J Z\}\le1$ and
$\rho_-=s\max\{0,-\min_SZ\}<1$ then
$\Kstar(P)\ge\Delta^2(1-2\rho_-)/\bigl(2(\sigma+\Delta^2)(1-\rho_-)\bigr)$.
\end{enumerate}
The same bounds hold for $\Kscal$ with $Z$ replaced by $X$ and $\sigma$ by $\sigis$.
\end{proposition}

Both $\rho_\pm=O(\Delta)$, so \cref{prop:sandwich} implies \cref{thm:local} with explicit
constants. Proof in \cref{app:local}; we checked the bounds numerically on 111 random
trees and a small-$\Delta$ sweep.

\section{Achieving the exponent}
\label{sec:achieve}

\Cref{fig:architecture} in \cref{app:details} summarizes the certifiers of this section and of
\cref{sec:local}.

\subsection{A predictable Kelly certifier with a logged-law plug-in}

\begin{algorithm}[t]
\caption{Projected plug-in Kelly certifier}
\label{alg:plugin}
\begin{algorithmic}[1]
\Require level $\delta$; refit times $r_k=\lceil1.15^k\rceil$
\State $M_0\gets1$; $(\hat\lambda,\hat\nu)\gets(0,0)$
\For{$t=1,2,\dots$}
  \If{$t-1=r_k$ for some $k$}
    \State $\hat p\gets$ empirical law of episodes $1,\dots,t-1$
    \State $\tilde p\gets\argmin_{q\in\Qs}\KL(\hat p\|q)$ \Comment{logged-law projection}
    \State $(\hat\lambda,\hat\nu)\gets\argmax\{\E_{\tilde p}\log(1+\lambda X+B\nu):\lambda\ge0,\ 1+\lambda X+B\nu\ge0\text{ on }\J\}$
  \EndIf
  \State observe episode $\tau_t$; $s_t\gets1-t^{-1/2}$
  \State $M_t\gets M_{t-1}\bigl(1+s_t[\hat\lambda X(\tau_t)+(B\hat\nu)(\tau_t)]\bigr)$
  \If{$M_t\ge1/\delta$} \Return certify at $\tau_\delta=t$ \EndIf
\EndFor
\end{algorithmic}
\end{algorithm}

\Cref{alg:plugin} bets, at each episode, the optimal bet for the current estimate of $P$,
shrunk towards the trivial bet. The estimate is not the empirical law $\hat p$ itself but
its projection onto action-neutral laws. This is the maximum-likelihood estimate of the
environment kernels with $\mu$ known: $\tilde p$ equals $\mu$ times the empirical
environment conditionals. The projection matters in practice. The empirical law is not
action-neutral, and sampling noise in its action frequencies looks like a positive edge
on the fair bets. The plug-in then overfits: $\KL(\hat p\|\Qs)$ is of order
$\rank(B)/2t$. In a memory task with $\rank B=46$, $\Kstar(\hat p)$ averaged $0.285$ at
$t=100$ against a true $\Kstar=0.021$, and the projection cut the median stopping time
from 6555 to 440 episodes against a lower bound of 334. Since the projection is the
identity on $\Qs$, it does not change the asymptotic analysis.

\begin{theorem}[Asymptotically instance-optimal plug-in certifier]
\label{thm:plugin}
Let the refit times $r_1<r_2<\cdots$ satisfy $r_{k+1}\le\gamma r_k+1$ for a constant
$\gamma$. Let the bet used at episode $t$ be the one fitted at the last refit before $t$.
\begin{enumerate}
\item[(i)] \emph{Soundness.} \Cref{alg:plugin} is $\delta$-sound.
\item[(ii)] \emph{Almost-sure rate.} If $\Kstar(P)<\infty$, then
  $\liminf_{t\to\infty}\frac1t\log M_t\ge\Kstar(P)$ almost surely under $P$.
\item[(iii)] \emph{Optimal expected stopping time.} If $0<\Kstar(P)<\infty$, then for every
  $\epsilon\in(0,1/3)$ there is a constant $C(\epsilon,P)$, independent of $\delta$, such
  that
  \[
  \frac{\log(1/\delta)}{\Kstar(P)}\ \le\ \E_P[\tau_\delta]\ \le\
  \frac{\log(1/\delta)}{(1-3\epsilon)\Kstar(P)}+C(\epsilon,P).
  \]
  Hence $\E_P[\tau_\delta]/\log(1/\delta)\to1/\Kstar(P)$ as $\delta\to0$.
\end{enumerate}
\end{theorem}

\begin{proof}[Proof sketch]
(i) Each factor $1+s_t(\hat\lambda X+B\hat\nu)$ is a convex combination of $1$ and a
feasible bet, and is predictable. Under every $q\in\Qn$, $M_t$ is therefore a
nonnegative supermartingale, and Ville's inequality gives
$q(\sup_tM_t\ge1/\delta)\le\delta$. (ii)--(iii) The projection has support exactly
$\supp P$ once every leaf of $\supp P$ has been observed. Its distance to $P$ is
controlled by Hoeffding's inequality. The optimal log-wealth is continuous at $P$ on laws
with that support. Together these make the probability of every bad event summable. The
lower bound in (iii) is \cref{thm:lower}(i), which holds for every $\delta$-sound
certifier. Full proof in
\cref{app:plugin}.
\end{proof}

The analysis assumes that each refit returns an exact maximizer of the empirical
log-wealth over the \emph{full} feasible set, i.e.\ with nonnegativity enforced on every
leaf of $\J$ and not only on observed leaves. Soundness (i) needs only feasibility, so a
numerical solver that returns a feasible but suboptimal bet remains sound; parts
(ii)--(iii) then hold with $\Kstar$ replaced by $\Kstar$ minus the solver's optimality
gap. Our log-barrier solver returns feasible iterates with duality gap at most
$\epsilon_{\mathrm{opt}}=10^{-10}$, so \cref{thm:plugin-explicit} below holds with $\Kstar$
replaced by $\Kstar-\epsilon_{\mathrm{opt}}$. The constant $C(\epsilon,P)$ is not explicit; the next theorem makes the whole bound
explicit. Write $P_{\min}=\min_SP$, $q^\star=P/Y^\star$ on $S$,
$\Lambda=\max\{\max_S\log(1/q^\star),\log(2/P_{\min})\}$, $c_S=|S|+\sqrt{2|S|/P_{\min}}$
and $\eta'=P_{\min}/(2c_S)$.

\begin{theorem}[Explicit plug-in bounds]
\label{thm:plugin-explicit}
Let $0<\Kstar(P)<\infty$ and run \cref{alg:plugin} with $\gamma=1.15$, $R_1=2$.
\begin{enumerate}
\item[(i)] For $\beta\in(0,1)$ and $T\ge1$, with $\ell=\log(4|S|T/\beta)$ and
$T_0=\max\{4,R_1+1,\lceil2+\gamma\ell/(2\eta'^2)\rceil\}$,
\begin{align*}
&P\Bigl(\log M_T<T\Kstar-R_\beta(T)\Bigr)\le\beta,\qquad\text{where}\\
&R_\beta(T)=(T_0-1)\bigl(\Kstar+\tfrac12\log T_0\bigr)+4\sqrt T
+2\Lambda c_S\sqrt{2\gamma T\ell}+\Lambda\sqrt{2T\log(2/\beta)}.
\end{align*}
\item[(ii)] There are constants $a,b,T_2,T_3$, given in closed form in \cref{app:plugin} and
depending only on $(\Kstar,\Lambda,|S|,P_{\min},\gamma,R_1)$, such that with
$T^\star=\log(1/\delta)/\Kstar$ and $L^\star=\max(2T^\star,3)$,
\[
\E_P[\tau_\delta]\le T^\star+\frac{a}{\Kstar}\sqrt{L^\star\log L^\star}
+\frac{b}{\Kstar}\log^2L^\star+\max(T_2,T_3)+2 .
\]
\end{enumerate}
In particular $\E_P[\tau_\delta]=\log(1/\delta)/\Kstar
+O\bigl(\sqrt{\log(1/\delta)\log\log(1/\delta)}\bigr)$, which matches
\cref{thm:lower}(i) to first order with an explicit second-order term, without
\cref{ass:fs}.
\end{theorem}

The proof (\cref{app:plugin}) replaces \cref{lem:continuity} by a quantitative continuity
lemma and controls the plug-in error at all refits with Hoeffding's and the
Azuma--Hoeffding inequalities.

\begin{conjecture}[Second-order term]
\label{conj:second-order}
The \emph{plug-in} bound of \cref{thm:plugin-explicit} can also be improved to
$O(\log\log(1/\delta))$ (proved below for the exact mixture, but not for the
plug-in), and no family of $\delta$-sound certifiers achieves
$\log(1/\delta)/\Kstar(P)+o(\log\log(1/\delta))$ uniformly over a relatively open subset of
$\{P:0<\Kstar(P)<\infty\}$.
\end{conjecture}

The next construction gives finite-sample bounds of a different kind, with polynomial
constants but a dimension penalty.

\subsection{A universal mixture with finite-sample guarantees}

Let $U$ be a basis of $\range(B)$, $D=[X,U]$, $d=1+\rank(B)$, and let
$Z=\{z=(\lambda,\nu'):\lambda\ge0,\ 1+Dz\ge0\text{ on }\J\}$ be the set of feasible bets
in these coordinates. Write $D_t=D(\tau_t)$.

\begin{assumption}[Full support]
\label{ass:fs}
$\Qn$ contains a law with full support on $\J$.
\end{assumption}

Under \cref{ass:fs}, $Z$ is a bounded convex body that contains a neighbourhood of
$0$ in $\{\lambda\ge0\}$ (\cref{app:mixture}). \Cref{ass:fs} is not a convenience. When
$\Delta>0$ it holds if and only if $m(\mathrm{root})>0$, if and only if $\Kstar<\infty$
(\cref{prop:trichotomy}(c)). It therefore fails exactly in the arbitrage regime, where the
certifier of \cref{prop:trichotomy}(d) stops after $1/P(G>0)$ episodes in expectation
and no mixture is needed. The plug-in certifier needs no such assumption
(\cref{thm:plugin-explicit}).

\begin{theorem}[Universal mixture]
\label{thm:mixture}
Let $M_T=\mathrm{vol}(Z)^{-1}\int_Z\prod_{t\le T}(1+D_tz)\,dz$ and
$\tau_\delta=\inf\{T:M_T\ge1/\delta\}$. Under \cref{ass:fs}:
\begin{enumerate}
\item[(i)] $M_T$ is a nonnegative supermartingale under every $q\in\Qn$, so
  $\tau_\delta$ is $\delta$-sound.
\item[(ii)] For every sequence of episodes and every $T$,
  $\log M_T\ge\max_{z\in Z}\sum_{t\le T}\log(1+D_tz)-d\log(T+1)-1$.
\item[(iii)] Suppose $0<\Kstar(P)<\infty$, let $r$ be the range of $\log Y^\star$ on
  $\supp P$, and let $c_T=\log(1/\delta)+1+d\log(T+1)$. For every $T$ with
  $T\Kstar>c_T$,
  \[
  P(\tau_\delta>T)\le\exp\Bigl(-\frac{2(T\Kstar-c_T)^2}{Tr^2}\Bigr).
  \]
  Consequently $P(\tau_\delta>T)\le\beta$ whenever
  $T\ge c_T/\Kstar+b_\beta^2/\Kstar{}^2+2b_\beta\sqrt{c_T}/\Kstar{}^{3/2}$, with
  $b_\beta=r\sqrt{\log(1/\beta)/2}$. This condition is implicit in $T$ only through
  $d\log(T+1)$. An explicit sufficient condition is that, for some $\epsilon\in(0,1/2)$,
  \[
  T\ \ge\ T_\beta(\epsilon):=\frac{\log(1/\delta)+1+d\log\frac{d}{e\epsilon\Kstar}+\epsilon\Kstar}
  {(1-2\epsilon)\Kstar}+\frac{b_\beta^2}{4\epsilon(1-2\epsilon)\Kstar{}^2}.
  \]
  With $\epsilon\to0$ slowly as $\delta\to0$, $T_\beta(\epsilon)=(1+o(1))\log(1/\delta)/\Kstar$
  for fixed $\beta$ and $P$.
\item[(iv)] Let $\bar T$ be the smallest $T$ with $T\Kstar\ge c_T$, and suppose
  $\bar T+1\ge2d/\Kstar$. Then
  \[
  \E_P[\tau_\delta]\le\bar T+\frac{r\sqrt{\pi\bar T}}{\Kstar}+\frac{8r^2}{\Kstar{}^2}+2,
  \qquad
  \bar T=\frac{\log(1/\delta)}{\Kstar}+O\Bigl(\frac{d\log\log(1/\delta)}{\Kstar}\Bigr).
  \]
\item[(v)] For each fixed $P$ with $0<\Kstar(P)<\infty$, the exact mixture
  also obeys the sharper expected-time bound
  \[
  \frac{\log(1/\delta)}{\Kstar}
  \le \E_P[\tau_\delta]
  \le \frac{\log(1/\delta)}{\Kstar}
       +\frac{d}{\Kstar}\log\log(1/\delta)+C_P
       \qquad(\delta\downarrow0),
  \]
  where $C_P<\infty$ depends on $P$ (through $\Kstar$, $d$ and the support of $P$) but
  not on $\delta$.
\end{enumerate}
\end{theorem}

\begin{proposition}[Finite-point mixture: conditional time and sampling cost]
\label{prop:finite-mixture}
Under \cref{ass:fs}, fix $M$ feasible bets $z_1,\ldots,z_M\in Z$
independently of the episodes and put
$M_T^{(M)}=M^{-1}\sum_{i=1}^M\prod_{t\le T}(1+D_tz_i)$.
The test that stops when $M_T^{(M)}\ge1/\delta$ is
$\delta$-sound, conditional on \emph{any} such fixed bets.
For $P$ with $0<\Kstar(P)<\infty$, suppose all sampled factors
are positive on $\supp P$, and write
\[
K_M=\max_i\E_P\log(1+D_1z_i),\qquad
u_M=\max_{j\in\supp P}\log(1+D(j)z_{i^\star}),
\]
where $i^\star$ maximizes $K_M$. If $K_M>0$, its stopping time
$\tau_\delta^{(M)}$ obeys
\[
\E_P[\tau_\delta^{(M)}\mid z_1,\ldots,z_M]
\le\frac{\log(1/\delta)+\log M+u_M}{K_M}.
\]
In addition, for \emph{independent exact uniform} draws from $Z$,
every $0<\epsilon<\Kstar$ and $0<\eta<1$ satisfy
\[
M\ge\left(\frac{2(\Kstar+\log2)}{\epsilon}\right)^d
       \log\frac1\eta
\quad\Longrightarrow\quad
\Pr(K_M\ge\Kstar-\epsilon)\ge1-\eta.
\]
On this event the preceding conditional bound applies. The
implemented hit-and-run points need not be independent exact
uniform draws: their mixture has the first, conditional guarantee,
but the displayed draw-probability bound is \emph{not} asserted
for that sampler without a separate mixing analysis. The required
$M$ grows exponentially in $d=1+\rank(B)$; $\rank(B)$ is reported
for every instance of \cref{sec:experiments} (it ranges from 4 to 62),
and for $d\ge10$ no practical $M$ meets the condition, which is why
the finite 300-point mixture of the experiments is not a test of
\cref{thm:mixture}.
\end{proposition}

Part (ii) is the regret bound of Cover's universal portfolio \citep{cover1991universal},
and using it as an e-process follows \citet{waudbysmith2024estimating} and
\citet{orabona2024tight}. Parts (iii)--(iv) convert the regret into explicit stopping-time
bounds through Hoeffding's inequality. Part (v) instead compares the
stopping time to a drifted random walk crossing the mixture's logarithmic
regret boundary and applies Wald's identity (\cref{app:mixture}). Its
constant $C_P$ is not explicit, and it concerns the exact integral,
not a finite sampled approximation. The penalty $d\log(T+1)$ is
large when $\rank(B)$ is large, and the mixture is then slow in practice. Averaging gives
the best of both constructions.

\begin{corollary}[Hedged certifier]
\label{cor:hedge}
Run \cref{alg:plugin} (wealth $M^{\mathrm{plug}}$) and the universal mixture (wealth
$M^{\mathrm{mix}}$) on the same episodes, and certify when
$\tfrac12(M^{\mathrm{plug}}_t+M^{\mathrm{mix}}_t)\ge1/\delta$. Then the hedged certifier
is $\delta$-sound, and pathwise
$\tau^{\mathrm{hedge}}_\delta\le\min(\tau^{\mathrm{plug}}_{\delta/2},\tau^{\mathrm{mix}}_{\delta/2})$.
Consequently, \cref{thm:mixture}(iii)--(iv) hold for $\tau^{\mathrm{hedge}}_\delta$ with
$\log(1/\delta)$ replaced by $\log(2/\delta)$, and
under \cref{ass:fs} (equivalently $0<\Kstar(P)<\infty$ for the fixed $P$),
$\E_P[\tau^{\mathrm{hedge}}_\delta]\le
\log(1/\delta)/\Kstar+(d/\Kstar)\log\log(1/\delta)+C_P$, by \cref{thm:mixture}(v). Moreover,
$\E_P[\tau^{\mathrm{hedge}}_\delta]/\log(1/\delta)\to1/\Kstar$. More generally, for weights
$w_i>0$ summing to one, $\sum_iw_iM^i$ is sound and stops no later than
$\min_i\tau^i_{w_i\delta}$.
\end{corollary}

\begin{proof}
A convex combination of nonnegative supermartingales with initial value $1$ is again one,
so Ville's inequality gives soundness. Pathwise $\sum_iw_iM^i_t\ge w_iM^i_t$, so
$M^i_t\ge1/(w_i\delta)$ implies certification. The remaining claims follow from
$\E[\min]\le\min\E$ and \cref{thm:plugin,thm:mixture,thm:lower} at level $\delta/2$.
\end{proof}

In our implementation the integral over $Z$ is replaced by an average over hit-and-run
points drawn before any episode is seen. Any fixed mixing law gives a sound
certifier, so soundness is exact. The bounds (ii)--(iv) are proved for the exact
integral; a fixed finite set of sampled points need not satisfy the pathwise
regret bound (ii) or the derived stopping-time bounds. We report the Monte Carlo
mixture separately.

\section{Consequences of the exponent}
\label{sec:consequences}

\subsection{Estimator invariance and information monotonicity}

\begin{theorem}[Estimator invariance and information monotonicity]
\label{thm:invariance}
\begin{enumerate}
\item[(a)] For every $c\in\R^N$, replacing $X$ by $X-Bc$ leaves $\Qn$, and hence $\Kstar$,
  unchanged. This covers per-decision importance sampling, every doubly robust payoff
  with any $Q$-model, and every action-centred baseline. Conversely, $\range(B)$ is
  exactly the set of payoffs whose mean is zero under every action-neutral law.
\item[(b)] Let a certifier use the episodes only through an i.i.d.\ per-episode statistic
  $S=s(\tau)$. If it is $\delta$-sound, then for every $\delta$,
  \[
  \E_P[\tau_\delta]\ \ge\ \frac{\log(1/\delta)}{K_S(P)},\qquad
  K_S(P)=\inf_{q\in\Qn}\KL(P\circ s^{-1}\|q\circ s^{-1})\le\Kstar(P).
  \]
  In particular $\Kscal\le\Kstar$. A certifier that uses only $X$ and is $\delta$-sound
  against every law on the range of $X$ with non-positive mean, as scalar concentration
  certificates are, satisfies $\E_P[\tau_\delta]\ge\log(1/\delta)/\Kscal$.
\end{enumerate}
\end{theorem}

\begin{proof}
(a) $\E_q[Bc]=0$ for $q\in\Qs$ (and $c$ ranges over all of $\R^N$, so the sign is immaterial). For the converse, strictly positive action-neutral laws
exist (take any environment with full support), so $\Qs$ spans $\ker B^\top$, whose
orthogonal complement is $\range(B)$. (b) The proof of \cref{thm:lower} uses only the
observed process and the null class. Apply \cref{thm:lower}(i) to the image experiment,
with observations $s(\tau_t)$ and null $\{q\circ s^{-1}:q\in\Qn\}$. This gives the bound
with $K_S$, and the data-processing inequality gives $K_S\le\Kstar$. For the last claim,
apply \cref{thm:lower}(i) with observations $X(\tau_t)$ and the scalar null, whose
exponent is $\Kscal$ by the duality of \citet{honda2010asymptotically}. Since
$\Kscal\le K_X\le\Kstar$, the comparison holds.
\end{proof}

\Cref{thm:invariance} has a practical reading. Whether a certificate uses IS,
per-decision IS or DR, and whichever $Q$-model the DR estimator uses, the best achievable
exponent on the log is $\Kstar$. The estimator only affects how close a particular
certificate comes to it. This explains why a sequence of statistic-level refinements on
the same log (clipping, prefix truncation, transferred surrogates, paired comparisons)
cannot beat a well-designed direct bound by more than the gap to $\Kstar$. Only new data
can do better. For instance, if independent rollouts of a reference policy are available,
with $\rho$ reference episodes per logged episode, the argument of \cref{thm:lower} applied
to the product experiment suggests the exponent
$\inf_{e:\Delta_e\le0}[\KL(P_{e_0}\|P_e)+\rho\KL(P^{\mathrm{ref}}_{e_0}\|P^{\mathrm{ref}}_e)]\ge\Kstar$.

\subsection{The coverage ceiling and the curse of memory}

\begin{theorem}[Coverage ceiling]
\label{thm:ceiling}
Assume the reward is terminal and that $0$ is a possible reward at every terminal node.
Let $A=\{W>0\}$ be the event that $\pi$ could have produced the logged actions, and let
$V(\pi)=\E_P[WR]$, $V(\mu)=\E_P[R]$ and $r_A=\E_P[R\1_A]$. If
$\Delta=V(\pi)-V(\mu)>0$ and $V(\mu)>r_A$, then
\[
\Kstar(P)\ \le\ P(A)\,\log\frac{V(\pi)-r_A}{V(\mu)-r_A}.
\]
\end{theorem}

\begin{proof}
Let $\eta=\Delta/(V(\pi)-r_A)$, which lies in $(0,1)$ because $V(\pi)-r_A>V(\mu)-r_A>0$
and $\Delta>0$, and define $q$ from $P$ by replacing the terminal reward
kernel at histories in $A$ by $(1-\eta)P(r\mid h)+\eta\delta_0$. This changes an
environment kernel only, so $q\in\Qs$. Since $W=0$ off $A$, $\E_q[WR]=(1-\eta)V(\pi)$ and
$\E_q[R]=V(\mu)-\eta r_A$, so $\E_qX=\Delta-\eta(V(\pi)-r_A)=0$. On $A$ the likelihood
ratio is at most $1/(1-\eta)$, and off $A$ it is $1$. Hence
$\KL(P\|q)\le P(A)\log\frac1{1-\eta}$, which is the claimed bound.
\end{proof}

The ceiling isolates the quantity that matters: the probability that the \emph{entire}
logged action sequence agrees with what $\pi$ could have done. It is not the
probability of logging any single action. In the memory task of \cref{sec:experiments},
making the final rewarding action rare ($p$ from $0.2$ to $0.001$) moves $\Kstar$ only
from $0.276$ to $0.190$.

\begin{corollary}[Memory delay]
\label{cor:memory}
Under the assumptions of \cref{thm:ceiling}, suppose a deterministic $\pi$ must carry a
cue across $d$ steps, at each of which $\mu$ logs uniformly over $k$ actions. Suppose also that $\Delta>0$ and that $d$ is large
enough that $V(\mu)>b\,k^{-d}$. Then $P(A)\le k^{-d}$ and
\[
\Kstar(P)\ \le\ k^{-d}\,\ell_d,\qquad
\ell_d:=\log\frac{V(\pi)-b\,k^{-d}}{V(\mu)-b\,k^{-d}}.
\]
Consequently every $\delta$-sound certifier has
$\E_P[\tau_\delta]\ge k^{d}\log(1/\delta)/\ell_d$ for every $\delta$. If $V(\pi)$ and
$V(\mu)$ do not depend on $d$, as when the distractor steps carry no reward, then
$\ell_d\downarrow\log(V(\pi)/V(\mu))$, so $\Kstar=O(k^{-d})$ and
$\E_P[\tau_\delta]=\Omega(k^d\log(1/\delta))$.
\end{corollary}

\begin{proof}
Every leaf in $A$ follows $\pi$'s action at each of the $d$ steps, so $P(A)\le k^{-d}$.
Hence $r_A\le b\,P(A)\le b\,k^{-d}<V(\mu)$. The map
$r\mapsto(V(\pi)-r)/(V(\mu)-r)$ has derivative $\Delta/(V(\mu)-r)^2>0$ on $r<V(\mu)$, so
the logarithm in \cref{thm:ceiling} is at most $\ell_d$. Apply \cref{thm:ceiling} and
then \cref{thm:lower}(i). Monotonicity of $\ell_d$ in $d$ follows from the same
derivative.
\end{proof}

The corollary applies to every sound certifier, not only to importance sampling, and it
needs no power or asymptotic qualification. Recent
exponential lower bounds for off-policy \emph{evaluation} under history-dependent logging
\citep{jajoo2026hardness,luo2026revealing} use a logger whose resets erase the
information that determines the target's value. In our terms, $A$ is the reset-free
path.

\begin{corollary}[Published hard instances]
\label{cor:hard}
On the two-lane construction of \citet{jajoo2026hardness}, $P(A)=\frac13(\frac23)^{H-2}$.
On the overwrite--reveal construction of \citet{luo2026revealing}, $P(A)=4^{-(H-1)}$. In
both families, \cref{thm:ceiling} gives $\Kstar\le P(A)\log\frac32$. By
\cref{thm:lower}(i), every $\delta$-sound one-sided certifier therefore needs
$\E_P[\tau_\delta]\ge3(\frac32)^{H-2}\log(1/\delta)/\log\frac32$ and
$\E_P[\tau_\delta]\ge4^{H-1}\log(1/\delta)/\log\frac32$ episodes respectively, for every
$\delta$.
\end{corollary}

\Cref{tab:hard} (\cref{app:details}) computes $\Kstar$ exactly on both families. $\Kstar/P(A)$ converges to
about $0.13$ and $0.15$, so $\Kstar=\Theta(P(A))$ and the ceiling is tight up to a factor
of about 3. The gap $V(\pi)-V(\mu)$ stays near $1/4$, so the growth in sample size is
coverage loss, not a vanishing gap. Each paper's own two-model class is a subset of
$\Qs$. Its exponent $\KL(P_1\|P_0)$ is therefore at least $\Kstar$, but it exceeds
$\Kstar$ only by a factor of $4.3$--$5$. Knowing that the environment is one of the two
published models does not escape the rate. The hardness phenomenon itself is due to
\citet{jajoo2026hardness} and \citet{luo2026revealing}. What we add is a single
instance-wise formula that recovers both rates, for one-sided certification.

\subsection{Structural assumptions and pool selection}
\label{sec:structure}

Assumptions about the environment, for example that some actions do not affect later
outcomes (\cref{lem:irrelevance}), add linear constraints $F^\top q=0$ to $\Qs$ and can
only raise the exponent, to $\Kstar_F\ge\Kstar$. \Cref{thm:structure} in
\cref{app:structure-main} prices both sides. The gain $\Kstar_F/\Kstar$ is attained by
the restricted plug-in. The risk does not depend on how the certifier is built: any
certifier with $\E_P[\tau_\delta]<\infty$ certifies falsely at the least-favourable law
$q^\star$ with probability at least $\exp(-\Kstar\,\E_P[\tau_\delta])$, which is
$\delta^{\Kstar/\Kstar_F}$ to first order for a certifier that attains $\Kstar_F$, and
every null excluded by $F$ admits a bet that certifies falsely almost surely.
\Cref{prop:pool} in the same appendix extends the plug-in to selection from a finite pool
by a union bound.

\section{Certification without enumerating the tree}
\label{sec:local}

The feasibility constraint of \cref{thm:duality} ranges over every leaf, so it cannot be
checked when the outcome tree is too large to enumerate. This section shows that
feasibility is equivalent to a sequence of local conditions along a single episode,
driven by the superhedging price $m$ of \cref{prop:superhedge}. Any upper bound on that
price therefore gives a sound certifier computable in $O(T)$ per episode.

Extend $m$ to all nodes as in \cref{prop:superhedge}. A function $\mbar$ on nodes is a
\emph{supersolution} if $\mbar(j)=-X(j)$ at leaves and
\begin{gather*}
\mbar(h)\ge\sum_a\mu(a\mid h)\,\mbar(ha)\ \text{at action nodes},\\
\mbar(h)\ge\mbar(ho)\ \text{for every }o\ \text{at environment nodes}.
\end{gather*}
By backward induction, $m$ is the smallest supersolution.

\paragraph{The cushion game.} Fix a supersolution with $\mbar(\mathrm{root})>0$. The
gambler chooses $\lambda\in[0,1/\mbar(\mathrm{root})]$ and, at every action node, a
\emph{fair fraction} $g(h,\cdot)$ with $g\ge-1$ and $\sum_a\mu(a\mid h)g(h,a)=0$. Along an
episode, the \emph{cushion} evolves as
\begin{align*}
C(\mathrm{root})&=1-\lambda\,\mbar(\mathrm{root}),\\
C(ha)&=\bigl[C(h)+\lambda\,d(h)\bigr]\bigl(1+g(h,a)\bigr),
  &d(h)&=\mbar(h)-\textstyle\sum_{a'}\mu(a'\mid h)\,\mbar(ha')\ge0,\\
C(ho)&=C(h)+\lambda\bigl[\mbar(h)-\mbar(ho)\bigr]\ \ (\ge C(h)),
\end{align*}
and the payoff is $Y=C(\text{leaf})$.

\begin{theorem}[Cushion parametrization]
\label{thm:cushion}
\begin{enumerate}
\item[(i)] \emph{Soundness.} For every supersolution and every strategy, $Y\ge0$ on every
  leaf and $Y=1+\lambda X+B\nu$ for some $\nu$. Hence $\E_qY\le1$ for every $q\in\Qn$.
\item[(ii)] \emph{Characterization.} The payoffs of the game are exactly the feasible bets
  of \cref{thm:duality} that satisfy the node constraints $w(h)\ge\lambda\mbar(h)$ at
  every non-leaf node $h$. Here $w(h)$ is one plus the payoff of the fair bets placed
  before reaching $h$.
\item[(iii)] \emph{Exactness.} With $\mbar=m$ the node constraints are implied by leaf
  feasibility, so the game generates every feasible bet of \cref{thm:duality}.
\end{enumerate}
\end{theorem}

The proof (\cref{app:cushion}) sets $w=C+\lambda\mbar$ and checks that $w$ is constant at
environment nodes and moves by a fair bet at action nodes.

\begin{corollary}[Local exponent]
\label{cor:local}
Let $\Kloc(\mbar)=\sup\E_P\log Y$ over strategies of the game. Then $\Kloc(m)=\Kstar$ and
$\Kloc(\mbar)\le\Kstar$ for every supersolution. If $\mbar(h)\le\max_j(-X_j)$ at every
non-leaf node, the scalar bet belongs to the game, so $\Kscal\le\Kloc(\mbar)$.
\end{corollary}

$\Kstar-\Kloc(\mbar)$ is the \emph{price of locality}. The next lemma gives a supersolution
that requires no knowledge of the environment.

\begin{lemma}[Closed-form ratio supersolution]
\label{lem:ratio}
Suppose $R\in[0,b]$ and that every action node at action depth $t$ satisfies
$\pi(a\mid h)/\mu(a\mid h)\le\rho_t$ for all $a$. Let $c_t=\prod_{s\ge t}\rho_s^{-1}$, with
$c_T=1$ after the last action, and let $W(h)$ and $t(h)$ be the running likelihood ratio
and the action depth of node $h$. Then
\[
\mbar(h)=b\bigl(1-\min(W(h),c_{t(h)})\bigr)\ \ (\text{non-leaf}),\qquad \mbar(j)=-X(j)\ \ (\text{leaf})
\]
is a supersolution for every environment.
\end{lemma}

\begin{proof}
Let $\psi_t(w)=1-\min(w,c_t)$, which is convex and nonincreasing with $\psi_t(0)=1$. At an
environment node, outcomes change neither $W$ nor $t$, so non-leaf children have equal
$\mbar$. A leaf child has $-X=(1-W)R\le b(1-W)^+\le b(1-\min(W,c_t))$, because $c_t\le1$
(each $\rho_s\ge1$, since $\E_\mu[\pi/\mu]=1$). At an action node, $\varrho(a)=\pi(a\mid h)/\mu(a\mid h)$
lies in $[0,\rho_t]$ and has $\mu$-mean $1$. The function $f(r)=\psi_{t+1}(Wr)$ is
convex, so the chord bound on $[0,\rho_t]$ gives
$\sum_a\mu(a)f(\varrho(a))\le(1-\rho_t^{-1})f(0)+\rho_t^{-1}f(\rho_t)=1-\min(W,c_{t+1}/\rho_t)=\psi_t(W)$.
\end{proof}

\begin{proposition}[Scalable certifier]
\label{prop:scalable}
Let episodes have finite horizon, finite action sets and arbitrary (possibly
non-enumerable) observation spaces. Assume $\mu>0$ is known, total rewards lie in
$[0,b]$, and the ratio bounds of \cref{lem:ratio} hold. For each episode $i$, let
$(\lambda_i,g_i)$ be a valid strategy of the cushion game: $0\le\lambda_i\le
1/\mbar(\mathrm{root})$, and at every action node $g_i(h,a)\ge-1$ and
$\sum_a\mu(a\mid h)g_i(h,a)=0$. Let $(\lambda_i,g_i)$ be chosen predictably, for example
by a recurrent network fitted to episodes $1,\dots,i-1$ whose outputs are mapped into
this set. Our implementation uses $1+g_i(h,a)=e^{s(a)}/\sum_{a'}\mu(a'\mid h)e^{s(a')}$
for network scores $s$, which is fair and exceeds $-1$ by construction, and
$\lambda_i=\mathrm{sigmoid}(\cdot)/\mbar(\mathrm{root})$. Let
$Y_i$ be the terminal cushion of episode $i$ under the supersolution of \cref{lem:ratio},
and $M_n=\prod_{i\le n}Y_i$. Then $P_e(\sup_nM_n\ge1/\delta)\le\delta$ for every
environment $e$ with $V_e(\pi)\le V_e(\mu)$.
\end{proposition}

\begin{proof}
\Cref{thm:cushion}(i) uses only the per-node identities along the realized path. So
$Y_i=w_i+\lambda_iX_i$, where $w_i-1$ is a finite sum of conditionally fair bets on logged
actions, and $\E_e[Y_i\mid\text{past}]=1+\lambda_i\Delta_e\le1$. $Y_i\ge0$ holds
pathwise, so Ville's inequality applies to the nonnegative supermartingale $M_n$.
\end{proof}

Each episode costs $O(T|A|)$ beyond evaluating the strategy. The fair fractions play the
role of DR control variates, but they enter \emph{multiplicatively on the cushion}, so
feasibility never requires a range bound that grows with the horizon. An additive DR bet
$1+\lambda(X-Bc)$ must instead bound $\sup(-(X-Bc))$ over all histories. For the ratio
supersolution, the price of locality was zero on every memory tree we tested and on
192 of 200 random trees (\cref{sec:experiments}).

\section{Experiments}
\label{sec:experiments}

The experiments test three predictions. First, the stopping time of the plug-in should
scale as $\log(1/\delta)/\Kstar$. Second, the theory's comparisons should hold in
practice: scalar versus full bets, the effect of assumptions, and hard instances.
Third, the certifiers must be sound in finite samples. All code and raw outputs are
included in the supplementary material.

\paragraph{Instances.}
\begin{itemize}
\item \emph{Memory task} (cue $d$, $p$). A cue $C\sim\mathrm{Bern}(0.1)$ is shown, then
  $d$ distractor steps follow, at which $\mu$ logs uniformly over $k=2$ actions while
  $\pi$ plays a fixed action. At the final step $\mu$ logs \emph{safe} with probability
  $1/2$ and \emph{risky} with probability $p$ ($p/2$ without the cue). $\pi$ plays risky
  if and only if it saw the cue. Risky pays $1$ under the cue and $0$ otherwise, and safe
  pays $\mathrm{Bern}(0.8)$. Here $\rank(B)=4,10,22,46$ for $d=0,1,2,3$.
\item \emph{Random trees.} Depth 2--4 with 2--3 actions, Dirichlet environment and policy
  kernels, and Bernoulli terminal rewards. A quarter of reward outcomes have zero logged
  probability but remain in the tree, so certificates must stay sound on them. The
  depth-2 tree of \cref{tab:comparators} (seed 1001) has $\rank(B)=10$.
\item \emph{Tiger POMDP} \citep{kaelbling1998planning}. Listening is 85\% accurate and
  costs $0.05$. The hidden door is integrated out, so the tree is the exact logged law
  of the POMDP. The target listens until the counts differ by two. The logger is
  $\epsilon$-uniform over a memory-one policy. A unit test checks the tree's values
  against direct POMDP simulation. $\rank(B)=14,30,62$ at $H=3,4,5$.
\item \emph{Hard instances} of \citet{jajoo2026hardness} and \citet{luo2026revealing}, as
  in \cref{cor:hard}.
\item \emph{DelayedCue} (scalable setting). One of $K=4$ cues, then $L\in\{5,10,20,40\}$
  distractor steps with random symbols. At each step $\mu$ mixes copying the symbol,
  repeating its previous action and uniform play, while $\pi$ tilts toward copying. A final
  reward is paid for recalling the cue. The tree has more than $16^L$ leaves and is
  never enumerated.
\end{itemize}

\paragraph{Certifiers.}
\begin{itemize}
\item \emph{CV-Kelly}: \cref{alg:plugin}.
\item \emph{Hedged}: \cref{cor:hedge}.
\item \emph{Scalar Kelly}: \cref{alg:plugin} with $\nu=0$, whose exponent is $\Kscal$.
\item \emph{Mixture}: \cref{thm:mixture} with 300 hit-and-run points.
  This finite Monte Carlo mixture has the conditional soundness and
  positive-drift stopping-time guarantee of \cref{prop:finite-mixture},
  but not the exact-integral regret or second-order bound. The
  independent-uniform sampling probability in that proposition does
  not apply to these correlated hit-and-run points.
\item \emph{DR oracle / DR fitted}: Kelly betting on the doubly robust payoff $X-Bc$,
  where $c$ is the efficient control variate under the true law (oracle) or is refit
  predictably (fitted), in the spirit of \citet{waudbysmith2024anytime}.
\item \emph{WSR-PDIS / WSR-DR}: the betting confidence sequence of
  \citet{waudbysmith2024estimating} applied to the scalar per-episode payoff $X$ (PDIS)
  or to the DR payoff $X-Bc_t$ with $c_t$ refit predictably. The payoff range is the
  minimum and maximum over all leaves, the bet is capped at half the inverse range, and
  the wealth is the average of an aGRAPA and an ONS
  \citep{cutkosky2018black} bettor. Under every null the payoff has mean at most
  $0$, so both are $\delta$-sound. These are the strongest standard sequential
  baselines that use one scalar statistic per episode.
\item Three scalar high-confidence bounds, each checked on a doubling grid of sample
  sizes with a union bound: \emph{LS} (logarithmic smoothing; \citealp{sakhi2024logarithmic}),
  \emph{EB} (an empirical-Bernstein bound on $X$) and \emph{HCOPE} (clipped IS with a
  Maurer--Pontil bound; \citealp{thomas2015hcope}).
\end{itemize}
Every certifier runs at the same $\delta$ and is censored at $60\log(1/\delta)/\Kstar$
episodes ($30\log(1/\delta)/\Kstar$ in the $\delta$ sweep and the structural experiment).
We report median stopping times over 100 runs with 95\%
order-statistic confidence intervals, and the fraction of runs that certified when it is
below one. $\infty$ means the median run was censored. The exact program of
\cref{thm:duality} is solved by a log-barrier method with damped Newton steps to a duality
gap of $10^{-10}$. On 30 random trees it matches an independent exponential-cone solve of
the primal (Clarabel via CVXPY) to within $10^{-11}$, and the superhedging recursion
matches the LP maximal stake to relative error $10^{-15}$.

\subsection{The plug-in attains the exponent}

\Cref{fig:delta} and \cref{tab:slopes} regress the median stopping time on $\log(1/\delta)$ for $\delta$ from
$10^{-8}$ to $10^{-2}$. \Cref{thm:lower,thm:plugin} predict a slope of $1/\Kstar$. The
plug-in's slope is within 10--16\% of it, and at $d=1$ the confidence interval contains 1.
At $d=0$ the interval $[1.08,1.21]$ excludes 1. This is expected: \cref{thm:plugin} fixes
the slope only as $\delta\to0$. Over a finite range of $\delta$, the cost of estimating the
bet, about $\tfrac12\rank(B)\log t$ nats by episode $t$, grows with the stopping time and
inflates the fitted slope. The same overhead puts scalar Kelly's slope ($1.72$) above its
asymptote.
Scalar Kelly matches its own, smaller exponent: $\Kstar/\Kscal=1.57$ and $2.47$. The Monte
Carlo mixture is far from its asymptote at $d=1$, where $\rank(B)=10$. It certified in only
47--62\% of runs, as its $d\log T$ penalty predicts. The hedge pays its $\log2$, at most
20 extra episodes, but keeps the plug-in's slope.
The same sweep on the depth-2 random tree and on Tiger ($H=3$, $\epsilon=0.2$; censoring at
$120\log(1/\delta)/\Kstar$ or $2\times10^5$ episodes, whichever is smaller; no run censored) gives plug-in slopes of $1.10$ $[0.98,1.25]$ and
$1.17$ $[1.05,1.29]$, and hedge slopes whose intervals contain 1. The scalar betting
certificates have slopes near their own asymptote $\Kstar/\Kscal$: scalar Kelly
$3.19$ and $4.22$, WSR-PDIS $2.92$ $[2.61,3.22]$ and $4.12$ $[3.72,4.51]$, against
$\Kstar/\Kscal=2.70$ and $4.05$. Their gap to the plug-in therefore grows linearly in
$\log(1/\delta)$, as \cref{thm:invariance}(b) predicts. WSR-DR has slopes $2.05$ and $3.78$.

\begin{table}[t]
\centering
\caption{Slope of the median stopping time against $\log(1/\delta)$, in units of
$1/\Kstar$ (theory: $1$ for CV-Kelly and the hedge by its plug-in component,
$\Kstar/\Kscal$ for scalar Kelly; no finite-sample bound for the MC mixture).
Brackets are 95\% percentile-bootstrap intervals over 2000 resamples of the 100 runs per
$\delta$.}
\label{tab:slopes}
\small
\resizebox{\linewidth}{!}{%
\begin{tabular}{lccccc}
\toprule
instance & CV-Kelly & hedged & scalar Kelly & mixture (MC) & $\Kstar/\Kscal$\\
\midrule
cue $d=0$, $p=0.2$ & 1.16 [1.08, 1.21] & 1.13 [1.03, 1.18] & 1.72 [1.55, 1.85] & 1.39 [1.28, 1.51] & 1.57\\
cue $d=1$, $p=0.2$ & 1.10 [1.00, 1.19] & 1.21 [1.08, 1.31] & 2.39 [2.17, 2.68] & 21.8 [9.4, $\infty$] & 2.47\\
random, depth 2 & 1.10 [0.98, 1.25] & 1.08 [0.93, 1.20] & 3.19 [2.84, 3.46] & --- & 2.70\\
Tiger $H=3$, $\epsilon=0.2$ & 1.17 [1.05, 1.29] & 1.02 [0.92, 1.12] & 4.22 [3.74, 4.64] & --- & 4.05\\
\bottomrule
\end{tabular}}
\end{table}

\begin{figure}[t]
\centering
\includegraphics{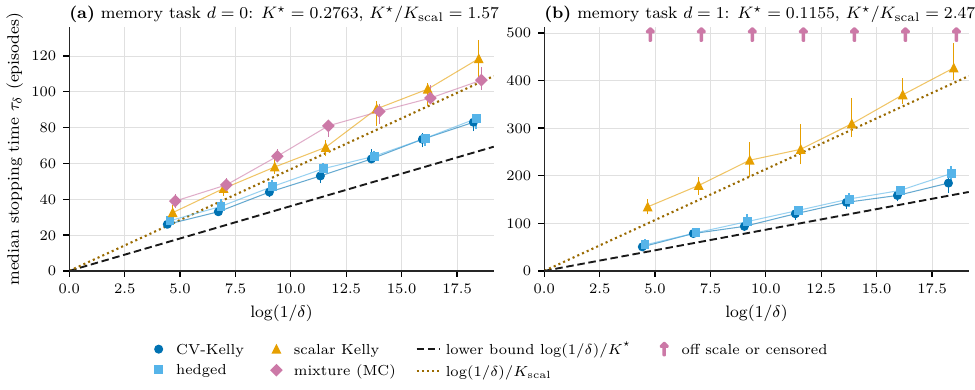}
\caption{Median stopping time against $\log(1/\delta)$ for $\delta$
from $10^{-8}$ to $10^{-2}$; 100 runs per point, bars are 95\% order-statistic intervals,
and points are offset horizontally for legibility. Dashed: the lower bound
$\log(1/\delta)/\Kstar$ on the expected stopping time (\cref{thm:lower}). Dotted:
$\log(1/\delta)/\Kscal$, the
asymptote of scalar Kelly. The plotted mixture uses 300 fixed sampled points, not
the exact integral in \cref{thm:mixture}. Arrows: median above the plotted range or censored at
$30\log(1/\delta)/\Kstar$; at $d=1$ the mixture certified in only 47--62\% of runs. Fitted
slopes are in \cref{tab:slopes}.}
\label{fig:delta}
\end{figure}

\subsection{Comparison with existing certificates}

\Cref{tab:comparators} and \cref{fig:efficiency}(a) compare all certifiers at $\delta=10^{-3}$.
\Cref{thm:invariance}(b) says that no sound certificate built on a per-episode statistic
$S$ can beat $\log(1/\delta)/K_S$ in expected stopping time. It is a limit on every such
certificate, not a guarantee that a particular scalar bound reaches it, and the gaps below
are measured against $\log(1/\delta)/\Kstar$ for all methods.
The mixture results here use 300 sampled points and do not carry
\cref{thm:mixture}'s exact-integral finite-sample bounds. CV-Kelly stops within
$1.13$--$1.43\times$ the lower bound on memory trees
and $1.55\times$ on the random tree.
It is the fastest, or tied with the hedge within the confidence intervals, on every memory
tree. The 95\% intervals for its median have half-widths of 8--18\%, for example
$[67,90]$ for cue $d=1$, $p=0.2$. On the shallow random tree, the oracle and fitted DR bets
are 18\% and 7\% faster: there a DR control variate is nearly optimal, and CV-Kelly pays
for plug-in estimation. The strongest baselines are the scalar betting certificates.
Scalar Kelly needs $1.05$--$2.65\times$ the episodes of CV-Kelly on memory trees, and
the WSR confidence sequence on the PDIS payoff needs $1.49$--$2.79\times$. Both bet on
$X$ alone, so both are limited by $\Kscal$ (\cref{thm:invariance}(b)), and their stopping
times are close to each other. WSR on the DR payoff is slower on memory trees
($2.7$--$27\times$), because its payoff range must cover $X-Bc$ over every history, but on
the random tree it is the fastest scalar method, at $1.26\times$ CV-Kelly. Against the best
scalar betting baseline on each memory tree, CV-Kelly's advantage is therefore
$1.05$--$2.65\times$, not the $9$--$52\times$ by which the doubling-grid bounds trail, when they
certify at all. \Cref{sec:random-comparators} repeats the comparison on 24 random trees.
The memory penalty of
\cref{cor:memory} is visible across rows: each step of delay multiplies this bound
by $2.1$--$2.5$, and CV-Kelly's stopping time by $1.6$--$3.0$.

\begin{table}[t]
\centering
\caption{Median stopping time (episodes) at $\delta=10^{-3}$, 100 runs per cell.
$\log(1/\delta)/\Kstar$ is the lower bound of \cref{thm:lower}(i) on the expected
stopping time of every $\delta$-sound certifier.
WSR: the betting confidence sequence of \citet{waudbysmith2024estimating} on the
PDIS or DR payoff (hedged aGRAPA/ONS bettor). The mixture column uses 300 sampled
points. A number in parentheses is the certification rate when below one.}
\label{tab:comparators}
\small
\setlength{\tabcolsep}{3pt}
\resizebox{\linewidth}{!}{%
\begin{tabular}{lrrrrrrrrrrrr}
\toprule
instance & $\frac{\log(1/\delta)}{\Kstar}$ & CV-Kelly & hedged & scalar & WSR-PDIS & WSR-DR & mixture & DR fit & DR orc & LS & EB & HCOPE\\
\midrule
cue $d{=}0$ $p{=}0.2$ & 25 & 31 & 37 & 52 & 57.5 & 114.5 & 47.5 & 72.5 & 83 & 1024 & 512 & 1024 (.91)\\
cue $d{=}0$ $p{=}0.02$ & 34 & 49 & 49 & 51.5 & 73 & 1337 (.98) & 99.5 (.90) & 582.5 & 862 & 1024 & $\infty$ & 2048\\
cue $d{=}1$ $p{=}0.2$ & 60 & 75.5 & 86 & 182 & 182.5 & 252.5 & 703.5 (.63) & 159 & 179 & 1024 & 2048 & 2048 (.87)\\
cue $d{=}1$ $p{=}0.02$ & 70 & 79.5 & 103.5 & 210.5 & 222 & 2168 & 347 (.79) & 772 & 1726 & 2048 & $\infty$ & 4096\\
cue $d{=}2$ $p{=}0.2$ & 151 & 200 & 194.5 & 441.5 & 391.5 & 545.5 & $\infty$ & 387 & 377 & 2048 & 4096 & 4096\\
cue $d{=}2$ $p{=}0.02$ & 178 & 239 & 242.5 & 507.5 & 550 & 2852 & $\infty$ & 908.5 & 3429 & 2048 & $\infty$ & 4096\\
random, depth 2 & 2349 & 3647 & 3437 & 8473.5 & 7563 & 4586.5 & $\infty$ & 3374 & 2999.5 & $\infty$ & 32768 & 65536\\
\bottomrule
\end{tabular}}
\end{table}

\begin{figure}[t]
\centering
\includegraphics{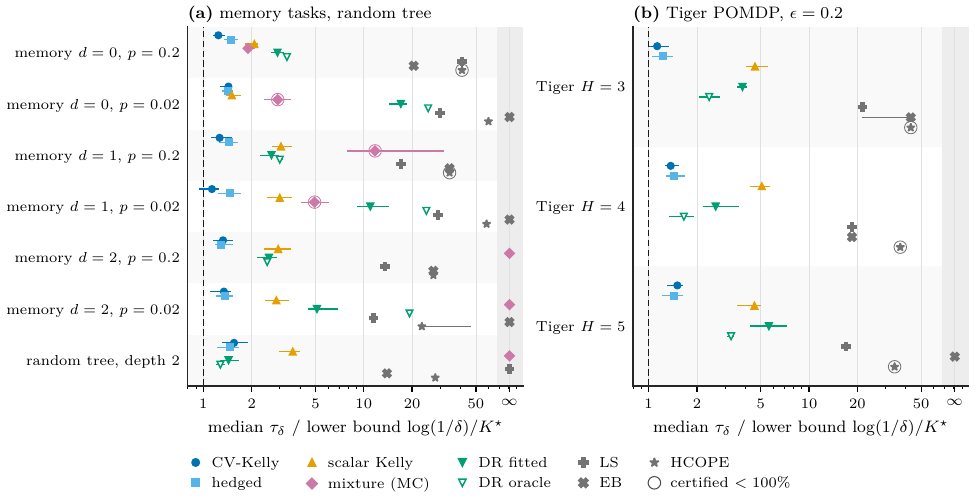}
\caption{Median stopping time at $\delta=10^{-3}$ divided by the
lower bound $\log(1/\delta)/\Kstar$ (dashed line at 1);
100 runs per cell. The mixture uses 300 sampled points. Horizontal
bars are 95\% intervals for the median, truncated at the censoring level
$60\log(1/\delta)/\Kstar$. Points in the shaded $\infty$ column: the median run was
censored. Circled points: certified in fewer than 100\% of runs. (a)~The instances of
\cref{tab:comparators}. (b)~The Tiger POMDP at $\epsilon=0.2$ (\cref{tab:tiger}).}
\label{fig:efficiency}
\end{figure}

\subsection{Comparison on 24 random trees}
\label{sec:random-comparators}

To check that the comparison does not depend on the instances chosen, we ran every
certifier on the first 24 random trees of the exponent study's seed schedule with
$\Kstar\ge0.005$ and at most 300 leaves (16 of depth 2 and 8 of depth 3, with 32--128 leaves
and $\rank(B)\in\{10,28,42\}$; every depth-4 tree exceeded the leaf cap). The cap keeps a
CV-Kelly run under a minute. \Cref{tab:random-comparators} summarizes 100 runs per tree and
certifier at $\delta=10^{-3}$. CV-Kelly stops within $1.17$--$1.64\times$ the lower bound on
every tree (median $1.38$). Scalar Kelly and WSR-PDIS are never faster. The best
scalar betting baseline on each tree needs $0.93$--$3.82\times$ CV-Kelly's episodes
(median $1.89$). A DR-type bet is faster than CV-Kelly on 3 of the 24 trees, by at most
$7\%$, and in each case the two 95\% intervals overlap. DR bets are also the least
reliable: their median ratio to CV-Kelly is $3.2$--$7.0$, and each is censored on one tree.
The doubling-grid bounds need $10$--$54\times$ the lower bound or fail.

\begin{table}[t]
\centering
\caption{24 random trees (depth 2--3), $\delta=10^{-3}$, 100 runs per tree and certifier.
Left: minimum, median and maximum over trees of the median stopping time divided by
$\log(1/\delta)/\Kstar$. Middle: the same for the ratio to CV-Kelly's median on the same tree.
``faster'': trees on which the median is below CV-Kelly's. ``finite'': trees on which the
median run certified within the horizon $60\log(1/\delta)/\Kstar$.}
\label{tab:random-comparators}
\small
\begin{tabular}{lrrrrrrrr}
\toprule
& \multicolumn{3}{c}{median $/$ lower bound} & \multicolumn{3}{c}{median $/$ CV-Kelly} & & \\
\cmidrule(lr){2-4}\cmidrule(lr){5-7}
certifier & min & median & max & min & median & max & faster & finite\\
\midrule
CV-Kelly & 1.17 & 1.38 & 1.64 & --- & --- & --- & --- & 24\\
scalar Kelly & 1.52 & 2.80 & 4.46 & 1.16 & 1.89 & 3.82 & 0 & 24\\
WSR-PDIS & 1.64 & 3.07 & 5.03 & 1.25 & 2.20 & 4.30 & 0 & 24\\
WSR-DR & 1.47 & 9.91 & $\infty$ & 0.93 & 6.96 & $\infty$ & 1 & 23\\
DR fit & 1.53 & 7.39 & $\infty$ & 0.96 & 5.21 & $\infty$ & 1 & 23\\
DR orc & 1.27 & 4.39 & $\infty$ & 0.93 & 3.16 & $\infty$ & 2 & 23\\
LS & 14.0 & 24.7 & 53.8 & 10.5 & 18.0 & 37.9 & 0 & 24\\
EB & 10.0 & $\infty$ & $\infty$ & 6.33 & $\infty$ & $\infty$ & 0 & 11\\
HCOPE & 20.0 & 48.3 & $\infty$ & 12.7 & 35.3 & $\infty$ & 0 & 19\\
\bottomrule
\end{tabular}
\end{table}

\subsection{Computational cost}
\label{sec:cost}

We also measure the computational cost of each certifier. \Cref{tab:cost} in \cref{app:matched} times every certifier of
\cref{tab:comparators} on one CPU thread (20 runs per cell, a fresh run of each method
per seed). The exact plug-in refits at $r_k=\lceil1.15^k\rceil$. A run that stops at
episode $T$ therefore solves $O(\log T)$ leaf-constrained programs (18--50 in these runs),
each a logged-law projection followed by the Kelly program of \cref{thm:duality}. On the
memory trees a CV-Kelly run takes $0.9$--$2.3$\,s. Scalar Kelly takes $0.09$--$0.44$\,s,
and the doubling-grid bounds take at most $0.007$\,s. So CV-Kelly uses $2$--$23\times$ the
wall clock of scalar Kelly for $1.05$--$2.65\times$ fewer episodes. The WSR certifiers cost about $0.05$\,s per run on a 72-leaf random tree, against
about $3$\,s for CV-Kelly. Peak traced memory is
at most $2.2$\,MiB for the plug-in and $19$\,MiB for the 300-point mixture. The plug-in is therefore best suited to settings where episodes are expensive, such as
logged deployment data.

One refit costs $O(|\J|r^2+r^3)$ with $r=\rank(B)$: at most $1.9$\,s up to 512 leaves
but $210$\,s at 2048 leaves (\cref{tab:scaling} in \cref{app:matched}). Exact certification is practical up to a few thousand leaves; for larger trees the
cushion certifier of \cref{sec:local} replaces the global program, at a per-episode cost
independent of the size of the tree.

\subsection{Soundness}
\label{sec:soundness}

\paragraph{Least-favourable nulls on 21 instances.} Each certifier was run under the
least-favourable law $q^\star$ of 21 instances: the memory tree with $d=0$, 16 random trees
(the first 16 of \cref{sec:random-comparators}) and Tiger with $H\in\{3,4\}$ and
$\epsilon\in\{0.05,0.2\}$. The horizon was 1000 episodes. Each cell used 1000 runs at
$\delta=0.05$ and 2000 at $\delta=10^{-3}$, except the four certifiers that refit the exact
program, which used 300. \Cref{tab:soundness} reports the rate pooled over the 21
instances, with its 95\% Clopper--Pearson interval, and the worst instance. Every pooled rate
is below $\delta$, and its upper confidence limit is below $\delta$ except for CV-Kelly with
the raw empirical plug-in at $\delta=10^{-3}$ (2 of 6300 runs, upper limit $1.1\times10^{-3}$). At
$\delta=0.05$ the worst single instance reaches $0.050$ (the hedge, 15 of 300 runs). No
instance's interval lies entirely above $\delta$: the largest lower limit over all certifiers
and instances is $0.035$ at $\delta=0.05$ and $3.1\times10^{-4}$ at $\delta=10^{-3}$. These
worst-instance figures are maxima over 21 cells and so are biased upwards. The exact-program
certifiers run closest to $\delta$, as expected for tests that are nearly growth-optimal
at the boundary null. The doubling-grid bounds never certified falsely, consistent with
their conservatism in \cref{tab:comparators}.

\begin{table}[t]
\centering
\caption{False certification under the least-favourable law $q^\star$ of 21 instances
(1 memory tree, 16 random trees, 4 Tiger trees), horizon 1000. ``runs'': per instance.
Pooled: rate over all instances with 95\% Clopper--Pearson interval. Worst: the
largest per-instance rate, with its interval at $\delta=0.05$. The four exact-program
certifiers used 300 runs per instance, so at $\delta=10^{-3}$ a single certification gives
a per-instance rate of $0.003$.}
\label{tab:soundness}
\small
\setlength{\tabcolsep}{3.5pt}
\resizebox{\linewidth}{!}{%
\begin{tabular}{lrllrlr}
\toprule
& \multicolumn{3}{c}{$\delta=0.05$} & \multicolumn{3}{c}{$\delta=10^{-3}$}\\
\cmidrule(lr){2-4}\cmidrule(lr){5-7}
certifier & runs & pooled & worst & runs & pooled & worst\\
\midrule
CV-Kelly & 300 & 0.025 [0.021, 0.029] & 0.040 [0.021, 0.069] & 300 & 0/6300 [0, 5.9e-4] & 0\\
CV-Kelly (raw plug-in) & 300 & 0.030 [0.025, 0.034] & 0.043 [0.023, 0.073] & 300 & 2/6300 [3.8e-5, 1.1e-3] & 0.003\\
hedged & 300 & 0.033 [0.029, 0.038] & 0.050 [0.028, 0.081] & 300 & 1/6300 [4.0e-6, 8.8e-4] & 0.003\\
mixture (MC) & 300 & 0.032 [0.028, 0.037] & 0.047 [0.026, 0.077] & 300 & 1/6300 [4.0e-6, 8.8e-4] & 0.003\\
scalar Kelly & 1000 & 0.018 [0.016, 0.020] & 0.031 [0.021, 0.044] & 2000 & 12/42000 [1.5e-4, 5.0e-4] & 0.001\\
WSR-PDIS & 1000 & 0.018 [0.016, 0.020] & 0.029 [0.020, 0.041] & 2000 & 22/42000 [3.3e-4, 7.9e-4] & 0.002\\
WSR-DR & 1000 & 0.007 [0.006, 0.008] & 0.034 [0.024, 0.047] & 2000 & 4/42000 [2.6e-5, 2.4e-4] & 0.001\\
DR fit & 1000 & 0.020 [0.018, 0.022] & 0.039 [0.028, 0.053] & 2000 & 6/42000 [5.2e-5, 3.1e-4] & 0.001\\
DR orc & 1000 & 0.021 [0.019, 0.023] & 0.036 [0.025, 0.049] & 2000 & 13/42000 [1.6e-4, 5.3e-4] & 0.001\\
LS, EB, HCOPE & 1000 & 0 [0, 1.8e-4] & 0 [0, 0.004] & 2000 & 0/42000 [0, 8.8e-5] & 0\\
\bottomrule
\end{tabular}}
\end{table}

\paragraph{Other vertices of the null.} $q^\star$ is one point of $\Qn$. On a memory tree and
a random tree we also ran each certifier at $\delta=0.05$ under 10 vertices of $\Qn$ obtained
from random linear objectives (1000 runs per law, horizon 1000). The worst rate over these
22 laws was at most $0.048$ $[0.036, 0.063]$ for every certifier, and the average over laws
was $0.009$--$0.025$ for the betting certifiers (\cref{fig:soundness}(a)).

\begin{figure}[t]
\centering
\includegraphics{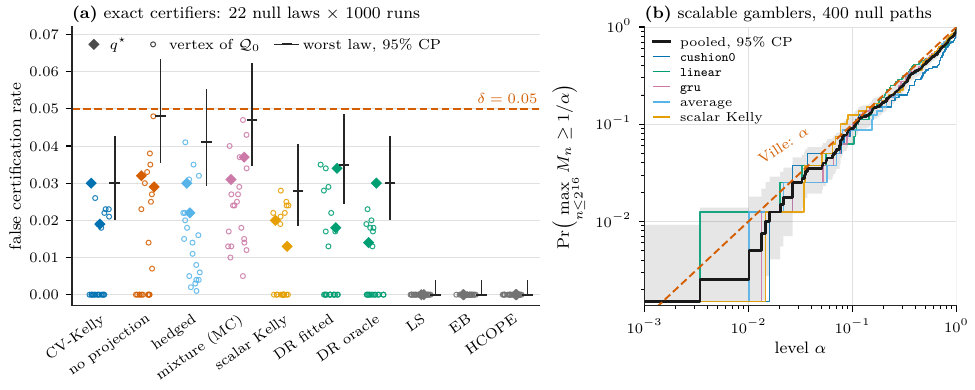}
\caption{Soundness. (a)~False certification rate at $\delta=0.05$
under 22 null laws: for each of a memory tree and a random tree, $q^\star$ (filled
diamonds) and 10 vertices of $\Qn$ (open circles), with 1000 runs per law and a horizon
of 1000 episodes. Points are jittered horizontally. Black bars: the worst law for each
certifier with its 95\% Clopper--Pearson interval. (b)~Scalable certifiers at the
DelayedCue boundary null ($\Delta=0$): the fraction of 400 wealth paths (80 null runs
$\times$ 5 Kelly-type gamblers, horizon $2^{16}$) whose maximum reached $1/\alpha$. Ville's
inequality bounds it by $\alpha$ (dashed). Grey band: 95\% Clopper--Pearson interval for
the pooled curve. Zero frequencies are drawn at the lower edge.}
\label{fig:soundness}
\end{figure}

\subsection{The curse of memory and the Tiger POMDP}

Over 40 memory trees ($k\in\{2,3,4\}$, $d\le8$, up to 3072 leaves), $\Kstar$ depends on
$(k,d)$ only through the number of cue histories $k^d$ (\cref{fig:memory}(a)). For example, $(k,d)=(2,2)$ and
$(4,1)$ agree to every digit. $\Kstar k^d$ decreases to $0.151$ ($p=0.2$) and $0.128$
($p=0.02$), so $\Kstar=\Theta(k^{-d})$ in this family, consistent with the upper
bound of \cref{cor:memory}. For $d\ge1$, the
coverage ceiling is $2.05$--$2.53\times\Kstar$, and $\Kscal/\Kstar$ lies in $0.40$--$0.45$.

\Cref{tab:tiger} and \cref{fig:memory}(b) show that the same behaviour arises in a standard partially observable
benchmark rather than a hand-built instance. Tiger's target needs memory of the listening
counts, and $\Kstar$ falls by a factor of 3--5 from $H=3$ to $H=8$. Across all twelve Tiger
instances, $\Kscal/\Kstar=0.20$--$0.26$, so the fair action bets are worth a factor of
about 4. The oracle DR bet reaches only $0.04$--$0.72$ of $\Kstar$, and the ratio
supersolution loses less than $1.5\%$ ($\Kloc/\Kstar\ge0.985$). These exponent statements
are exact computations on all twelve trees. The stopping-time experiments are narrower: at
$\epsilon=0.2$ and $H\in\{3,4,5\}$, where the exact plug-in runs in reasonable time,
CV-Kelly and the hedge stop within $1.13$--$1.51\times$ the lower bound. Relative to
CV-Kelly, scalar Kelly needs
$3.0$--$4.1\times$ as many episodes, WSR-PDIS $3.3$--$4.4\times$, WSR-DR
$2.8$--$8.1\times$, the DR Kelly bets $1.2$--$3.7\times$, and the grid bounds
$11$--$38\times$ or fail (\cref{fig:efficiency}(b)).
Under $q^\star$, CV-Kelly and the hedge certified in 0 of 100 runs at each of these $H$
(Clopper--Pearson upper limit $3.6\%$).
At $H=6$--$8$ (284--1148 leaves, $\rank(B)$ up to 510) a single CV-Kelly run did not finish
in 15 minutes, so we ran only the certifiers without a global refit, for both $\epsilon$
(100 runs per cell). The best practical scalar certificate, scalar Kelly or WSR-PDIS,
stops at $4.5$--$6.6\times$ the lower bound $\log(1/\delta)/\Kstar$, consistent with
$\Kstar/\Kscal\approx4$--$5$ on these trees. The oracle DR bet, which needs the true law,
ranges from $1.7\times$ to $22\times$, and the grid bounds need $16$--$55\times$ or fail.
On the largest Tiger trees the gap between $\Kstar$ and what scalar certificates attain is
therefore at least as large as at $H\le5$.

\begin{table}[t]
\centering
\caption{Tiger POMDP. Left: exact exponents for two logger exploration rates $\epsilon$.
Right: median stopping time and 95\% interval at $\epsilon=0.2$, $\delta=10^{-3}$, 100
runs per cell.}
\label{tab:tiger}
\small
\setlength{\tabcolsep}{4pt}
\begin{tabular}{rrcc}
\toprule
$H$ & leaves & $\Kstar$ ($\epsilon{=}0.05$) & $\Kstar$ ($\epsilon{=}0.2$)\\
\midrule
3 & 32 & 0.0289 & 0.0363\\
4 & 68 & 0.0266 & 0.0312\\
5 & 140 & 0.0087 & 0.0144\\
6 & 284 & 0.0080 & 0.0132\\
7 & 572 & 0.0058 & 0.0112\\
8 & 1148 & 0.0055 & 0.0109\\
\bottomrule
\end{tabular}
\hfill
\begin{tabular}{rrlll}
\toprule
$H$ & $\frac{\log(1/\delta)}{\Kstar}$ & CV-Kelly & hedged & scalar\\
\midrule
3 & 190 & 215 [192, 252] & 234.5 [202, 269] & 875.5 [776, 1045]\\
4 & 221 & 303.5 [282, 341] & 316.5 [288, 367] & 1123 [958, 1255]\\
5 & 480 & 725 [632, 782] & 690 [590, 779] & 2190.5 [1712, 2379]\\
\bottomrule
\end{tabular}
\end{table}

\begin{figure}[t]
\centering
\includegraphics{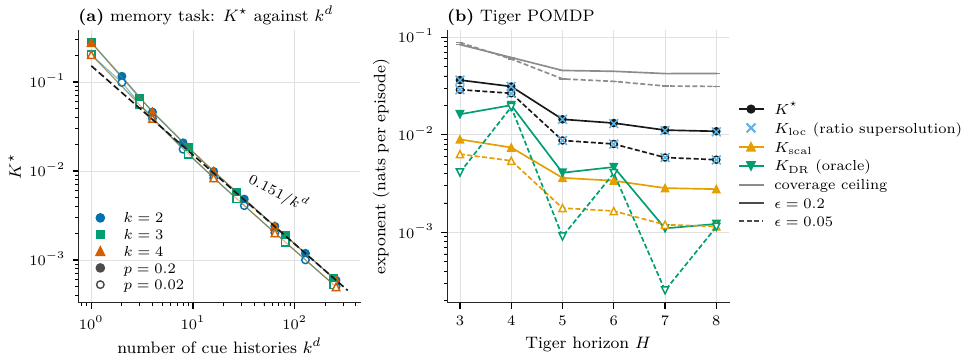}
\caption{(a)~$\Kstar$ of the memory task against the number of cue
histories $k^d$ (filled: $p=0.2$; open: $p=0.02$). Trees with equal $k^d$ coincide, and
$\Kstar k^d$ decreases to $0.151$ at $p=0.2$ (dashed), the $\Theta(k^{-d})$ rate of
\cref{cor:memory}. (b)~Exact exponents of the Tiger POMDP against the horizon $H$, for
logger exploration $\epsilon=0.2$ (solid) and $0.05$ (dashed): $\Kstar$, the cushion
exponent $\Kloc$ with the ratio supersolution (crosses, on top of $\Kstar$), $\Kscal$,
the oracle DR exponent and the coverage ceiling of \cref{thm:ceiling}.}
\label{fig:memory}
\end{figure}

\subsection{Pricing a structural assumption}

On the memory task, the assumption that distractor actions are irrelevant restores the
$d=0$ exponent at every delay ($\Kstar_F/\Kstar=2.39$, $6.03$, $13.35$ at $d=1,2,3$), and
the restricted certifier's stopping time is flat in $d$. Under $q^\star$, which violates
the assumption, it certifies falsely in every run. \Cref{thm:structure}(ii) forces a
false-certification rate of at least $1.4\%$, $19\%$ and $47\%$ on any certifier that
fast (\cref{app:structure-exp}, \cref{tab:structure}).

\subsection{Random trees: locality and the Gaussian approximation}

On 200 random trees (\cref{app:random}), the ratio supersolution has zero price of
locality on 192; the eight exceptions are depth-2 trees with a large advantage (worst
$\Kloc/\Kstar=0.69$). The median of $\Kscal/\Kstar$ is $0.37$, and the oracle DR bet's
exponent falls from $0.49\Kstar$ at depth 2 to $0.10\Kstar$ at depth 4. At each tree's own
advantage, $\Kstar$ differs from the efficiency approximation $\Delta^2/2\sigeff$ by a
factor between $0.014$ and $19.4$ (\cref{fig:random}), so \cref{thm:local} is a
small-advantage statement only.

\subsection{Scalable certification}

\Cref{tab:scalable} and \cref{fig:scalable} report the cushion certifier of \cref{prop:scalable} on DelayedCue,
with $\delta=10^{-3}$ and a horizon of $2^{16}$ episodes, on a T4 GPU, over 30 alternative
runs per cell. With 30 runs the 95\% order-statistic interval for a median spans the 10th
to 21st order statistics (shown in \cref{fig:scalable}). For \texttt{cushion0} it extends
$11$--$32\%$ below and $16$--$63\%$ above the median (unbounded above at $L=40$,
$\Delta=0.02$), so differences of less than about $1.5\times$ between gamblers in a single
cell are within the interval. The gamblers are
\texttt{cushion0} ($\lambda$ only, $g=0$), \texttt{linear} (fair fractions linear in
history features, including $\log(\pi/\mu)$) and \texttt{gru} (a recurrent head).
These large-tree results cannot be normalized by $\Kstar$, since that exponent
requires enumerating the DelayedCue tree. The matched study at the end of this section
does so on enumerable DelayedCue trees. None of
the 80 null runs (20 per $L$) certified. Pooled over $L$ and the five Kelly-type gamblers
(the three above, their average and scalar Kelly), the frequency with which wealth ever
crossed $1/\alpha$ within the horizon was $0.43$, $0.17$, $0.09$ and $0.035$ at
$\alpha=0.5$, $0.2$, $0.1$, $0.05$ (400 wealth paths, \cref{fig:soundness}(b); 95\% Clopper--Pearson intervals
$[0.38,0.48]$, $[0.13,0.21]$, $[0.06,0.12]$, $[0.02,0.06]$). These are close to Ville's
bound $\alpha$, so the tests are not conservative at the boundary null. \texttt{cushion0} is the
fastest single gambler in every cell, although the average of the three gamblers is
faster at $L=20$, $\Delta=0.05$ ($5433$ against $6286$). It beats scalar Kelly by $1.8$--$2.3\times$ at
$L\le10$ and by $1.2$--$1.8\times$ at $L=20$ (ratios of median stopping times, scalar Kelly
over \texttt{cushion0}; these are not comparable to the ratios to the lower bound
$\log(1/\delta)/\Kstar$ reported for the matched instances below, which use a different
denominator), and the grid bounds never certify reliably.
Stopping times grow with $L$ for every certifier, which is the scalable counterpart of the
curse of memory. The learned tilts lose online.
On held-out data the linear tilt has the highest growth rate, but refitting it from few
episodes costs more wealth than its higher rate recovers within $2^{16}$ episodes.
\Cref{cor:hedge} removes the need to choose a gambler in advance. Averaging the three
gamblers' wealths costs at most $\log3$ nats against the best of them. In practice the
average's median is within the 95\% interval of \texttt{cushion0}'s in every cell. The
choice matters run by run: at $\Delta=0.05$, \texttt{linear} or \texttt{gru} crossed
$1/\delta$ before \texttt{cushion0} in 10--14 of the 30 runs at every $L$, and the average
beat \texttt{cushion0} in 8--12 of them.

\begin{table}[t]
\centering
\caption{Scalable certification on DelayedCue at $\delta=10^{-3}$: median episodes over
30 alternative runs (certification rate in parentheses when below one). $\infty$:
censored at $2^{16}$. \emph{average} is the mean of the three gamblers' wealths
(\cref{cor:hedge}).}
\label{tab:scalable}
\small
\setlength{\tabcolsep}{4pt}
\resizebox{\linewidth}{!}{%
\begin{tabular}{rrrrrrrrrr}
\toprule
$L$ & $\Delta$ & \texttt{cushion0} & \texttt{linear} & \texttt{gru} & average & scalar Kelly & EB & LS & HCOPE\\
\midrule
5 & 0.02 & 14858 & 56080 (.60) & $\infty$ (.20) & 18053 & 33977 (.83) & 65536 (.67) & $\infty$ (0) & $\infty$ (0)\\
5 & 0.05 & 2462 & 7046 & 28608 (.93) & 2552 & 5715 & 16384 (.70) & $\infty$ (.03) & $\infty$ (.23)\\
10 & 0.02 & 19502 (.93) & $\infty$ (.37) & $\infty$ (.03) & 22694 (.93) & 41110 (.87) & 65536 (.53) & $\infty$ (0) & $\infty$ (0)\\
10 & 0.05 & 3510 & 6682 & 46493 (.83) & 3470 & 6384 & 16384 (.77) & $\infty$ (.13) & $\infty$ (.37)\\
20 & 0.02 & 26996 (.87) & $\infty$ (.20) & $\infty$ (.17) & 27440 (.83) & 48838 (.70) & $\infty$ (.03) & $\infty$ (0) & $\infty$ (0)\\
20 & 0.05 & 6286 & 11804 & 8814 (.77) & 5433 & 7648 & $\infty$ (.40) & $\infty$ (.13) & $\infty$ (.23)\\
40 & 0.02 & 55092 (.60) & $\infty$ (.03) & $\infty$ (0) & 58620 (.53) & $\infty$ (.50) & $\infty$ (0) & $\infty$ (0) & $\infty$ (0)\\
40 & 0.05 & 8045 & 15609 & 56642 (.53) & 8753 & 10258 & $\infty$ (0) & $\infty$ (.10) & $\infty$ (.40)\\
\bottomrule
\end{tabular}}
\end{table}

\paragraph{Matched small instances.}
\label{sec:matched}
On four enumerable DelayedCue trees with 32--512 leaves (\cref{app:matched},
\cref{tab:matched}), the stopping time in units of the lower bound factors into a
supersolution term $\Kstar/\Kloc(\mbar)$, a gambler-class term and an online-learning
term. The supersolution term is $1$ to six digits on all four trees, so the closed-form
supersolution is lossless there and the whole gap comes from the gamblers: the class
term is $1.33$--$1.48$ for \texttt{cushion0} and the online term $1.3$--$1.6$. Overall,
\texttt{cushion0} stops at $1.7$--$2.2\times$ the lower bound and CV-Kelly at
$1.6$--$3.2\times$, with overlapping 95\% intervals. A CV-Kelly run takes $176$--$231$\,s
at 512 leaves, against $40$--$100$\,s for the scalable runner with all gamblers.

\section{Conclusion}
\label{sec:conclusion}

We showed that certifying policy improvement from history-dependent logs with a known
logger has an exact instance-dependent exponent, $\Kstar$. No sound certifier stops in
fewer than $\log(1/\delta)/\Kstar$ expected episodes at any $\delta$, and whenever
$0<\Kstar<\infty$ a plug-in Kelly certifier attains this rate to first order as
$\delta\to0$. The optimal certificate bets on the logged actions as well as on the
advantage, which explains why refinements of the off-policy estimator cannot change the
exponent, why memory forces exponential sample sizes, and how much a structural
assumption is worth. Coverage enters through the superhedging price and decides whether
certification is impossible, possible at a finite rate, or possible in a constant number
of episodes. In our experiments the plug-in's median stopping time was within
$1.13$--$1.64\times$ the lower bound (Tiger with $H\le5$), and the cushion certifier
extends the approach to trees too large to enumerate.

A natural next step is to bring the scalable certifier closer to $\Kstar$, for example by
learning the supersolution or the pool of gamblers from data. Extending the framework to
loggers that are unknown and must be estimated would further broaden the range of
deployments in which such certificates can be used.

\appendix

\section{Overview of the appendix}
\label{app:overview}

The appendix has two parts. Part~I proves the results of the main text in the order in
which they appear there: the lower bound and the duality theorem (\cref{sec:exponent}),
the coverage regimes and the local expansion (\cref{sec:coverage}), the plug-in and
mixture certifiers (\cref{sec:achieve}), structural assumptions (\cref{sec:structure})
and local certification (\cref{sec:local}). Each proof opens with a short overview and is
split into named steps. Part~II collects the additional experiments and the
implementation details. \Cref{tab:appx-map} maps each result to its proof, and
\cref{tab:notation} summarizes the notation. Results that are not listed in
\cref{tab:appx-map} are proved in the main text.

\begin{table}[h]
\centering
\small
\caption{Where each result is proved.}
\label{tab:appx-map}
\begin{tabular}{p{0.52\linewidth}p{0.40\linewidth}}
\toprule
Result & Proof\\
\midrule
\Cref{thm:lower} (lower bound) & \Cref{app:lower}\\
\Cref{thm:duality} (duality) & \Cref{app:duality}\\
\Cref{prop:trichotomy} (coverage trichotomy) & \Cref{app:coverage}\\
\Cref{prop:positivity} (positivity under partial coverage) & \Cref{app:coverage}\\
\Cref{thm:local} (local expansion) & \Cref{app:local}\\
\Cref{prop:sandwich} (explicit local sandwich) & \Cref{app:local}\\
\Cref{thm:plugin} (plug-in certifier), with \cref{lem:continuity,lem:projection} & \Cref{app:plugin}\\
\Cref{thm:plugin-explicit} (explicit plug-in bounds) & \Cref{app:plugin}\\
\Cref{thm:mixture} (universal mixture) & \Cref{app:mixture}\\
\Cref{prop:finite-mixture} (finite-point mixture) & \Cref{app:mixture}\\
\Cref{lem:irrelevance} and \cref{thm:structure} (structural assumptions) & Stated in \cref{app:structure-main}, proved in \cref{app:structure}\\
\Cref{prop:pool} (pool selection) & \Cref{app:structure-main}\\
\Cref{thm:cushion} (cushion parametrization) & \Cref{app:cushion}\\
\Cref{cor:local} (local exponent) & \Cref{app:cushion}\\
\bottomrule
\end{tabular}
\end{table}

\begin{table}[h]
\centering
\small
\caption{Main notation. Definitions are in \cref{sec:setting} unless another location is
given.}
\label{tab:notation}
\begin{tabular}{p{0.22\linewidth}p{0.70\linewidth}}
\toprule
Symbol & Meaning\\
\midrule
\multicolumn{2}{l}{\emph{Trees, policies and laws}}\\
$\tau$, $h$, $\J$ & Episode, history (node), set of leaves (complete episodes)\\
$ha$, $ho$, $h\preceq j$ & Child of $h$ by action $a$ or outcome $o$; node $h$ lies on the path to leaf $j$\\
$\mu$, $\pi$ & Logger (known, $\mu(a\mid h)>0$) and target policy\\
$R\in[0,b]$ & Total reward of an episode\\
$W$, $X$ & $W(\tau)=\prod_{t}\pi(a_t\mid h_t)/\mu(a_t\mid h_t)$ and $X=(W-1)R$\\
$\Es$, $P_e$ & All environments on the tree; logged law induced by $e$ and $\mu$\\
$P$, $S$ & True logged law $P_{e_0}$; its support $S=\supp P$\\
$\Delta$ & $\Delta(P)=V(\pi)-V(\mu)=\E_PX$\\
$B$ & Fair-bet matrix, $B_{j,(h,a)}=\1\{ha\preceq j\}-\mu(a\mid h)\1\{h\preceq j\}$ (\cref{lem:polytope})\\
$\Qs$, $q$ & Action-neutral laws $\{q\in\Delta(\J):B^\top q=0\}$; a generic law\\
$\Qn$ & Null laws $\{q\in\Qs:\E_qX\le0\}$\\
\midrule
\multicolumn{2}{l}{\emph{Exponents and optimal bets}}\\
$\Kstar$ & Certification exponent $\inf\{\KL(P\,\|\,q):q\in\Qn\}$ (\cref{def:exponent})\\
$\Kscal$ & Scalar exponent $\sup\{\E_P\log(1+\lambda X):\lambda\ge0,\ 1+\lambda X\ge0\text{ on }\J\}$\\
$\Ys$, $Y$ & Feasible bets $Y=1+\lambda X+B\nu\ge0$ on $\J$ (\cref{thm:duality})\\
$Y^\star$, $q^\star$ & Optimal bet $1+\lambda^\star X+B\nu^\star$; optimal null law, $q^\star=P/Y^\star$ on $S$\\
$m$ & Superhedging price on nodes (\cref{prop:superhedge}); $m=m(\mathrm{root})$ in \cref{sec:coverage}\\
$\sigeff$, $\sigis$ & $\min_\nu\Var_P(X+B\nu)$ and $\Var_P(X)$ (\cref{sec:coverage})\\
\midrule
\multicolumn{2}{l}{\emph{Certifiers}}\\
$\tau_\delta$ & Stopping time of a sequential certifier at level $\delta$\\
$M_t$, $Y_t$ & Wealth after $t$ episodes; the factor $1+s_t[\hat\lambda X(\tau_t)+(B\hat\nu)(\tau_t)]$ at episode $t$ (\cref{alg:plugin})\\
$P_{\min}$, $\Lambda$, $c_S$ & $\min_SP$ and the constants of \cref{thm:plugin-explicit}\\
$U$, $D$, $d$ & Basis of $\range(B)$; $D=[X,U]$; $d=1+\rank(B)$ (\cref{sec:achieve})\\
$Z$, $D_t$ & Feasible bets $\{z:\lambda\ge0,\ 1+Dz\ge0\text{ on }\J\}$ in these coordinates; $D_t=D(\tau_t)$\\
$G$ & Dominance certificate $G=X+B\nu\ge0$ in the arbitrage regime (\cref{prop:trichotomy}(d))\\
$\mbar$ & Supersolution, an upper bound on $m$ (\cref{sec:local})\\
$C$, $g$, $d(h)$ & Cushion, fair fraction and drop of the cushion game (\cref{sec:local})\\
$\Kloc(\mbar)$ & Local exponent, the best growth rate of the cushion game (\cref{cor:local})\\
\bottomrule
\end{tabular}
\end{table}

\clearpage
\begin{center}
{\Large\bfseries Part I: Proofs}
\end{center}
\addvspace{\baselineskip}

\section{Proof of \texorpdfstring{Theorem~\ref{thm:lower}}{the lower bound} (lower bound)}
\label{app:lower}

\Cref{thm:lower} states that every stopping time $\tau$ with $\E_P[\tau]<\infty$ satisfies
$q^{\otimes\infty}(\tau<\infty)\ge\exp(-\E_P[\tau]\,\KL(P\|q))$ for every law $q$ on $\J$
(\cref{eq:tradeoff}), and derives from it bounds for sequential certifiers (i), fixed-sample
tests (ii) and the least-favourable law (iii).

\begin{proof}[Proof of \cref{thm:lower}]
\emph{Proof overview.} We first prove \eqref{eq:tradeoff} for the bounded stopping times
$\tau\wedge N$, using the change-of-measure argument of \citet[Lemma~1]{garivier2016optimal}
on the event $\{\tau\le N\}$. Letting $N\to\infty$ gives \eqref{eq:tradeoff}. Parts (i)--(iii)
then follow by choosing $q$.

\paragraph{Step 1: truncated stopping times.} Fix $q$ with $\KL(P\|q)<\infty$; otherwise
\eqref{eq:tradeoff} is trivial. Then $P\ll q$ on $\J$, and the per-episode log-likelihood
ratio is bounded on $\supp P$. For each $N$, $\tau\wedge N$ is a bounded stopping time. Let
$A_N=\{\tau\le N\}$. Wald's identity and the data-processing inequality, as in
\citet[Lemma~1]{garivier2016optimal}, give
\begin{align*}
\E_P[\tau\wedge N]\,\KL(P\|q)
&\ \ge\ \kl\bigl(P(A_N),\,q(A_N)\bigr)\\
&\ \ge\ P(A_N)\log\frac{P(A_N)}{q(A_N)}+\bigl(1-P(A_N)\bigr)\log\bigl(1-P(A_N)\bigr).
\end{align*}
The second inequality drops the nonnegative term $-(1-x)\log(1-y)$ from $\kl(x,y)$.

\paragraph{Step 2: letting $N\to\infty$.} As $N\to\infty$:
\begin{itemize}
\item $P(A_N)\to1$, because $\E_P[\tau]<\infty$;
\item $q(A_N)\uparrow q(\tau<\infty)$;
\item $\E_P[\tau\wedge N]\uparrow\E_P[\tau]$, by monotone convergence.
\end{itemize}
The right side of Step~1 therefore tends to $\log(1/q(\tau<\infty))$, which proves
\eqref{eq:tradeoff}.

\paragraph{Step 3: consequences.}
(i) If $\E_P[\tau_\delta]<\infty$, apply \eqref{eq:tradeoff} to every $q\in\Qn$ and use
$q(\tau_\delta<\infty)\le\delta$. This gives $\E_P[\tau_\delta]\KL(P\|q)\ge\log(1/\delta)$,
and taking the infimum over $q$ proves the claim.

(ii) Apply the data-processing inequality to the event $\{\phi_n=1\}$:
\[
n\KL(P\|q)\ \ge\ \kl\bigl(P^{\otimes n}(\phi_n=1),\,q^{\otimes n}(\phi_n=1)\bigr)\ \ge\ \kl(1-\beta,\delta),
\]
because $\kl(x,y)$ decreases in $y$ for $y\le x$. Then take the infimum over $q$.

(iii) Take $q=q^\star$ in \eqref{eq:tradeoff}; $\KL(P\|q^\star)=\Kstar(P)$ by
\cref{thm:duality}.
\end{proof}

\section{Proof of \texorpdfstring{Theorem~\ref{thm:duality}}{the duality theorem} (duality)}
\label{app:duality}

\Cref{thm:duality} states that $\Kstar(P)=\sup_{Y\in\Ys}\E_P\log Y$ for every $P\in\Qs$,
with attainment by some $Y^\star=1+\lambda^\star X+B\nu^\star$ and a least-favourable law
$q^\star$ satisfying $q^\star_j=P_j/Y^\star_j$ on $\supp P$ when $\Kstar(P)<\infty$.

\paragraph{Notation used throughout the appendix.} $S=\supp P$ and
$P_{\min}=\min_{j\in S}P_j$. For a law $p$ on $\J$, let
$f_p(Y)=\sum_{j:p_j>0}p_j\log Y_j$, with $\log0=-\infty$.

\paragraph{Proof overview.} Weak duality, $\sup_{Y\in\Ys}f_P(Y)\le\Kstar(P)$, was proved in
\cref{sec:exponent}. We prove the reverse inequality, the attainment claims and the form of
$q^\star$. If $\Kstar(P)<\infty$ (Case~A), the refined Slater condition gives strong
Lagrangian duality, and optimizing out the scale of the dual variable turns the Lagrangian dual
into $\sup_{Y\in\Ys}\E_P\log Y$. If $\Kstar(P)=\infty$ (Case~B), LP duality produces feasible
bets with unbounded growth rate.

\subsection*{Case A: \texorpdfstring{$\Kstar(P)<\infty$}{K*(P) finite}}

\paragraph{Step A1: the primal and strong duality.} Since $\Kstar(P)<\infty$, some
$\bar q\in\Qn$ is positive on $S$. Write the primal as
\[
\min\ g(q)=\sum_{j\in S}P_j\log\frac{P_j}{q_j}\quad\text{s.t.}\quad q\ge0,\ \1^\top q=1,\
B^\top q=0,\ X^\top q\le0.
\]
The objective $g$ is convex on the open domain $\{q:q_j>0\ \forall j\in S\}$, and every
constraint is affine. Since $\bar q$ is feasible and lies in the domain, the refined Slater
condition \citep[\S5.2.3]{boyd2004convex} gives zero duality gap and attainment of the
Lagrangian dual.

\paragraph{Step A2: the Lagrangian.} Use multipliers $\lambda\ge0$ for $X^\top q\le0$, $\nu$
for $B^\top q=0$, $c$ for $\1^\top q=1$ and $\varrho\ge0$ for $q\ge0$, and set
$w=c\1+\lambda X+B\nu-\varrho$. The Lagrangian is
\[
L(q)=\sum_{j\in S}P_j\log\frac{P_j}{q_j}+\sum_jq_jw_j-c.
\]
Minimizing over $q$ coordinate by coordinate:
\begin{itemize}
\item for $j\notin S$, $\inf_{q_j\ge0}q_jw_j$ is $0$ if $w_j\ge0$ and $-\infty$ otherwise;
\item for $j\in S$, $\inf_{q_j>0}[P_j\log(P_j/q_j)+q_jw_j]$ equals $P_j\log w_j+P_j$,
  attained at $q_j=P_j/w_j$, if $w_j>0$, and $-\infty$ otherwise.
\end{itemize}
The dual function is therefore $\E_P\log w+1-c$ on $\{w\ge0,\ w>0\text{ on }S\}$. Since
increasing $w$ increases the dual function, $\varrho=0$ is optimal.

\paragraph{Step A3: removing the scale.} Every dual-feasible point has $c>0$. Indeed,
\[
\E_{\bar q}w=c+\lambda\E_{\bar q}X\le c,
\]
while $\E_{\bar q}w>0$ because $w\ge0$ everywhere and $w>0$ on $S$, where $\bar q>0$.
Writing $w=cY$ with $Y=1+(\lambda/c)X+B(\nu/c)\in\Ys$, the dual becomes
\[
\sup_{c>0,\ Y\in\Ys}\bigl[\E_P\log Y+(\log c+1-c)\bigr]=\sup_{Y\in\Ys}\E_P\log Y,
\]
because $\log c+1-c\le0$ with equality at $c=1$. This proves $\Kstar=\sup f_P$, and dual
attainment gives an optimal $Y^\star=1+\lambda^\star X+B\nu^\star$.

\paragraph{Step A4: the least-favourable law.} The primal minimum is attained at some
$q^\star$, because $\Qn$ is compact and $\KL(P\|\cdot)$ is lower semicontinuous. By strong
duality, $q^\star$ minimizes the Lagrangian at the dual optimum. The coordinatewise
minimization of Step~A2 then gives
\[
q^\star_j=P_j/Y^\star_j\ \text{ on }S,\qquad q^\star_jY^\star_j=0\ \text{ off }S.
\]
Complementary slackness gives $\lambda^\star\E_{q^\star}X=0$. Hence
$\KL(P\|q^\star)=\E_P\log Y^\star$.

\subsection*{Case B: \texorpdfstring{$\Kstar(P)=\infty$}{K*(P) infinite}}

We show that $\sup f_P=\infty$. There are two sub-cases.

\paragraph{Step B1: $\Qn=\emptyset$.} Then $m:=\min_{q\in\Qs}\E_qX>0$ by compactness of
$\Qs$. LP duality for $\min\{X^\top q:q\ge0,\ \1^\top q=1,\ B^\top q=0\}$ gives $\nu$ with
$X+B\nu\ge m$ on every leaf. So $Y_k=1+kX+B(k\nu)\ge1+km$ is feasible for every $k\ge0$, and
\[
f_P(Y_k)\ge\log(1+km)\to\infty.
\]

\paragraph{Step B2: $\Qn\ne\emptyset$, but every $q\in\Qn$ vanishes somewhere on $S$.} $\Qn$ is
a polytope. A point in its relative interior has maximal support $U$: every $q\in\Qn$ has
$\supp q\subseteq U$. By assumption $S\setminus U\neq\emptyset$, so
$\max\{\1_{S\setminus U}^\top q:q\in\Qn\}=0$. The LP duality of \cref{lem:evariables},
applied to $Y=\1_{S\setminus U}$, gives $\lambda\ge0$ and $\nu$ with
$G=\lambda X+B\nu\ge\1_{S\setminus U}$ on every leaf. Then $Y_k=1+kG\ge1$ is feasible and
\[
f_P(Y_k)\ge P(S\setminus U)\log(1+k)\to\infty.
\]

\paragraph{Extension to structural constraints.} The same proof applies verbatim with $B$
replaced by $[B,F]$ (\cref{thm:structure}), since $\Qs_F$ is again a polytope. \hfill$\square$

\section{Proofs for \texorpdfstring{Section~\ref{sec:coverage}}{the coverage regimes} (coverage regimes)}
\label{app:coverage}

This section proves \cref{prop:trichotomy} (coverage trichotomy) and \cref{prop:positivity}
(positivity under partial coverage). Recall that $m=m(\mathrm{root})$ is the superhedging price
of \cref{prop:superhedge} and that $\mu>0$ on every action node in \cref{prop:trichotomy}.

\subsection{A dual certificate}

Both proofs use the following certificate. By finite LP duality (as in
\cref{prop:superhedge}),
\[
\max_\nu\min_{j\in\J}(X+B\nu)_j=\min_{q\in\Qs}\E_qX=-m,
\]
and the maximum is attained because $\Qs\neq\emptyset$ makes both programs feasible and
bounded. Hence if $m\le0$ there is $\nu$ with
\[
G:=X+B\nu\ge-m\ge0\quad\text{on }\J.
\]
Every $q\in\Qs$ has $\E_q[B\nu]=0$, so $\E_qG=\E_qX$. In particular $\E_PG=\Delta$ because
$P\in\Qs$.

\subsection{Proof of \texorpdfstring{Proposition~\ref{prop:trichotomy}}{the coverage trichotomy} (coverage trichotomy)}

\paragraph{Proof overview.} Parts (a) and (b) follow from the definition of $\Kstar$ and
compactness of $\Qn$. Part (c) is a cycle of implications; the step $m=0$ uses the dual
certificate $G$, and the step $m>0\Rightarrow$\cref{ass:fs} mixes a maximizer of $\E_q[-X]$
with a full-support law. Part (d) uses that $G$ vanishes almost surely under every null law.

\paragraph{Part (a).} If $\Delta\le0$ then $P\in\Qn$, so $\Kstar(P)\le\KL(P\|P)=0$.
Soundness applied to $q=P\in\Qn$ gives $P(\tau<\infty)\le\delta<1$, so $P(\tau=\infty)>0$
and $\E_P\tau=\infty$.

\paragraph{Part (b).} If $\Delta>0$ then $P\notin\Qn$. $\Qn$ is a closed subset of the
simplex, hence compact, and $q\mapsto\KL(P\|q)$ is lower semicontinuous, so the infimum is
attained when $\Qn\neq\emptyset$ (if $\Qn=\emptyset$, $\Kstar=\infty>0$). At a minimizer
$q^\star\neq P$, so $\KL(P\|q^\star)>0$.

\paragraph{Part (c).} Assume $\Delta>0$. We prove four implications.

\emph{\cref{ass:fs} $\Rightarrow m>0$.} If $m<0$ then $\Qn=\emptyset$
(\cref{prop:superhedge}). If $m=0$, take $G$ from the dual certificate and a full-support
$q\in\Qn$. Then
\[
0\le\E_qG=\E_qX\le0,
\]
and with $G\ge0$ this forces $G=0$ $q$-a.s., hence $G\equiv0$ on $\J$. This contradicts
$\E_PG=\Delta>0$.

\emph{$m>0\Rightarrow$ \cref{ass:fs}.} Let $q_m\in\argmax_{\Qs}\E_q[-X]$, so
$\E_{q_m}X=-m<0$. Let $u\in\Qs$ be the law of $\mu$ run against the environment that draws
every outcome uniformly; since $\mu>0$ on every action node, $u$ has full support. For
$\varepsilon=\min\{\tfrac12,\,m/(m+|\E_uX|)\}>0$, the mixture
$q_\varepsilon=(1-\varepsilon)q_m+\varepsilon u$ lies in the convex set $\Qs$, has full
support, and
\[
\E_{q_\varepsilon}X\le-(1-\varepsilon)m+\varepsilon|\E_uX|\le0,
\]
so $q_\varepsilon\in\Qn$.

\emph{\cref{ass:fs} $\Rightarrow\Kstar<\infty$}, since $\KL(P\|q)<\infty$ for a
full-support $q$.

\emph{$\Kstar<\infty\Rightarrow m>0$.} If $m<0$, $\Kstar=\infty$. If $m=0$, every $q\in\Qn$
has $\E_qG=\E_qX\le0$ and $G\ge0$, so $q(G>0)=0$. But $\E_PG=\Delta>0$ gives $P(G>0)>0$, so
$P\not\ll q$ and $\KL(P\|q)=\infty$ for every $q\in\Qn$.

\paragraph{Part (d).} With $G$ as above, every $q\in\Qn$ has $q(G>0)=0$ (shown in (c)), so
\[
q^{\otimes\infty}(\tau<\infty)\le\sum_tq(G(j_t)>0)=0\le\delta.
\]
Under $P$ the episodes are i.i.d., so $\tau$ is geometric with success probability
$p=P(G>0)>0$ and $\E_P\tau=1/p$. \hfill$\square$

\subsection{Proof of \texorpdfstring{Proposition~\ref{prop:positivity}}{positivity} (positivity under partial coverage)}

\paragraph{Proof overview.} We identify the set of observable laws of null environments with
the null polytope of the covered tree for the claim $X^0$. Since uncovered rewards are
invisible to the data, the adversary can set them to $0$. The equivalence then follows from
\cref{prop:trichotomy}(a),(b) on the covered tree, and a one-step example shows that
$\Delta^0\le0<V(\pi)-V(\mu)$ can occur.

\paragraph{Step 1: decomposition of an environment.} An environment $e$ is specified by its
kernels at environment nodes reached with positive probability under $\mu$, which determine
the observable law $q=P_e$ on the covered tree, together with the unobserved kernels and
rewards after uncovered actions, which do not affect $q$. For every $e$,
\[
V_e(\mu)=\E_q R,\qquad V_e(\pi)=\E_q[W^0R]+V^{\mathrm{unc}}_e(\pi),
\]
with $V^{\mathrm{unc}}_e(\pi)\in[0,\infty)$ because rewards are nonnegative.

\paragraph{Step 2: the null of observable laws.} We claim
\[
\bigcup_{e}\{P_e: V_e(\pi)\le V_e(\mu)\}=\{q\in\Qs^0:\E_qX^0\le0\},
\]
where $\Qs^0$ is the action-neutral polytope of the covered tree.
\begin{itemize}
\item ``$\subseteq$'': $\E_qX^0=V_e(\pi)-V^{\mathrm{unc}}_e(\pi)-V_e(\mu)\le0$.
\item ``$\supseteq$'': given $q$ with $\E_qX^0\le0$, set every unobserved reward to $0$, so
$V^{\mathrm{unc}}_e(\pi)=0$ and $V_e(\pi)-V_e(\mu)=\E_qX^0\le0$.
\end{itemize}

\paragraph{Step 3: reduction to the covered tree.} Since data are drawn from the covered tree
only, soundness and $\Kstar$ are defined by this null, which is the null of the covered tree
with claim $X^0$. Applying \cref{prop:trichotomy}(a),(b) on the covered tree (where $\mu>0$ by
definition) gives $\Kstar=0$ iff $\Delta^0\le0$.

\paragraph{Step 4: an instance with $\Delta^0\le0<V(\pi)-V(\mu)$.} Take one step and two
actions. The logger $\mu$ always plays action $1$, with reward $r_1<b$, and $\pi$ plays the
uncovered action $2$, whose reward is $b$. Then $\Delta^0=-r_1\le0$ while
$V(\pi)-V(\mu)=b-r_1>0$. \hfill$\square$

\section{Proof of \texorpdfstring{Theorem~\ref{thm:local}}{the local expansion} (local expansion)}
\label{app:local}

This section proves \cref{thm:local}, which states that
$\Kstar(P)=\frac{\Delta^2}{2\sigeff}(1+O(\Delta))$ and
$\Kscal(P)=\frac{\Delta^2}{2\sigis}(1+O(\Delta))$ as $\Delta\to0$, and the explicit
sandwich of \cref{prop:sandwich}.

\paragraph{Assumptions.} We make the hypotheses of \cref{thm:local} precise. Along the
sequence of instances, $\Delta\to0$, $\sigeff$ is bounded away from zero,
$\sup_\J|X|\le\kappa_X$, and some minimizer $\theta^\star$ of $\Var_P(X+B\nu)$ has
$\sup_S|X+B\theta^\star|\le\kappa_S$, for constants $\kappa_X,\kappa_S$.

\paragraph{Proof overview.} The efficient residual $Z=X+B\theta^\star$ defines both a feasible
bet $1+sZ$ and a feasible null law $P(1-\alpha(Z-\Delta))$. The bet needs $Z$ bounded on
\emph{every} leaf, not only on $S$, so we first replace $\theta^\star$ by a canonical
minimizer $\theta^\circ$ that is bounded on all of $\J$ (Step~1). Evaluating the dual at the
shrunken bet gives the lower bound (Step~2); evaluating the primal at the linear tilt gives
the upper bound (Step~3). The scalar exponent follows by removing $B$ (Step~4), and the same
two test points with exact series bounds give \cref{prop:sandwich}.

\subsection{Proof of \texorpdfstring{Theorem~\ref{thm:local}}{the local expansion}}

\paragraph{Step 1: a canonical efficient residual is bounded on every leaf.} For an action
node $h$ and $\nu\in\R^N$, write
$\tilde\nu(h,a)=\nu(h,a)-\sum_{a'}\mu(a'\mid h)\nu(h,a')$, so that
\[
(B\nu)(j)=\sum_{(h,a)\preceq j}\tilde\nu(h,a)
\]
by \eqref{eq:fairbets}. Call an action node \emph{$P$-reached} if $P(h)>0$. Since $\mu>0$,
every child $ha$ of a $P$-reached node is also $P$-reached. Let $\theta^\circ$ agree with
$\theta^\star$ at $P$-reached nodes and vanish elsewhere. Every leaf in $S$ passes only
through $P$-reached nodes, so $B\theta^\circ=B\theta^\star$ on $S$. Hence $\theta^\circ$ is
again a minimizer, and $Z:=X+B\theta^\circ$ satisfies $\sup_S|Z|\le\kappa_S$.

It remains to bound $Z$ off $S$. Because $P\in\Qs$, fair bets have conditional mean zero
under $P$ at every $P$-reached node $g$. Therefore
\[
\E_P[B\theta^\circ\mid g]=\sum_{(h,a)\preceq g}\tilde\theta^\circ(h,a).
\]
Let $j\notin S$, and let $g$ be the deepest node on the path to $j$ with $P(g)>0$. Every
action node on the path after $g$ has $P(h)=0$, where $\tilde\theta^\circ=0$. So
$(B\theta^\circ)(j)=\E_P[B\theta^\circ\mid g]=\E_P[Z\mid g]-\E_P[X\mid g]$, and
\[
|Z(j)|\le|X(j)|+\kappa_S+\kappa_X\le2\kappa_X+\kappa_S.
\]
Thus $Z$ is bounded on all of $\J$ by $\kappa:=2\kappa_X+\kappa_S$. This leafwise bound is
what makes the efficient bet below feasible. We use $\theta^\star:=\theta^\circ$ from now
on.

Two properties of $Z$ are used repeatedly. Since $P\in\Qs$, $\E_P[B\nu]=0$ for all $\nu$, so
$\E_PZ=\Delta$. The normal equations of the projection state that $Z-\Delta$ is orthogonal in
$L^2(P)$ to every column of $B$.

\paragraph{Step 2: lower bound (dual).} Let $s=\Delta/\E_PZ^2$ and
$Y=1+sZ=1+sX+B(s\theta^\star)$. For $\Delta$ small, $s\kappa\le1/2$, so $Y\ge1/2$ on $\J$
and $Y\in\Ys$. Using $\log(1+x)\ge x-x^2/2-|x|^3$ for $|x|\le1/2$,
\begin{align*}
\Kstar\ge\E_P\log(1+sZ)
&\ge s\Delta-\tfrac12s^2\E_PZ^2-s^3\kappa^3
&&\text{($\E_PZ=\Delta$, $|Z|\le\kappa$)}\\
&=\frac{\Delta^2}{2\E_PZ^2}-O(\Delta^3),
&&\E_PZ^2=\sigeff+\Delta^2.
\end{align*}

\paragraph{Step 3: upper bound (primal).} Let $\alpha=\Delta/\sigeff$,
$g=1-\alpha(Z-\Delta)$ and $q=Pg$. For small $\Delta$, $g\ge1/2$, so $q\ge0$. We check that
$q\in\Qn$:
\begin{align*}
\1^\top q&=\E_Pg=1,\\
B^\top q&=\E_P[gB]=\E_P[B]-\alpha\E_P[(Z-\Delta)B]=0
&&\text{(orthogonality of $Z-\Delta$)},\\
\E_qX&=\E_P[gX]=\Delta-\alpha\sigeff=0
&&\text{(see below)}.
\end{align*}
The last line uses $\E_P[(Z-\Delta)X]=\E_P[(Z-\Delta)(Z-B\theta^\star)]=\Var_P(Z)$.
Finally
\[
\KL(P\|q)=-\E_P\log g=\tfrac12\alpha^2\sigeff+O(\alpha^3)
=\frac{\Delta^2}{2\sigeff}+O(\Delta^3).
\]

Combining Steps~2 and~3, and using that $\sigeff$ is bounded away from zero,
$\Kstar=\frac{\Delta^2}{2\sigeff}(1+O(\Delta))$.

\paragraph{Step 4: scalar exponent.} Repeat the argument with $B$ removed. The dual point
$Y=1+sX$ with $s=\Delta/\E_PX^2$ lower-bounds $\Kscal$. For the upper bound, any law $q$ on
$\J$ with $\E_qX\le0$ satisfies $\E_q(1+\lambda X)\le1$ for every $\lambda\ge0$, so Jensen's
inequality gives $\Kscal\le\KL(P\|q)$. Take $q=P(1-\alpha(X-\Delta))$ with
$\alpha=\Delta/\sigis$. \hfill$\square$

\subsection{Proof of \texorpdfstring{Proposition~\ref{prop:sandwich}}{the explicit sandwich} (explicit local sandwich)}

Let $Z=X+B\theta^\circ$ be the canonical efficient residual of Step~1 and
$\sigma=\sigeff=\Var_PZ$. Then $\E_PZ=\Delta$, $\E_PZ^2=\sigma+\Delta^2$, and $Z-\Delta$ is
$L^2(P)$-orthogonal to every column of $B$. The proof uses the same test points as
\cref{thm:local}, replacing the $O(\cdot)$ remainders by exact series bounds.

\paragraph{Part (i): upper bound.} Let $\alpha=\Delta/\sigma$, $g=1-\alpha(Z-\Delta)$ and
$q=Pg$. On $S$, $\alpha(Z-\Delta)\le\rho_+<1$, so $g>0$ on $S$ and $q\ge0$ (off $S$, $q=0$).
As in Step~3, $\1^\top q=1$, $B^\top q=0$ and $\E_qX=\Delta-\alpha\sigma=0$, so $q\in\Qn$ and
\[
\Kstar\le\KL(P\|q)=\E_P[-\log(1-x)],\qquad x=\alpha(Z-\Delta)\le\rho_+.
\]
We bound $-\log(1-x)$ pointwise:
\begin{align*}
-\log(1-x)&=x+\sum_{k\ge2}x^k/k\le x+\frac{x^2}{2(1-x)}\le x+\frac{x^2}{2(1-\rho_+)}
&&(0\le x\le\rho_+),\\
-\log(1-x)&=-\log(1+|x|)\le x+x^2/2
&&(x<0).
\end{align*}
Since $\E_Px=0$ and $\E_Px^2=\alpha^2\sigma$,
\[
\Kstar\le\frac{\alpha^2\sigma}{2(1-\rho_+)}=\frac{\Delta^2}{2\sigma(1-\rho_+)}.
\]

\paragraph{Part (ii): lower bound.} Let $Y=1+sZ=1+sX+B(s\theta^\circ)$ with
$s=\Delta/(\sigma+\Delta^2)>0$. The hypothesis $s\max\{0,-\min_\J Z\}\le1$ is equivalent to
$Y\ge0$ on $\J$, so $Y\in\Ys$ and $\Kstar\ge\E_P\log Y$. On $S$, $x=sZ\ge-\rho_->-1$. We bound
$\log(1+x)$ pointwise:
\begin{align*}
\log(1+x)&\ge x-x^2/2&&(x\ge0),\\
\log(1+x)&=-\sum_k|x|^k/k\ge x-\frac{x^2}{2(1-|x|)}&&(-\rho_-\le x<0).
\end{align*}
In both cases $\log(1+x)\ge x-x^2/(2(1-\rho_-))$. Taking expectations with $\E_Px=s\Delta$ and
$\E_Px^2=s^2(\sigma+\Delta^2)$,
\[
\Kstar\ge s\Delta-\frac{s^2(\sigma+\Delta^2)}{2(1-\rho_-)}
=\frac{\Delta^2}{\sigma+\Delta^2}\Bigl(1-\frac1{2(1-\rho_-)}\Bigr)
=\frac{\Delta^2(1-2\rho_-)}{2(\sigma+\Delta^2)(1-\rho_-)}.
\]

\paragraph{Scalar version.} Replace $Z$ by $X$, $\sigma$ by $\sigis$ and drop $B$. In (i),
$q=P(1-\alpha(X-\Delta))$ is a law with $\E_qX=\Delta-\alpha\sigis=0$, and the Jensen
argument of Step~4 gives $\Kscal\le\KL(P\|q)$; the same series bound applies. In (ii), $1+sX$
is a feasible scalar bet. Since $\rho_\pm=O(\Delta)$ under the hypotheses of
\cref{thm:local}, the two bounds reproduce \cref{thm:local}. \hfill$\square$

\section{Proof of \texorpdfstring{Theorems~\ref{thm:plugin} and~\ref{thm:plugin-explicit}}{the plug-in theorems} (plug-in certifier)}
\label{app:plugin}

Part (i) of \cref{thm:plugin} (soundness) is proved in the main text. This section proves
parts (ii) and (iii) of \cref{thm:plugin} and both parts of \cref{thm:plugin-explicit}.

\paragraph{Notation.} Throughout this section $0\le\Kstar(P)<\infty$, $Y^\star$ is the
optimal wealth of \cref{thm:duality}, $L_j=\log Y^\star_j$ for $j\in S$, and
$r=\max_SL-\min_SL$. Let $\bar q$ be as in Case A of \cref{app:duality}. Recall that
$f_p(Y)=\E_p\log Y$, that $\hat Y_t$ denotes the bet fitted at the last refit before
episode $t$, and that $M_t=\prod_{s\le t}Y_s$ is the wealth of \cref{alg:plugin}.

\paragraph{Proof overview.} The plug-in bet is close to $Y^\star$ whenever the projected law
$\tilde p$ is close to $P$ and has support $S$. \Cref{lem:continuity} turns closeness of laws
into closeness of log-bets, and \cref{lem:projection} shows that $\tilde p$ has the right
support and is close to $P$ with exponentially high probability. Off these good events the
shrinkage $s_t=1-t^{-1/2}$ still guarantees a loss of at most $\frac12\log t$ per episode.
The strong law of large numbers then gives the almost-sure rate, and summing tail
probabilities gives the expected stopping time. For \cref{thm:plugin-explicit} the
qualitative continuity lemma is replaced by a quantitative one, and the strong law by the
Azuma--Hoeffding inequality.

\subsection{Two lemmas: continuity of the optimal bet and the logged-law projection}

\begin{lemma}[Continuity of the optimal bet]
\label{lem:continuity}
For every $\epsilon'>0$ there is $\eta>0$ such that, for every law $p$ with $\supp p=S$ and
$\|p-P\|_1\le\eta$, every maximizer $Y(p)$ of $f_p$ over $\Ys$ satisfies
$\max_{j\in S}|\log Y_j(p)-L_j|\le\epsilon'$.
\end{lemma}

\begin{proof}
\emph{Compactness and uniqueness.} Every $Y\in\Ys$ has $\E_{\bar q}Y\le1$, so
$0\le Y_j\le1/\bar q_j$ for $j\in S$. The set $\Ys_S=\{Y|_S:Y\in\Ys\}$ is a linear image of a
polyhedron, hence closed, and it is bounded, hence compact. If $\supp p=S$, then $f_p$ is
strictly concave on $\Ys_S\cap(0,\infty)^S$, so the restriction $y(p)$ of any maximizer is
unique. In particular $y(P)=Y^\star|_S$, which is positive because
$f_P(Y^\star)=\Kstar>-\infty$.

\emph{Convergence.} Let $p_n\to P$ with $\supp p_n=S$, and let $\bar y$ be a limit point of
$y(p_n)$. Since $f_{p_n}(y(p_n))\ge f_{p_n}(Y^\star)\to\Kstar$, and each $P_j$ with $j\in S$
is positive, upper semicontinuity of $\log$ gives
\[
\Kstar\le\limsup_nf_{p_n}(y(p_n))\le f_P(\bar y)\le\Kstar.
\]
By uniqueness, $\bar y=Y^\star|_S$. Hence $y(p_n)\to Y^\star|_S$, a positive vector, and the
logarithms converge uniformly on the finite set $S$. The stated $\eta$ follows by
contradiction over sequences.
\end{proof}

\begin{lemma}[The logged-law projection]
\label{lem:projection}
Let $\hat p$ be the empirical law of $n$ episodes from $P$,
$\tilde p=\argmin_{q\in\Qs}\KL(\hat p\|q)$ and $\hat\eta=\|\hat p-P\|_\infty$.
\begin{enumerate}
\item[(a)] If $\hat\eta<P_{\min}$, then $\supp\tilde p=S$.
\item[(b)] $\|\tilde p-P\|_1\le|S|\hat\eta+\sqrt{2|S|/P_{\min}}\,\hat\eta$.
\item[(c)] $P(\hat\eta>\eta')\le2|S|\exp(-2n\eta'^2)$.
\end{enumerate}
\end{lemma}

\begin{proof}
(a) $\tilde p$ is the maximum-likelihood law with $\mu$ known, so it equals $\mu$ times the
empirical environment conditionals. Since $\mu>0$, a leaf outside $S$ has an environment
transition of $P$-probability zero. That transition is never observed, so its empirical
conditional is zero and $\supp\tilde p\subseteq S$. If $\hat\eta<P_{\min}$, every leaf of
$S$ has been observed. Then every environment transition on a path in $S$ has positive
empirical frequency, so $\tilde p>0$ on $S$.

(b) Since $P\in\Qs$ and $\hat p=P=0$ off $S$,
\begin{align*}
\KL(\hat p\|\tilde p)&\le\KL(\hat p\|P)
&&(\text{$\tilde p$ minimizes over $\Qs\ni P$})\\
&\le\chi^2(\hat p\|P)=\sum_{j\in S}(\hat p_j-P_j)^2/P_j\le|S|\hat\eta^2/P_{\min}.
\end{align*}
Pinsker's inequality gives $\|\hat p-\tilde p\|_1\le\sqrt{2|S|/P_{\min}}\,\hat\eta$. Add
$\|\hat p-P\|_1\le|S|\hat\eta$.

(c) Apply Hoeffding's inequality to each $j\in S$ and take a union bound.
\end{proof}

\subsection{Proof of \texorpdfstring{Theorem~\ref{thm:plugin}}{the plug-in theorem}, parts (ii) and (iii)}

\paragraph{Step 1: per-episode bounds.} Let $n(t)$ be the number of episodes behind the bet
used at episode $t$. From $r_k\le t-1<r_{k+1}\le\gamma r_k+1$ we get $n(t)\ge(t-2)/\gamma$.
Write the factor as $Y_t=(1-s_t)+s_t\hat Y_t(j_t)$ with $s_t=1-t^{-1/2}$, where $\hat Y_t$ is
the exact optimizer at the projected law $\tilde p_{n(t)}$ and $j_t$ is the leaf of episode
$t$. Since $\hat Y_t\ge0$ and $\log(1-x)\ge-2x$ for $x\le1/2$, for every $t$:
\begin{align}
\log Y_t&\ge\log(1-s_t)=-\tfrac12\log t, \label{eq:always}\\
\log Y_t&\ge\log s_t+\log\hat Y_t(j_t)\ge\log\hat Y_t(j_t)-2t^{-1/2}\qquad(t\ge4).
\label{eq:good}
\end{align}
The first bound holds always; the second is useful on the good events below.

\paragraph{Step 2: good events.} Given $\epsilon'>0$, take $\eta$ from \cref{lem:continuity}
and $\eta'<P_{\min}$ with $(|S|+\sqrt{2|S|/P_{\min}})\eta'\le\eta$. Let
$G_t=\{\|\hat p_{n(t)}-P\|_\infty\le\eta'\}$. On $G_t$, \cref{lem:projection} gives
$\supp\tilde p_{n(t)}=S$ and $\|\tilde p_{n(t)}-P\|_1\le\eta$. Since $j_t\in S$ almost surely,
\cref{lem:continuity}, \eqref{eq:good} and \cref{lem:projection}(c) give
\begin{align*}
\log Y_t&\ge L_{j_t}-\epsilon'-2t^{-1/2}\quad\text{on }G_t,\ t\ge4,\\
P(G_t^c)&\le2|S|\exp\bigl(-2\eta'^2(t-2)/\gamma\bigr).
\end{align*}

\paragraph{Step 3: almost-sure rate (part (ii)).} $\sum_tP(G_t^c)<\infty$, so by the
Borel--Cantelli lemma $G_t$ holds for all $t\ge t_0(\omega)$ almost surely. Summing
\eqref{eq:always} for $t<t_0$ and the good-event bound for $t\ge t_0$,
\[
\log M_t\ge\sum_{s\le t}L_{j_s}-t\epsilon'-4\sqrt t-C(\omega),
\]
where $C(\omega)$ collects finitely many early terms. By the strong law,
$\frac1t\sum_{s\le t}L_{j_s}\to\E_PL=\Kstar$, so $\liminf\frac1t\log M_t\ge\Kstar-\epsilon'$.
Since $\epsilon'$ is arbitrary, (ii) follows. This includes the case $\Kstar=0$.

\paragraph{Step 4: a high-probability wealth bound (part (iii)).} If $\Kstar=0$ the upper
bound is $+\infty$, so assume $\Kstar>0$. Fix $\epsilon\in(0,1/3)$ and set
$\epsilon'=\epsilon\Kstar$. For $T\ge16$ let $T_0=\lceil\sqrt T\rceil\ge4$ and
\[
E_T=\bigcap_{T_0\le t\le T}G_t,\qquad H_T=\Bigl\{\sum_{t\le T}L_{j_t}\ge T(\Kstar-\epsilon')\Bigr\}.
\]
On $E_T\cap H_T$, apply \eqref{eq:always} for $t<T_0$ and the good-event bound for
$t\ge T_0$. Using $\sum_{t\le T}t^{-1/2}\le2\sqrt T$, $T_0\le2\sqrt T$ and
$\sum_{t<T_0}L_{j_t}\le T_0\max(\max_SL,0)$,
\[
\log M_T\ge T\Kstar(1-2\epsilon)-\sqrt T(\log T+C_1),\qquad C_1=4+2\max(\max_SL,0).
\]

\paragraph{Step 5: from wealth to stopping time.} Let $T_1(\epsilon,P)\ge16$ be such that
$\sqrt T(\log T+C_1)\le\epsilon T\Kstar$ for all $T\ge T_1$. For
\[
T\ge T^\dagger:=\max\bigl(T_1,\lceil\log(1/\delta)/((1-3\epsilon)\Kstar)\rceil\bigr),
\]
the event $E_T\cap H_T$ implies $\log M_T\ge(1-3\epsilon)T\Kstar\ge\log(1/\delta)$, so
$\tau_\delta\le T$. Hence
\[
P(\tau_\delta>T)\le P(E_T^c)+P(H_T^c)\le
2|S|\sum_{t\ge\sqrt T}e^{-2\eta'^2(t-2)/\gamma}+e^{-2T\epsilon^2\Kstar{}^2/r^2}.
\]
The second term is Hoeffding's inequality for the i.i.d.\ variables $L_{j_t}$ of range $r$
(it is $0$ if $r=0$). The first term is $O(e^{-c\sqrt T})$, so both are summable. Therefore
\begin{align*}
\E_P[\tau_\delta]=\sum_{T\ge0}P(\tau_\delta>T)
&\le T^\dagger+\sum_{T\ge T^\dagger}[P(E_T^c)+P(H_T^c)]\\
&\le\frac{\log(1/\delta)}{(1-3\epsilon)\Kstar}+C(\epsilon,P),
\end{align*}
with $C(\epsilon,P)=T_1+1+\sum_{T\ge16}[P(E_T^c)+P(H_T^c)]$, which does not depend on
$\delta$. The lower bound is \cref{thm:lower}(i). \hfill$\square$

\subsection{Proof of \texorpdfstring{Theorem~\ref{thm:plugin-explicit}}{the explicit plug-in bounds} (explicit plug-in bounds)}

Recall from the main text: $P_{\min}=\min_SP$, $q^\star=P/Y^\star$ on $S$,
$\Lambda=\max\{\max_S\log(1/q^\star),\allowbreak\log(2/P_{\min})\}$, $c_S=|S|+\sqrt{2|S|/P_{\min}}$ and
$\eta'=P_{\min}/(2c_S)$; part (i) uses $\ell=\log(4|S|T/\beta)$ and
$T_0=\max\{4,R_1+1,\lceil2+\gamma\ell/(2\eta'^2)\rceil\}$.

\paragraph{Proof overview.} Step~1 is a quantitative replacement for
\cref{lem:continuity}: log-bets are bounded by $\Lambda$ on $S$, which converts an
$\ell_1$ error in the law into a loss of at most $2\Lambda\|\tilde p-P\|_1$ in growth rate.
Step~2 controls all projected laws at once by a union bound. Step~3 replaces the strong law
by the Azuma--Hoeffding inequality, and Step~4 assembles part (i). Part (ii) chooses
$\beta=2/T^2$ and sums the resulting tail bound, using two elementary facts about
$x/\log x$.

\paragraph{Step 1: quantitative continuity.} Let $\tilde p\in\Qs$ with
$\supp\tilde p\subseteq S$ and $\|\tilde p-P\|_1\le P_{\min}/2$. Then
$\tilde p_j\ge P_j-P_{\min}/2\ge P_{\min}/2$ on $S$, so $\supp\tilde p=S$; let $\hat y$ be a
maximizer of $f_{\tilde p}$ over $\Ys$. We bound $\hat y$ on $S$ from both sides.
\begin{itemize}
\item \emph{Upper bound.} $q^\star\in\Qn$ by \cref{thm:duality}, and every
$Y=1+\lambda X+B\nu\in\Ys$ has $\E_{q^\star}Y=1+\lambda\E_{q^\star}X\le1$. Since $Y\ge0$,
$q^\star_jY_j\le1$, i.e.\ $Y_j\le1/q^\star_j$ on $S$.
\item \emph{Lower bound.} $Y\equiv1\in\Ys$ and $\Ys$ is convex, so the one-sided derivative
of $f_{\tilde p}$ at $\hat y$ towards $1$ is nonpositive: $\E_{\tilde p}[(1-\hat y)/\hat y]\le0$,
i.e.\ $\E_{\tilde p}[1/\hat y]\le1$, hence $\hat y_j\ge\tilde p_j\ge P_{\min}/2$ on $S$. The
same argument at $P$ gives $Y^\star_j\ge P_j\ge P_{\min}$.
\end{itemize}
Hence $|\log\hat y_j|\le\Lambda$ and $|\log Y^\star_j|\le\Lambda$ for $j\in S$ (note
$q^\star_j\le1$, so $\log(1/q^\star_j)\ge0$). Using optimality of $\hat y$ for $\tilde p$
twice,
\begin{equation}
\E_P\log\hat y\ge\E_{\tilde p}\log\hat y-\Lambda\|\tilde p-P\|_1
\ge\E_{\tilde p}\log Y^\star-\Lambda\|\tilde p-P\|_1\ge\Kstar-2\Lambda\|\tilde p-P\|_1 .
\label{eq:quant-cont}
\end{equation}

\paragraph{Step 2: all projections are good simultaneously.} Fix $T\ge1$ and
$\beta\in(0,1)$. If $T<T_0$, then \eqref{eq:always} gives
\[
\log M_T\ge-\tfrac T2\log T_0\ge T\Kstar-(T_0-1)\bigl(\Kstar+\tfrac12\log T_0\bigr)
\]
deterministically, so assume $T\ge T_0$. Let
$A_n=\{\|\hat p_n-P\|_\infty\le\sqrt{\ell/(2n)}\}$. By Hoeffding's inequality and a union
bound over $j\in S$, $P(A_n^c)\le2|S|e^{-\ell}=\beta/(2T)$, so $A=\bigcap_{n\le T}A_n$ has
$P(A^c)\le\beta/2$.

Here $\eta'=P_{\min}/(2c_S)$ is the constant of \cref{thm:plugin-explicit}, not the $\eta'$
of Step~2 of the previous proof. For $t\ge T_0$ we have
$n(t)\ge(t-2)/\gamma\ge\ell/(2\eta'^2)$, so on $A_{n(t)}$ the deviation is at most
$\eta'<P_{\min}$. \Cref{lem:projection}(a,b) then give $\supp\tilde p_{n(t)}=S$ and
\[
\|\tilde p_{n(t)}-P\|_1\le c_S\sqrt{\ell/(2n(t))}\le c_S\sqrt{\gamma\ell/(2(t-2))}
\le c_S\eta'=P_{\min}/2,
\]
so Step~1 applies to $\tilde p_{n(t)}$.

\paragraph{Step 3: martingale concentration.} Let $\mathcal F_{t-1}$ be generated by
episodes $1,\dots,t-1$. The bet $\hat Y_t$ and the event $A_{n(t)}$ are
$\mathcal F_{t-1}$-measurable, and $j_t\sim P$ is independent of $\mathcal F_{t-1}$. Define,
for $T_0\le t\le T$,
\[
D_t=\1_{A_{n(t)}}\bigl(\log\hat Y_t(j_t)-\E_P[\log\hat Y_t(j)\mid\mathcal F_{t-1}]\bigr).
\]
On $A_{n(t)}$, and since $j_t\in S$ almost surely, $\log\hat Y_t(j_t)\in[-\Lambda,\Lambda]$.
So $(D_t)$ is a martingale difference sequence whose conditional range has length at most
$2\Lambda$. The Azuma--Hoeffding inequality gives
\[
P\Bigl(\sum_tD_t<-\Lambda\sqrt{2T\log(2/\beta)}\Bigr)
\le\exp\bigl(-2\cdot2T\Lambda^2\log(2/\beta)/(4T\Lambda^2)\bigr)=\beta/2.
\]

\paragraph{Step 4: assembling part (i).} On the complement of both bad events (probability
at least $1-\beta$), all $A_{n(t)}$ hold. By \eqref{eq:always} for $t<T_0$, \eqref{eq:good}
for $t\ge T_0\ge4$, and \eqref{eq:quant-cont} with the $\ell_1$ bound of Step~2,
\begin{align*}
\log M_T&\ge-\tfrac{T_0-1}2\log T_0+\sum_{t=T_0}^T\bigl(\E_P[\log\hat Y_t\mid\mathcal F_{t-1}]
-2t^{-1/2}\bigr)-\Lambda\sqrt{2T\log(2/\beta)}\\
&\ge-\tfrac{T_0-1}2\log T_0+(T-T_0+1)\Kstar-2\Lambda c_S\sqrt{\gamma\ell/2}
\sum_{t=T_0}^T(t-2)^{-1/2}\\
&\qquad-4\sqrt T-\Lambda\sqrt{2T\log(2/\beta)} .
\end{align*}
Since $T_0-2\ge2$, $\sum_{t=T_0}^T(t-2)^{-1/2}\le\sum_{k=1}^{T}k^{-1/2}\le2\sqrt T$, and the
right-hand side is at least $T\Kstar-R_\beta(T)$. This proves part (i). \hfill$\square$

\paragraph{Part (ii): constants.} With $\eta'$ as in \cref{thm:plugin-explicit}, set
\begin{align*}
c_l&=3+\log(2|S|), & \kappa&=R_1+3+\gamma c_l/(2\eta'^2), & \kappa'&=\max(e,\kappa),\\
a&=4+2\Lambda\bigl(1+c_S\sqrt{2\gamma c_l}\bigr), & b&=\kappa(\Kstar+\tfrac12), &&\\
c_1&=\max(e,16a^2/\Kstar{}^2), & c_2&=\max(e^2/4,4b/\Kstar), &&
\end{align*}
and
\[
T_2=\max\bigl(2c_1\log c_1,\,4c_2\log^2(4c_2)\bigr),\qquad T_3=\max(3,2\kappa'\log\kappa').
\]
Note $4\sqrt T+2\Lambda c_S\sqrt{2\gamma Tc_l\log T}+\Lambda\sqrt{4T\log T}\le a\sqrt{T\log T}$
for $T\ge3$.

\paragraph{Step 5: two elementary facts.}
\begin{itemize}
\item[(F1)] If $c\ge e$ and $x\ge2c\log c$, then $x/\log x\ge c$. Indeed, $x/\log x$ is
increasing on $[e,\infty)$, $2c\log c\ge e$, and at $x=2c\log c$ the claim reads
$\log c\ge\log(2\log c)$, i.e.\ $c\ge2\log c$.
\item[(F2)] If $c_2\ge e^2/4$ and $x\ge4c_2\log^2(4c_2)$, then $x/\log^2x\ge c_2$. Apply
(F1) to $y=\sqrt x$ and $c=2\sqrt{c_2}\ge e$, noting $4c^2\log^2c=4c_2\log^2(4c_2)$.
\end{itemize}

\paragraph{Step 6: a monotone envelope for $R_\beta$.} Take $\beta=\beta_T=2/T^2$ with
$T\ge3$. Then $\ell=\log(2|S|T^3)\le c_l\log T$ and $\log(2/\beta)=2\log T$, using
$\log T\ge1$. Each entry of the maximum defining $T_0$ is at most $\kappa\log T$ (e.g.\
$\lceil2+\gamma\ell/(2\eta'^2)\rceil\le(3+\gamma c_l/(2\eta'^2))\log T$), so
$T_0\le\kappa\log T$. If moreover $T\ge T_3$, (F1) with $c=\kappa'$ gives $\kappa\log T\le T$,
hence $\log T_0\le\log T$ and
\[
T_0\bigl(\Kstar+\tfrac12\log T_0\bigr)\le\kappa(\Kstar+\tfrac12)\log^2T=b\log^2T.
\]
Together with $4\sqrt T\le4\sqrt{T\log T}$,
\[
R_{\beta_T}(T)\le U(T):=a\sqrt{T\log T}+b\log^2T\qquad(T\ge\max(3,T_3)).
\]

\paragraph{Step 7: the wealth threshold is crossed after $T_1$.} $U$ is increasing. Let
$u_0=U(L^\star)/\Kstar$ and $T_1=T^\star+u_0+\max(T_2,T_3)$. We claim
$T\Kstar-U(T)\ge\log(1/\delta)$, i.e.\ $T-T^\star\ge U(T)/\Kstar$, for every $T\ge T_1$.
\begin{itemize}
\item If $T\le L^\star$, then $U(T)/\Kstar\le u_0\le T-T^\star$.
\item If $T>L^\star\ge2T^\star$, it suffices that $U(T)/\Kstar\le T/2$. This follows from
$a\sqrt{T\log T}/\Kstar\le T/4$, i.e.\ $T/\log T\ge16a^2/\Kstar{}^2$, which holds by (F1)
since $T\ge2c_1\log c_1$; and from $b\log^2T/\Kstar\le T/4$, which holds by (F2) since
$T\ge4c_2\log^2(4c_2)$.
\end{itemize}

\paragraph{Step 8: summing the tail.} For every integer $T\ge T_1$ (note $T_1\ge3$), part (i)
gives $P(\tau_\delta>T)\le P(\log M_T<\log(1/\delta))\le\beta_T=2/T^2$, and
\[
\E_P[\tau_\delta]=\sum_{T\ge0}P(\tau_\delta>T)\le\lceil T_1\rceil+\sum_{T\ge3}\frac2{T^2}
\le T_1+1+2\Bigl(\frac{\pi^2}6-\frac54\Bigr)\le T_1+2 .
\]
For the asymptotic form, $u_0=O(\sqrt{T^\star\log T^\star}+\log^2T^\star)$ with constants
not depending on $\delta$. \hfill$\square$

\paragraph{Implementation.} The proof analyses the exact optimizer. The implementation
uses a log-barrier method with damped Newton steps, run to a duality gap of at most
$10^{-10}$. Its line search keeps every iterate strictly feasible (wealth positive on every
leaf, $\lambda>0$), so part (i) of \cref{thm:plugin} holds exactly for the returned bets.
The refit list is capped at $\lceil1.15^{199}\rceil\approx1.2\times10^{12}$ episodes.

\section{Proof of \texorpdfstring{Theorem~\ref{thm:mixture} and Proposition~\ref{prop:finite-mixture}}{the universal mixture results} (universal mixture)}
\label{app:mixture}

Recall the coordinates of \cref{thm:mixture}: $U$ is a basis of $\range(B)$, $D=[X,U]$,
$d=1+\rank(B)$, $Z=\{z=(\lambda,\nu'):\lambda\ge0,\ 1+Dz\ge0\text{ on }\J\}$ and
$D_t=D(\tau_t)$. The mixture wealth is $M_T=\mathrm{vol}(Z)^{-1}\int_Z\prod_{t\le T}(1+D_tz)\,dz$.

\paragraph{Proof overview.} We first show that $Z$ is a bounded convex body, so that the
uniform mixture over $Z$ is well defined. Soundness (i) is an average of supermartingales.
The regret bound (ii) compares the mixture with the best fixed bet by integrating over a
small homothetic copy of $Z$ around it. Parts (iii)--(iv) apply Hoeffding's inequality to
the log-wealth of the optimal bet $Y^\star$, and part (v) replaces Hoeffding by Wald's
identity. \Cref{prop:finite-mixture} repeats the soundness and Wald arguments for a finite
average and adds a volume argument for random sampling.

\subsection{Geometry of \texorpdfstring{$Z$}{Z}}

Suppose $\bar q\in\Qn$ has full support (\cref{ass:fs}), and let $v=(v_0,v')$ satisfy
$Dv\ge0$ on $\J$ with $v_0\ge0$. Then
\[
0\le\E_{\bar q}[Dv]=v_0\E_{\bar q}X\le0,
\]
so $Dv=0$ on every leaf. If $\Delta(P)\neq0$, then $X\notin\range(B)$, because
$\E_PX\neq0$ while $\E_P[B\nu]=0$. The columns of $D$ are therefore linearly independent,
and $v=0$. So $Z$ has no nonzero recession direction, that is, it is bounded. It contains a
neighbourhood of $0$ in $\{\lambda\ge0\}$ because $D$ is bounded. Hence
$0<\mathrm{vol}(Z)<\infty$.

\subsection{Proof of \texorpdfstring{Theorem~\ref{thm:mixture}}{the universal mixture theorem}}

\paragraph{Part (i): soundness.} Each $z\in Z$ defines a feasible bet, so
$\prod_{t\le T}(1+D_tz)$ is a nonnegative supermartingale with initial value $1$ under every
$q\in\Qn$. Averaging over $z$ preserves this, by Fubini's theorem. Ville's inequality gives
soundness.

\paragraph{Part (ii): regret against the best fixed bet.} Let $z^\star$ maximize
$\sum_{t\le T}\log(1+D_tz)$ over the compact set $Z$, and fix $\alpha\in(0,1)$. The set
$Z_\alpha=(1-\alpha)z^\star+\alpha Z$ lies in $Z$ by convexity and has volume
$\alpha^d\mathrm{vol}(Z)$. For $z=(1-\alpha)z^\star+\alpha z'$,
\[
1+D_tz=(1-\alpha)(1+D_tz^\star)+\alpha(1+D_tz')\ge(1-\alpha)(1+D_tz^\star).
\]
Integrating over $Z_\alpha$ only,
\[
M_T\ge\alpha^d(1-\alpha)^T\prod_{t\le T}(1+D_tz^\star).
\]
Take $\alpha=1/(T+1)$ and use $T\log(1+1/T)\le1$.

\paragraph{Part (iii), Step 1: reduction to a sum of i.i.d.\ terms.} The optimal bet
$Y^\star$ of \cref{thm:duality} lies in $Z$. Let $S_T=\sum_{t\le T}\log Y^\star(j_t)$. Its
summands are i.i.d.\ under $P$, with mean $\Kstar$ and range $r$. If $\tau_\delta>T$, then
$\log M_T<\log(1/\delta)$, and by (ii) $S_T<c_T$. Put $g(T)=T\Kstar-c_T$. If $g(T)>0$,
Hoeffding's inequality gives
\[
P(\tau_\delta>T)\le P(S_T<c_T)\le\exp\bigl(-2g(T)^2/(Tr^2)\bigr).
\]

\paragraph{Part (iii), Step 2: the implicit sample-size condition.} The bound above is at
most $\beta$ if and only if $g(T)\ge b_\beta\sqrt T$. As a quadratic in $x=\sqrt T$, with
$c=c_T$ held fixed, the inequality holds when
\[
x\ge\frac{b_\beta+\sqrt{b_\beta^2+4\Kstar c}}{2\Kstar}.
\]
Squaring, and using $\sqrt{b_\beta^2+4\Kstar c}\le b_\beta+2\sqrt{\Kstar c}$,
\[
\frac{\bigl(b_\beta+\sqrt{b_\beta^2+4\Kstar c}\bigr)^2}{4\Kstar{}^2}
\le\frac{\bigl(b_\beta+\sqrt{\Kstar c}\bigr)^2}{\Kstar{}^2}
=\frac c\Kstar+\frac{b_\beta^2}{\Kstar{}^2}+\frac{2b_\beta\sqrt c}{\Kstar{}^{3/2}},
\]
so the condition $T\ge c/\Kstar+b_\beta^2/\Kstar{}^2+2b_\beta\sqrt c/\Kstar{}^{3/2}$
suffices. The condition is checked at each $T$ with $c=c_T$. Since the argument holds
pointwise in $T$, this is valid even though $c_T$ grows with $T$.

\paragraph{Part (iii), Step 3: the explicit condition.} Fix $\epsilon\in(0,1/2)$. Two
elementary bounds control the $T$-dependent terms. By concavity of $\log$,
$\log y\le\log x+y/x-1$ for all $x,y>0$; we apply it with $y=T+1$ and
$x=d/(\epsilon\Kstar)$. Together with the AM--GM inequality this gives
\begin{align*}
b_\beta\sqrt T&\le\epsilon T\Kstar+\frac{b_\beta^2}{4\epsilon\Kstar}
&&\text{(AM--GM)},\\
d\log(T+1)&\le d\log\frac{d}{e\epsilon\Kstar}+\epsilon\Kstar(T+1)
&&\text{(concavity of $\log$)}.
\end{align*}
Hence
\[
g(T)-b_\beta\sqrt T\ \ge\ (1-2\epsilon)T\Kstar-\Bigl[\log(1/\delta)+1+d\log\frac{d}{e\epsilon\Kstar}
+\epsilon\Kstar+\frac{b_\beta^2}{4\epsilon\Kstar}\Bigr],
\]
which is nonnegative exactly when $T\ge T_\beta(\epsilon)$. For $\beta<1$ this also gives
$g(T)>0$, as Hoeffding's inequality requires.

\paragraph{Part (iv): expected stopping time.} We have $g'(T)=\Kstar-d/(T+1)\ge\Kstar/2$
whenever $T+1\ge2d/\Kstar$. So for $T\ge\bar T$, $g(T)\ge(T-\bar T)\Kstar/2$. Writing
$u=T-\bar T$,
\[
P(\tau_\delta>\bar T+u)\le\exp\Bigl(-\frac{u^2\Kstar{}^2}{2(\bar T+u)r^2}\Bigr)\le
\begin{cases}
\exp\bigl(-u^2\Kstar{}^2/(4\bar Tr^2)\bigr),&u\le\bar T,\\
\exp\bigl(-u\Kstar{}^2/(4r^2)\bigr),&u>\bar T.
\end{cases}
\]
Hence $\E_P[\tau_\delta]=\sum_{T\ge0}P(\tau_\delta>T)\le\bar T+1+\Sigma_1+\Sigma_2$, where
the two sums collect the two cases:
\begin{align*}
\Sigma_1&\le\int_0^\infty e^{-u^2/(2\sigma^2)}du=\sigma\sqrt{\pi/2}=\frac{r\sqrt{\pi\bar T}}{\Kstar}
&&\text{(Gaussian part, $\sigma^2=2\bar Tr^2/\Kstar{}^2$)},\\
\Sigma_2&\le\frac1{1-e^{-a}}\le1+\frac1a
&&\text{(geometric part, $a=\Kstar{}^2/(4r^2)$)}.
\end{align*}
Together, $\E_P[\tau_\delta]\le\bar T+2+r\sqrt{\pi\bar T}/\Kstar+4r^2/\Kstar{}^2$, which is
within the stated bound. Finally, $\bar T$ solves $T\Kstar=\log(1/\delta)+1+d\log(T+1)$ up
to rounding, which gives $\bar T=\log(1/\delta)/\Kstar+O(d\log\log(1/\delta)/\Kstar)$ for
fixed $P$.

\paragraph{Part (v), Step 1: a dominating crossing time.} Write $L=\log(1/\delta)$ and
$\xi_t=\log Y^\star(j_t)$. Under $P$ the $\xi_t$ are i.i.d., bounded on $\supp P$, and have
mean $\Kstar>0$. Put
\[
 \sigma_L=\inf\{T\ge1:\ \textstyle\sum_{t=1}^T\xi_t\ge L+1+d\log(T+1)\}.
\]
The pathwise bound in (ii) gives $\tau_\delta\le\sigma_L$. Because
$(L+1+d\log(T+1))/T\to0$, Hoeffding's inequality (or the deterministic drift when the range
of $\xi_t$ is zero) bounds $P(\sigma_L>T)$ by an exponentially summable tail for all
sufficiently large $T$. In particular $\E_P\sigma_L<\infty$.

\paragraph{Part (v), Step 2: overshoot and Wald's identity.} Let
$u=\max_{j\in\supp P}\log Y^\star(j)$, which is finite. At the first crossing the preceding
partial sum is below the preceding boundary, so, also when $\sigma_L=1$,
\[
 \sum_{t=1}^{\sigma_L}\xi_t
 \le L+1+d\log(\sigma_L+1)+u.
\]
Wald's identity applies to the integrable stopping time with bounded increments. With
$x=\E_P\sigma_L$, Jensen's inequality yields
\[
 \Kstar x\le L+1+u+d\log(x+1).
\]

\paragraph{Part (v), Step 3: solving for $x$.} The constant
$B=\sup_{y\ge0}\{d\log(y+1)-\Kstar y/2\}$ is finite and nonnegative. Hence
$x\le2(L+1+u+B)/\Kstar$, and substituting this back into the logarithm gives
\[
 x\le\frac{L}{\Kstar}+\frac{d}{\Kstar}\log L+O_P(1).
\]
The lower bound for $\E_P\tau_\delta$ is \cref{thm:lower}(i), since
$\tau_\delta\le\sigma_L<\infty$ almost surely. \hfill$\square$

\subsection{Proof of \texorpdfstring{Proposition~\ref{prop:finite-mixture}}{the finite-point mixture proposition}}

\paragraph{Proof overview.} Soundness holds for any fixed average of feasible bets. The
conditional stopping-time bound follows from the best of the $M$ bets by Wald's identity.
The sampling claim lower-bounds the probability that one uniform draw lands in a small
homothetic copy of $Z$ on which the growth rate is within $\epsilon$ of $\Kstar$.

\paragraph{Step 1: soundness.} For each feasible fixed $z_i$, its wealth product is a
nonnegative test supermartingale under every null law. A fixed average retains this
property, conditional on the bets; Ville's inequality gives $\delta$-soundness.

\paragraph{Step 2: conditional stopping time.} Let $i^\star$ maximize $K_M>0$ and write
$S_T=\sum_{t\le T}\log(1+D_tz_{i^\star})$. The sampled mixture satisfies
$M_T^{(M)}\ge e^{S_T}/M$ pathwise, so its stopping time is at most the first crossing
$\sigma$ of $L+\log M$ by $S_T$, where $L=\log(1/\delta)$. The increments are bounded and
i.i.d.\ under $P$, with positive mean $K_M$. A Hoeffding tail (or deterministic drift when
their range is zero) gives $\E_P\sigma<\infty$. At the first crossing,
\[
S_\sigma<L+\log M+u_M
\]
(with a weak inequality if equality occurs), because the preceding sum is below the
threshold. Wald's identity then yields $K_M\E_P\sigma\le L+\log M+u_M$, proving the
conditional stopping-time bound.

\paragraph{Step 3: a good region of $Z$.} Let $z^\star$ attain $\Kstar$ and put
$g(z)=\E_P\log(1+D_1z)$, taking $g(z)=-\infty$ when a factor vanishes on $\supp P$. Every
$w=z/2$ with $z\in Z$ has factors
\[
1+D(j)w=\bigl(1+(1+D(j)z)\bigr)/2\ge1/2,
\]
and hence $g(w)\ge-\log2$. Set $t=\epsilon/(\Kstar+\log2)\in(0,1)$. Convexity of $Z$ and
concavity of the logarithm show that every
\[
z'=(1-t)z^\star+\frac{t}{2}z,\qquad z\in Z,
\]
lies in $Z$ and satisfies
\[
g(z')\ge(1-t)\Kstar-t\log2=\Kstar-\epsilon.
\]

\paragraph{Step 4: sampling.} The homothetic image of $Z$ in Step~3 occupies the fraction
$(t/2)^d$ of its volume. Therefore one independent exact uniform draw has probability at
least $(t/2)^d$ of meeting the desired log rate, and $M$ draws miss it with probability at
most $\exp[-M(t/2)^d]$. Such draws have positive factors on $\supp P$ almost surely: each
zero-factor set is a proper affine face of the full-dimensional $Z$, and there are finitely
many outcomes. This proves the stated probability and conditional implication.
\hfill$\square$

\section{Proofs for \texorpdfstring{Section~\ref{sec:structure}}{structural assumptions} (structural assumptions and pool selection)}
\label{app:structure-main}

\Cref{sec:structure} summarizes how structural assumptions change the exponent. This
section states the underlying results in full (\cref{lem:irrelevance,thm:structure}),
proves them, and adds the pool-selection result \cref{prop:pool}. The experiment that
prices an assumption is in Part~II (\cref{app:structure-exp}).

\subsection{Statements}

\Cref{thm:ceiling} shows that exponential sample sizes are unavoidable without further
assumptions. Assumptions about the environment restrict $\Es$, add linear constraints to
$\Qs$, and can only raise $\Kstar$. For $q\in\Qs$, let
\[
e_q(j)=q_j\Big/\prod_{(h,a)\preceq j}\mu(a\mid h)
\]
be the product of the environment probabilities along the path to $j$.

\begin{lemma}[Action irrelevance is linear]
\label{lem:irrelevance}
Let $S$ be a set of action depths. The assumption that environment kernels after the
actions at depths in $S$ do not depend on those actions is equivalent to
$e_q(j)=e_q(j')$ for every pair of leaves with the same outcomes that differ only in
actions at depths in $S$. Since $\mu$ is known, these are linear equalities
$F^\top q=0$.
\end{lemma}

\begin{theorem}[Value and risk of an assumption]
\label{thm:structure}
Let $\Qs_F=\{q\in\Qs:F^\top q=0\}$, assume $P\in\Qs_F$, and let
$\Kstar_F=\inf\{\KL(P\|q):q\in\Qs_F,\ \E_qX\le0\}$.
\begin{enumerate}
\item[(i)] \emph{Value.}
\[
\Kstar_F=\sup\Bigl\{\E_P\log(1+\lambda X+B\nu+F\eta):\lambda\ge0,\
1+\lambda X+B\nu+F\eta\ge0\text{ on }\J\Bigr\}\ \ge\ \Kstar,
\]
with equality whenever the least-favourable law $q^\star$ of \cref{thm:duality}
satisfies $F^\top q^\star=0$. Certifiers built from these bets are $\delta$-sound for
every null in $\Qs_F$, and by \cref{thm:lower}(i) applied to the restricted null, no
certifier that is $\delta$-sound on $\Qs_F$ can beat $\log(1/\delta)/\Kstar_F$ in expectation.
\item[(ii)] \emph{Unavoidable risk.} Let $0<\Kstar(P)<\infty$. Every certifier, whether
  or not it uses $F$, with $\E_P[\tau_\delta]<\infty$ certifies falsely under $q^\star$
  with probability
  \[
  q^{\star\otimes\infty}(\tau_\delta<\infty)\ \ge\ \exp\bigl(-\Kstar(P)\,\E_P[\tau_\delta]\bigr).
  \]
  In particular, a certifier that is $\delta$-sound on $\Qs_F$ and stops in
  $\E_P[\tau_\delta]\le(1+o(1))\log(1/\delta)/\Kstar_F$ certifies falsely under $q^\star$
  with probability at least $\delta^{(1+o(1))\Kstar/\Kstar_F}$. When $\Kstar_F>\Kstar$
  this exceeds $\delta$ by a factor that grows without bound as $\delta\to0$, and
  $q^\star\in\Qn\setminus\Qs_F$.
\item[(iii)] \emph{Every excluded null is exploitable.} For every
  $q_0\in\Qn\setminus\Qs_F$ there is a feasible restricted bet $Y_0=1+F\eta_0$ with
  $\E_{q_0}\log Y_0>0$. The certifier that bets $Y_0$ at every episode certifies under
  $q_0$ with probability one, at every $\delta$.
\end{enumerate}
\end{theorem}

The ratio $\Kstar_F/\Kstar$ measures, in exponent units, what an assumption buys. Part (ii)
shows what it costs, and the cost does not depend on how the certifier is built. Any
procedure, restricted plug-in or otherwise, that turns the assumption into a speed-up
over $\log(1/\delta)/\Kstar$ must certify falsely at the least-favourable environment
$q^\star$ with probability at least $\exp(-\Kstar\E_P[\tau_\delta])$. For a restricted
certifier that attains $\Kstar_F$, this is $\delta^{\Kstar/\Kstar_F}$ to first order. The
exponent ratio therefore prices the gain and the risk together. Part (iii) adds that every
excluded null, not only $q^\star$, admits a restricted bet that certifies falsely almost
surely. \Cref{app:structure-exp} measures both sides for the tested instance.

\subsection{Proofs of \texorpdfstring{Lemma~\ref{lem:irrelevance} and Theorem~\ref{thm:structure}}{the structural results}}
\label{app:structure}

\begin{proof}[Proof of \cref{lem:irrelevance}]
$e_q(j)$ is the product of the environment kernels along the path to $j$. If those kernels
do not depend on the actions at depths in $S$, the products agree for paired leaves.
Conversely, if the products agree for all such pairs, summing out later outcomes, which
have the same branching structure on both sides, shows that each kernel agrees. Since $\mu$
is known, $e_q(j)=e_q(j')$ is the linear equality
\[
\frac{q_j}{\mu(j)}-\frac{q_{j'}}{\mu(j')}=0,
\]
where $\mu(j)$ is the product of logged propensities along $j$.
\end{proof}

\paragraph{Proof of \cref{thm:structure}: overview.} For (i), the duality of
\cref{thm:duality} holds verbatim with $B$ replaced by $[B,F]$, since $\Qs_F$ is again a
polytope. Shrinking the constraint set raises the infimum, and the proof of
\cref{thm:lower} uses nothing about the null beyond its definition. Part (ii) is
\cref{thm:lower}(iii) applied to $\tau_\delta$, with
$\E_P[\tau_\delta]\le(1+o(1))\log(1/\delta)/\Kstar_F$ substituted; if $\Kstar_F>\Kstar$,
then $q^\star\notin\Qs_F$, since otherwise $q^\star$ would be feasible for the restricted
problem and give $\Kstar_F\le\Kstar$. Part (iii) perturbs the constant bet $1$ in the
direction $F\,F^\top q_0$ and uses the strong law of large numbers. We prove (i), then
(iii), then (ii).

\paragraph{Step 1: the dual formula and $\Kstar_F\ge\Kstar$ (part (i)).}
\Cref{lem:evariables}, weak duality and \cref{app:duality} use $B$ only through two facts:
$\Qs=\{q\ge0,\1^\top q=1,B^\top q=0\}$ is a polytope, and $\E_q[B\nu]=0$ on it. Both remain
true with $[B,F]$ and $\Qs_F$, which gives the dual formula. Since $\Qs_F\subseteq\Qs$, the
infimum over the smaller set is larger, so $\Kstar_F\ge\Kstar$. If $q^\star\in\Qs_F$, it is
feasible for the restricted problem and attains $\Kstar$ there, so $\Kstar_F=\Kstar$.

\paragraph{Step 2: soundness and the restricted lower bound (part (i)).} Every restricted
bet $Y$ has
\[
\E_qY=1+\lambda\E_qX\le1\qquad\text{for }q\in\Qs_F\cap\Qn,
\]
so the certifiers of \cref{sec:achieve} built from them are sound for those nulls. The
proof of \cref{thm:lower} uses nothing about the null beyond its definition, so
\cref{thm:lower}(i) applied to the restricted null $\Qs_F\cap\Qn$ gives the lower bound
$\log(1/\delta)/\Kstar_F$ for certifiers that are $\delta$-sound on $\Qs_F$.

\paragraph{Step 3: a profitable bet against an excluded null (part (iii)).} Let
$q_0\in\Qn\setminus\Qs_F$, and set
\[
v=F^\top q_0\neq0,\qquad \eta_0=\varepsilon v,\qquad
0<\varepsilon\le\frac{1}{2\max_j|(Fv)_j|}.
\]
Then $Y_0=1+F\eta_0\ge1/2$ on every leaf, so it is a feasible restricted bet (with
$\lambda=0$, $\nu=0$). Its mean increment is positive:
\[
\E_{q_0}[F\eta_0]=\varepsilon\,v^\top F^\top q_0=\varepsilon\|v\|^2>0,
\]
which also implies $\E_{q_0}[(Fv)^2]>0$. Using $\log(1+x)\ge x-x^2$ for $x\ge-1/2$,
\[
\E_{q_0}\log Y_0\ge\varepsilon\|v\|^2-\varepsilon^2\E_{q_0}[(Fv)^2]>0
\quad\text{for}\quad\varepsilon<\|v\|^2/\E_{q_0}[(Fv)^2].
\]
By the strong law of large numbers,
\[
\frac1t\log\prod_{s\le t}Y_0(\tau_s)\to\E_{q_0}\log Y_0>0
\qquad\text{almost surely under }q_0.
\]
The wealth therefore crosses $1/\delta$ in finite time almost surely, which proves part
(iii).

\paragraph{Step 4: the unavoidable risk (part (ii)).} Let $\tau$ be any stopping time with
$\E_P[\tau]<\infty$. By \cref{thm:lower}(iii),
\[
q^\star(\tau<\infty)\ge\exp(-\Kstar\E_P[\tau]),
\]
whatever null class $\tau$ was designed for. If
$\E_P[\tau_\delta]\le(1+\epsilon_\delta)\log(1/\delta)/\Kstar_F$ with $\epsilon_\delta\to0$,
this is at least $\delta^{(1+\epsilon_\delta)\Kstar/\Kstar_F}$. When $\Kstar_F>\Kstar$, the
ratio to $\delta$ is
\[
\delta^{-(1-(1+\epsilon_\delta)\Kstar/\Kstar_F)}\to\infty .
\]
Finally $q^\star\in\Qn$ by construction. If also $q^\star\in\Qs_F$, then $q^\star$ is
feasible for the restricted problem and $\Kstar_F\le\KL(P\|q^\star)=\Kstar$. So
$\Kstar_F>\Kstar$ forces $q^\star\notin\Qs_F$, which is consistent with soundness on
$\Qs_F$. \hfill$\square$

\subsection{Selecting from a finite pool}

\begin{proposition}[Pool selection]
\label{prop:pool}
Let $\pi_1,\dots,\pi_K$ be candidates with exponents $\Kstar_j$. Run \cref{alg:plugin}
for each at level $\delta/K$ on the same episodes, and select the first candidate to
certify. With probability at least $1-\delta$, no candidate with $\Delta_j\le0$ is
selected. The expected stopping time is at most
$\min_j[\log(K/\delta)/((1-3\epsilon)\Kstar_j)+C_j(\epsilon,P)]$.
\end{proposition}

This follows from a union bound and \cref{thm:plugin}. A sound selection rule must
separate $P$ from the joint null $\{q\in\Qs:\E_qX_j\le0\ \forall j\}$, whose exponent
$\Kstar_{\mathrm{all}}$ is at least $\max_j\Kstar_j$. Closing the gap between
$\max_j\Kstar_j$ and $\Kstar_{\mathrm{all}}$ resembles the top-two problem in
best-arm identification \citep{degenne2019pure}, and we leave it open.

\section{Proofs for \texorpdfstring{Section~\ref{sec:local}}{local certification} (local certification)}
\label{app:cushion}

This section proves \cref{thm:cushion} (cushion parametrization) and \cref{cor:local}
(local exponent). Recall the cushion game of \cref{sec:local}. A supersolution $\mbar$
satisfies $\mbar(j)=-X(j)$ at leaves, $\mbar(h)\ge\sum_a\mu(a\mid h)\mbar(ha)$ at action
nodes and $\mbar(h)\ge\mbar(ho)$ at environment nodes. The cushion starts at
$C(\mathrm{root})=1-\lambda\mbar(\mathrm{root})$ and evolves by
\begin{align*}
C(ha)&=\bigl[C(h)+\lambda\,d(h)\bigr]\bigl(1+g(h,a)\bigr),
  &d(h)&=\mbar(h)-\textstyle\sum_{a'}\mu(a'\mid h)\,\mbar(ha')\ge0,\\
C(ho)&=C(h)+\lambda\bigl[\mbar(h)-\mbar(ho)\bigr],
\end{align*}
with payoff $Y=C(\text{leaf})$.

\subsection{Proof of \texorpdfstring{Theorem~\ref{thm:cushion}}{the cushion parametrization}}

\paragraph{Proof overview.} The proof tracks the potential $w=C+\lambda\mbar$. It is
constant across outcomes at environment nodes and moves by a fair bet at action nodes, so
at a leaf it equals one plus the payoff of a sequence of fair bets. Subtracting
$\lambda\mbar(\text{leaf})=-\lambda X$ gives the form $1+B\nu+\lambda X$. The node
constraints of (ii) are $C\ge0$, and running the argument backwards recovers a fair
fraction from any feasible bet. Part (iii) is a backward induction showing that $w\ge\lambda m$
at every node.

\paragraph{Step 1: the potential $w$.} Put
\[
w(h)=C(h)+\lambda\mbar(h),\qquad\text{so } w(\mathrm{root})=1.
\]
At an environment node, $w(ho)=C(h)+\lambda\mbar(h)=w(h)$.

\paragraph{Step 2: action nodes move $w$ by a fair bet.} At an action node,
\[
w(ha)-w(h)=[C(h)+\lambda d(h)]\,g(h,a)
  +\lambda\Bigl[\mbar(ha)-\sum_{a'}\mu(a'\mid h)\mbar(ha')\Bigr],
\]
which has $\mu$-mean zero. It is therefore a fair bet with stake
\[
\nu(h,a)=[C(h)+\lambda d(h)]g(h,a)+\lambda\mbar(ha).
\]

\paragraph{Step 3: the leaf identity.} Summing along the path,
$w(\text{leaf})=1+(B\nu)(\text{leaf})$, and since $\mbar(\text{leaf})=-X(\text{leaf})$,
\[
Y=C(\text{leaf})=w(\text{leaf})-\lambda\mbar(\text{leaf})=1+B\nu+\lambda X.
\]

\paragraph{Step 4: nonnegativity, which proves (i).} The cushion stays nonnegative:
\begin{itemize}
\item $C(\mathrm{root})\ge0$ because $\lambda\le1/\mbar(\mathrm{root})$;
\item the action update multiplies a nonnegative quantity by $1+g\ge0$;
\item the environment update adds $\lambda[\mbar(h)-\mbar(ho)]\ge0$.
\end{itemize}
This proves (i). It also shows that $w(h)-\lambda\mbar(h)=C(h)\ge0$ at every node, which
are the node constraints of (ii).

\paragraph{Step 5: the converse, which proves (ii).} Suppose $(\lambda,\nu)$ satisfies
$w(h)\ge\lambda\mbar(h)$ at every node, and set $C=w-\lambda\mbar\ge0$. At an action node,
\[
\sum_a\mu(a\mid h)C(ha)=w(h)-\lambda\sum_a\mu(a\mid h)\mbar(ha)=C(h)+\lambda d(h).
\]
If this is positive,
\[
g(h,a)=\frac{C(ha)}{C(h)+\lambda d(h)}-1
\]
is a fair fraction with $g\ge-1$. If it is zero, every $C(ha)=0$ (because $\mu>0$) and any
fair fraction works. Environment nodes satisfy the update because $w$ is constant across
outcomes. This proves (ii).

\paragraph{Step 6: exactness with $\mbar=m$, which proves (iii).} Take a feasible
$(\lambda,\nu)$ and argue backward from the leaves:
\begin{align*}
\text{leaves:}\quad& w+\lambda X\ge0,\ \text{that is},\ w\ge\lambda m;\\
\text{environment nodes:}\quad& w(h)=w(ho)\ge\lambda m(ho)\ \text{for every }o,\\
&\quad\text{so}\ w(h)\ge\lambda\max_om(ho)=\lambda m(h);\\
\text{action nodes:}\quad& w(h)=\textstyle\sum_a\mu(a\mid h)w(ha)\ge\lambda m(h).
\end{align*}
\hfill$\square$

\subsection{Proof of \texorpdfstring{Corollary~\ref{cor:local}}{the local exponent corollary}}

\begin{proof}[Proof of \cref{cor:local}]
By (ii), $\Kloc(\mbar)$ is the program of \cref{thm:duality} with additional constraints, so
$\Kloc(\mbar)\le\Kstar$, and by (iii) $\Kloc(m)=\Kstar$. The scalar bet has $\nu=0$, so
$w\equiv1$. Its feasibility requires $\lambda\max_j(-X_j)\le1$. If
$\mbar(h)\le\max_j(-X_j)$ at every non-leaf node, then $\lambda\mbar(h)\le1=w(h)$, so the
node constraints hold.
\end{proof}

\clearpage
\begin{center}
{\Large\bfseries Part II: Additional experiments and implementation details}
\end{center}
\medskip

Part II collects the experiments that support the main text but are not needed to
follow it: the structural-assumption experiment for \cref{thm:structure}
(\cref{app:structure-exp}), wall-clock cost and the matched scalable instances
(\cref{app:matched}), the random-tree study (\cref{app:random}), and the implementation
details of all certifiers and comparators (\cref{app:details}).

\section{Experiment: pricing a structural assumption}
\label{app:structure-exp}

Let $F$ be the assumption that distractor actions in the memory task do not affect later
outcomes (\cref{lem:irrelevance}). \Cref{tab:structure} shows both sides of
\cref{thm:structure}.

\paragraph{Value.} $F$ restores the $d=0$ exponent at every delay. The gain
$\Kstar_F/\Kstar=2.39$, $6.03$, $13.35$ is visible in the stopping times: the restricted
certifier's stopping time is flat in $d$, while the universal one doubles with each step.

\paragraph{Risk.} Under $q^\star$, which violates $F$, the restricted certifier separates
$q^\star$ from the restricted null at exponent $0.21$--$0.27$ and certifies falsely in
every run. Under the restricted least-favourable law $q^\star_F$, neither certifier ever
certified. \Cref{thm:structure}(ii) predicts this failure for any certifier this fast, not
only for ours. The restricted certifier's mean stopping time under $P$ is $36.9$, $35.8$
and $36.3$ episodes. So every certifier with these means must certify falsely under
$q^\star$ with probability at least $\exp(-\Kstar\E_P[\tau])=0.014$, $0.19$ and $0.47$,
that is $14$--$470$ times $\delta$. The observed rate of $100\%$ respects the floor. The
floor rises with $d$ along with the gain $\Kstar_F/\Kstar$: the larger the speed-up, the
higher the false-certification rate at $q^\star$ must be.

\begin{table}[t]
\centering
\caption{Value and risk of the assumption $F$ (distractor actions are irrelevant),
$\delta=10^{-3}$. Median episodes with 95\% intervals; ``false'' is the false certification
rate. $\le3.6\%$ is the Clopper--Pearson limit for 0 of 100. In parentheses: the floor
$\exp(-\Kstar\E_P[\tau])$ of \cref{thm:structure}(ii), evaluated at the restricted
certifier's mean stopping time under $P$, below which no certifier that fast can keep
its false-certification rate under $q^\star$.}
\label{tab:structure}
\small
\setlength{\tabcolsep}{4pt}
\begin{tabular}{rrrrlll}
\toprule
 & & & & \multicolumn{2}{c}{under $P$ (alternative)} & under $q^\star$ (violates $F$)\\
\cmidrule(lr){5-6}\cmidrule(lr){7-7}
$d$ & $\Kstar$ & $\Kstar_F$ & gain & universal & restricted & restricted\\
\midrule
1 & 0.1155 & 0.2763 & 2.39 & 82.5 [70, 95] & 35 [32, 38] & 61.5 [54, 68], 100\% false ($\ge1.4\%$)\\
2 & 0.0458 & 0.2763 & 6.03 & 203 [184, 229] & 35 [33, 37] & 40.5 [36, 44], 100\% false ($\ge19\%$)\\
3 & 0.0207 & 0.2763 & 13.35 & 417.5 [374, 486] & 35.5 [34, 38] & 39 [35, 42], 100\% false ($\ge47\%$)\\
\bottomrule
\end{tabular}
\end{table}

\section{Additional cost and scalable results}
\label{app:matched}

\subsection{Wall-clock cost}

\begin{table}[t]
\centering
\caption{Wall-clock seconds per run at $\delta=10^{-3}$ (median of 20 runs, one CPU
thread). In parentheses: median number of leaf-constrained refits. A run lasts until it
certifies or reaches the horizon of \cref{tab:comparators}. Grid: the slowest of LS, EB
and HCOPE.}
\label{tab:cost}
\small
\setlength{\tabcolsep}{4pt}
\begin{tabular}{lrrrrrrr}
\toprule
instance & CV-Kelly & hedged & scalar & mixture & DR fit & DR orc & grid\\
\midrule
cue $d{=}0$ $p{=}0.2$ & 0.91 (18) & 1.34 & 0.23 & 0.25 & 0.37 & 0.49 & 0.004\\
cue $d{=}0$ $p{=}0.02$ & 0.92 (20.5) & 1.58 & 0.44 & 0.33 & 0.52 & 0.30 & 0.005\\
cue $d{=}1$ $p{=}0.2$ & 1.62 (24.5) & 2.09 & 0.12 & 0.37 & 0.71 & 0.66 & 0.005\\
cue $d{=}1$ $p{=}0.02$ & 1.67 (25) & 1.77 & 0.09 & 0.36 & 0.55 & 0.22 & 0.006\\
cue $d{=}2$ $p{=}0.2$ & 2.30 (32) & 2.67 & 0.10 & 0.59 & 0.40 & 0.48 & 0.006\\
cue $d{=}2$ $p{=}0.02$ & 2.26 (32.5) & 2.99 & 0.12 & 0.59 & 0.40 & 0.20 & 0.007\\
random, depth 2 & 2.02 (50) & 7.33 & 0.20 & 5.34 & 0.33 & 0.20 & 0.14\\
\bottomrule
\end{tabular}
\end{table}

\Cref{tab:scaling} shows how one refit, on an empirical law of 4096 episodes, grows with
the tree. Up to 512 leaves a refit takes at most $1.9$\,s, and $\log(\text{seconds})$ grows
with slope $1.3$ in $\log(\text{leaves})$. Beyond that the dense Newton systems
dominate: each step forms and solves an $r\times r$ system with $r=\rank(B)$, at cost
$O(|\J|r^2+r^3)$. A refit takes $6.9$\,s at 768 leaves and $210$\,s at 2048 leaves, where
$r=852$. Since refits are geometric, a run of $T$ episodes costs at
most $1+\log_{1.15}T$ such solves. Even so, exact certification is practical only up to a
few thousand leaves. This is the regime boundary at which \cref{sec:local} replaces the
global program with the cushion certifier, whose per-episode cost does not depend on the
size of the tree. On every cue and DelayedCue tree in \cref{tab:scaling}, its closed-form
supersolution loses nothing: $\Kloc(\mbar)=\Kstar$ to six digits.

\begin{table}[t]
\centering
\caption{Cost of one plug-in refit (logged-law projection plus Kelly program on an
empirical law of 4096 episodes; median of 3) and of computing $\Kstar$ and $\Kloc(\mbar)$
once on the true law, in seconds, with the refit's peak traced memory.}
\label{tab:scaling}
\small
\setlength{\tabcolsep}{4pt}
\begin{tabular}{lrrrrrr}
\toprule
tree & leaves & $\rank(B)$ & refit & MiB & $\Kstar$ & $\Kloc$\\
\midrule
cue $d{=}2$ & 48 & 22 & 0.026 & 0.0 & 0.010 & 0.015\\
cue $d{=}4$ & 192 & 94 & 0.10 & 0.6 & 0.052 & 0.11\\
cue $d{=}6$ & 768 & 382 & 6.9 & 9.1 & 3.3 & 7.4\\
random, depth 4 & 512 & 170 & 1.6 & 2.9 & 1.1 & --\\
DelayedCue $K{=}2$, $L{=}2$ & 128 & 52 & 0.017 & 0.2 & 0.022 & 0.090\\
DelayedCue $K{=}2$, $L{=}3$ & 512 & 212 & 1.9 & 3.9 & 0.32 & 2.6\\
DelayedCue $K{=}4$, $L{=}1$ & 512 & 240 & 1.1 & 3.7 & 0.41 & 1.9\\
DelayedCue $K{=}2$, $L{=}4$ & 2048 & 852 & 210 & 62 & 8.2 & 94\\
\bottomrule
\end{tabular}
\end{table}

\subsection{Scalable certification and matched instances}

\begin{figure}[t]
\centering
\includegraphics{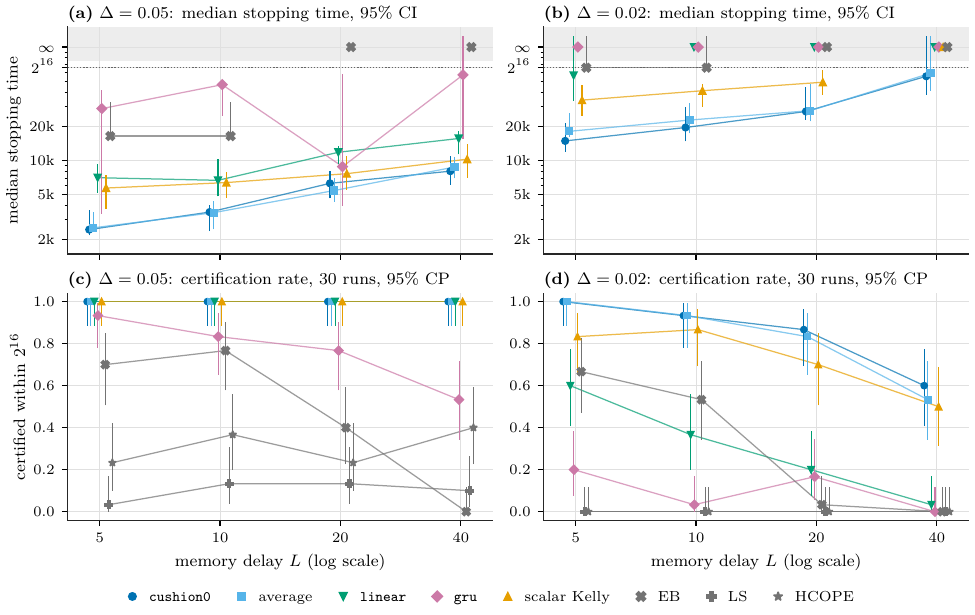}
\caption{Cushion certifier on DelayedCue at $\delta=10^{-3}$ with a
horizon of $2^{16}$ episodes and 30 runs per cell (the data of \cref{tab:scalable}).
(a,\,b)~Median stopping time with 95\% order-statistic intervals; points in the shaded
band have a censored median. (c,\,d)~Fraction of runs certified within $2^{16}$ episodes,
with 95\% Clopper--Pearson intervals. Points are offset horizontally for legibility.}
\label{fig:scalable}
\end{figure}

To see where the scalable certifier loses against $\Kstar$, we enumerate DelayedCue at
$\Delta=0.05$ for $(K,L)=(2,1),(2,2),(2,3),(4,1)$, which gives 32--512 leaves. We run the
same online configuration as \cref{tab:scalable} (30 runs, horizon $2^{16}$, on CPU)
together with CV-Kelly on the enumerated tree. Let $K_{\mathrm{class}}$ be the best growth
rate in a gambler's class, fitted on the exact law of the tree. The stopping time, in
units of the lower bound, then factors as
\[
\frac{T\Kstar}{\log(1/\delta)}=\underbrace{\frac{\Kstar}{\Kloc(\mbar)}}_{\text{supersolution}}
\cdot\underbrace{\frac{\Kloc(\mbar)}{K_{\mathrm{class}}}}_{\text{gambler class}}
\cdot\underbrace{\frac{TK_{\mathrm{class}}}{\log(1/\delta)}}_{\text{online learning}}.
\]
\Cref{tab:matched} reports each factor.

\paragraph{Supersolution factor.} It is $1$ to six digits on all four trees, so the
closed-form ratio supersolution of the scalable certifier is lossless here. The whole gap
comes from the gamblers.

\paragraph{Gambler factors.} The class factor is $1.33$--$1.48$ for \texttt{cushion0} and
$1.24$--$1.42$ for \texttt{linear}. The online factor is $1.3$--$1.6$ for \texttt{cushion0}
but $2.9$--$3.2$ for \texttt{linear}. This explains the held-out observation of \cref{tab:scalable}:
\texttt{linear}'s richer class has the higher growth rate, but learning it online costs
more than the class gains.

\paragraph{Comparison with the exact plug-in.} \texttt{cushion0} stops at
$1.7$--$2.2\times$ the lower bound and CV-Kelly at $1.6$--$3.2\times$. Their 95\%
intervals overlap on every tree. At $(2,3)$ the point estimate favours \texttt{cushion0}
($1172$ against $1782$ episodes), because the plug-in must estimate $212$ fair-bet
coordinates. In wall clock, a CV-Kelly run takes $2.7$\,s at 32 leaves and
$176$--$231$\,s at 512 leaves. The scalable runner takes $40$--$100$\,s per run, for all
gamblers and baselines together over the full $2^{16}$ episodes.

\begin{table}[t]
\centering
\caption{Matched DelayedCue instances at $\Delta=0.05$, $\delta=10^{-3}$, 30 runs.
Exponent ratios are computed on the enumerated tree. Median stopping times are in
units of the lower bound $\log(1/\delta)/\Kstar$ of \cref{thm:lower} (shown in episodes).
Online factor: median stopping time times $K_{\mathrm{class}}/\log(1/\delta)$.}
\label{tab:matched}
\small
\setlength{\tabcolsep}{3.5pt}
\begin{tabular}{lrrrrrrrrrr}
\toprule
& & & \multicolumn{3}{c}{exponent ratios $\Kstar/\cdot$} & \multicolumn{2}{c}{online factor} &
\multicolumn{3}{c}{median $/$ lower bound}\\
\cmidrule(lr){4-6}\cmidrule(lr){7-8}\cmidrule(lr){9-11}
$(K,L)$ & leaves & bound & $\Kloc$ & \texttt{cushion0} & \texttt{linear} & \texttt{cushion0} &
\texttt{linear} & \texttt{cushion0} & CV-Kelly & scalar\\
\midrule
$(2,1)$ & 32 & 542 & 1.000 & 1.36 & 1.32 & 1.26 & 3.07 & 1.71 & 1.56 & 6.08\\
$(2,2)$ & 128 & 552 & 1.000 & 1.33 & 1.25 & 1.63 & 3.20 & 2.17 & 1.83 & 4.99\\
$(2,3)$ & 512 & 562 & 1.000 & 1.36 & 1.24 & 1.54 & 3.20 & 2.09 & 3.17 & 5.26\\
$(4,1)$ & 512 & 1138 & 1.000 & 1.48 & 1.42 & 1.51 & 2.93 & 2.24 & 1.93 & 5.34\\
\bottomrule
\end{tabular}
\end{table}

\section{Random trees: locality and the Gaussian approximation}
\label{app:random}

These enumerated trees isolate the price of the supersolution, $\Kloc(\mbar)$ against
$\Kstar$. The matched DelayedCue study of \cref{sec:matched} also measures the loss from
learning a gambler.

\paragraph{Price of locality.} On 200 random trees, the ratio supersolution had zero price
of locality (relative error below $10^{-4}$) on 192 trees. The eight exceptions are
depth-2 trees with a large advantage ($\Delta\ge0.12$), the worst with
$\Kloc/\Kstar=0.69$. With the exact price $m$ in place of $\mbar$ the loss disappears, as
\cref{cor:local} requires, so the loss is the slack of the closed form, which binds only
when the optimal stake is large. Looser supersolutions pay a price even on memory trees:
at $d=1$ a constant $\mbar=1,1.5,3,10$ gives $\Kloc=0.1045,0.0942,0.0703,0.0314$, against
$\Kstar=0.1155$.

\paragraph{Scalar and DR bets.} The median of $\Kscal/\Kstar$ is $0.37$. The oracle DR
bet has median $K_{\mathrm{DR}}/\Kstar=0.49$, $0.28$ and $0.10$ at depths 2, 3 and 4. Its
additive range must cover every history, and that cost grows with depth.

\paragraph{Gaussian approximation.} \Cref{fig:random} compares $\Kstar$ with the
efficiency bound of \cref{thm:local} on the same trees.

\begin{figure}[t]
\centering
\includegraphics{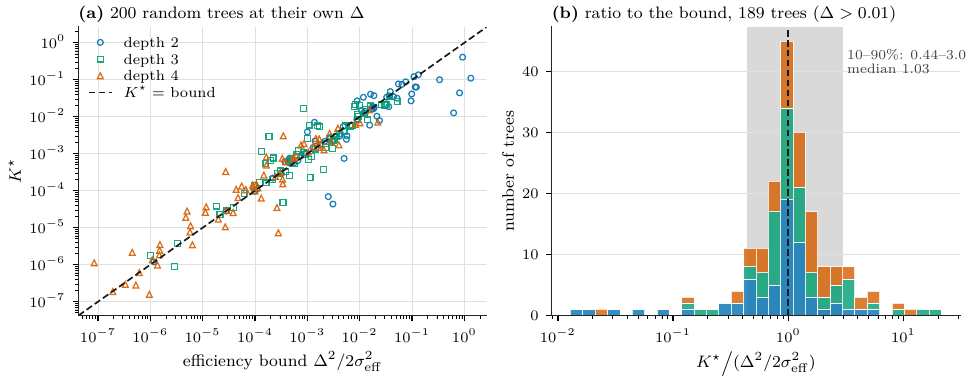}
\caption{$\Kstar$ and the semiparametric efficiency bound on 200
random trees. (a)~$\Kstar$ against $\Delta^2/2\sigeff$, each tree at its own advantage
$\Delta$. (b)~Histogram of the ratio $\Kstar/(\Delta^2/2\sigeff)$ on the 189 of the 200 trees
with $\Delta>0.01$, stacked by depth; shaded: the 10--90\% range. The filter removes trees
where the ratio is numerically unstable ($\Delta\to0$) and is unrelated to the 192 trees
with zero price of locality in the text. \Cref{thm:local} makes the
ratio tend to 1 only as $\Delta\to0$. At a fixed advantage it ranges from $0.014$ to
$19.4$, and it is below 1 on 47\% of trees.}
\label{fig:random}
\end{figure}

\section{Implementation and experimental details}
\label{app:details}

\begin{figure}[t]
\centering
\includegraphics{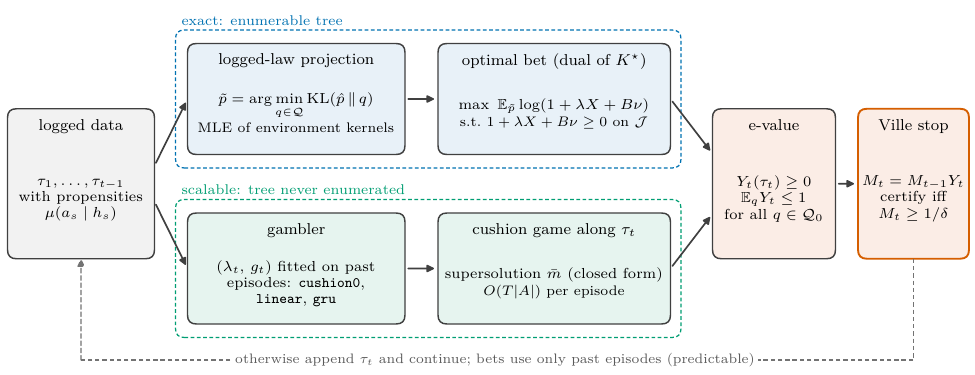}
\caption{Certification pipeline. Each lane turns the past episodes
into a predictable bet and evaluates it on the new episode $\tau_t$, giving an e-value
$Y_t$. The product $M_t$ is a nonnegative supermartingale under every null law, so
stopping once $M_t\ge1/\delta$ is $\delta$-sound by Ville's inequality. Top: for an
enumerable tree, \cref{alg:plugin} projects the empirical law onto $\Qs$ and solves the
dual program of \cref{thm:duality}. Bottom: when the tree cannot be enumerated, the
cushion game of \cref{thm:cushion} with the closed-form supersolution of
\cref{lem:ratio} evaluates the bet along the realized episode only
(\cref{prop:scalable}).}
\label{fig:architecture}
\end{figure}

\begin{table}[t]
\centering
\caption{Exact exponents on the published hard instances. $V(\pi)=3/4$ throughout.
$\KL_2=\KL(P_1\|P_0)$ is the exponent within each paper's two-model class. The ceiling
of \cref{thm:ceiling} equals $P(A)\log(3/2)$ exactly.}
\label{tab:hard}
\small
\begin{tabular}{lrrrrrrrr}
\toprule
family & $H$ & leaves & $P(A)$ & $V(\mu)$ & $\Kstar$ & $\Kstar/P(A)$ & ceiling & $\KL_2$\\
\midrule
Jajoo & 3 & 18 & 0.2222 & 0.5556 & 0.02414 & 0.1086 & 0.09010 & 0.1221\\
Jajoo & 4 & 54 & 0.1481 & 0.5370 & 0.01724 & 0.1164 & 0.06007 & 0.0814\\
Jajoo & 5 & 162 & 0.0988 & 0.5247 & 0.01198 & 0.1213 & 0.04005 & 0.0543\\
Jajoo & 6 & 486 & 0.0658 & 0.5165 & 0.00820 & 0.1245 & 0.02670 & 0.0362\\
Jajoo & 7 & 1458 & 0.0439 & 0.5110 & 0.00556 & 0.1267 & 0.01780 & 0.0241\\
Jajoo & 8 & 4374 & 0.0293 & 0.5073 & 0.00375 & 0.1280 & 0.01187 & 0.0161\\
\midrule
Luo & 3 & 48 & 0.0625 & 0.5156 & 0.00907 & 0.1451 & 0.02534 & 0.0407\\
Luo & 4 & 288 & 0.0156 & 0.5039 & 0.00231 & 0.1476 & 0.00634 & 0.0102\\
Luo & 5 & 1728 & 0.0039 & 0.5010 & 0.00058 & 0.1482 & 0.00158 & 0.0025\\
\bottomrule
\end{tabular}
\end{table}

\paragraph{Exact programs.} \Cref{thm:duality}'s program is solved in the coordinates
$(\lambda,\nu)$ by a log-barrier method on the leaf constraints, with damped Newton steps
and a backtracking line search that keeps every iterate strictly feasible. The recovered
$q^\star=P/Y^\star$ is checked to be action-neutral, to sum to one and to satisfy
$\E_{q^\star}X\le10^{-7}$, and its divergence matches the dual value to within $10^{-6}$. The
logged-law projection $\argmin_{Q\in\Qn}\KL(\hat P_n\|Q)$ solves the dual program
$\max_\nu\E_{\hat P_n}\log(1+B\nu)$ with the same barrier Newton method and recovers
$Q=\hat P_n/(1+B\nu^\star)$.
$\Kloc$ adds one zero-probability barrier row per node, and $\Kstar_F$ adds the columns of
$F$.

\paragraph{Universal mixture.} The mixture averages over 300 points of a hit-and-run walk on
$Z$, drawn before any episode is seen. A fixed mixing law preserves exact soundness.

\paragraph{Comparators.} \emph{LS} compares a logarithmic-smoothing lower bound on $V(\pi)$
from $WR$ with an LS upper bound on $V(\mu)$ from $R$, over a grid of smoothing levels.
\emph{HCOPE} clips importance weights at a level chosen on the first 20\% of the data and uses
a Maurer--Pontil empirical-Bernstein bound for $V(\pi)$ against an empirical-Bernstein upper
bound for $V(\mu)$. \emph{EB} is an empirical-Bernstein lower bound on $\E X$. All three are
evaluated on the doubling grid $n=2,4,8,\dots$ with a union bound over grid points and
parameters, which makes them anytime-valid. \emph{DR oracle} bets
$1+\lambda(X-Bc)$, with $c$ the $L^2(P)$-efficient control variate under the true law, and
$\lambda$ refit predictably by Kelly. \emph{DR fitted} refits $c$ on past episodes as well.
Every DR bet must be feasible on every leaf.

\paragraph{Null laws for soundness.} Besides $q^\star$, we use vertices of $\Qn$ obtained by
maximizing random linear objectives over $\Qn$ with an LP solver. These are adversarial
boundary laws that put zero mass on many leaves.

\paragraph{DelayedCue.} A cue $c\in\{1,\dots,4\}$ is drawn. At each of $L$ distractor steps a
symbol $s_t$ is shown uniformly at random. The logger plays $s_t$ with probability $0.4$,
repeats its previous action with probability $0.3$, and otherwise plays uniformly. The
target mixes the logger's distribution with probability $0.9$ and plays $s_t$ with
probability $0.1$. Matching the symbol pays $0.5/L$. At the final step the logger recalls the
cue with probability $0.3$ (otherwise uniform), and the target with probability $0.9$. The
final reward is Bernoulli with mean $p_{\mathrm{hit}}$ on recall and $0.5$ otherwise, with
$p_{\mathrm{hit}}$ set so that $\Delta$ takes the stated value exactly. The per-step ratio
bounds of \cref{lem:ratio} are computed exactly, and the simulator attains them. The
\texttt{linear} gambler's fair fractions are a softmax-tilt of $\mu$ by a linear function of
features including $\log(\pi/\mu)$. The \texttt{gru} gambler adds a GRU head over the
observable history. Both choose $(\lambda,\theta)$ by maximizing past mean log-wealth.

\paragraph{Compute.} Tree experiments ran on 8 CPU cores, with BLAS limited to one thread per
worker process. The saved comparison sweep (all methods and cells of
\cref{tab:comparators}, including tree construction and exponent calculations)
records 674 seconds of wall-clock time as a batch total. The scalable experiment ran on one NVIDIA T4 GPU and took 3.8 hours.
The plug-in refits at $r_k=\lceil1.15^k\rceil$ and thus performs at most
$1+\lfloor\log_{1.15}T\rfloor$ distinct scheduled refits by episode $T$, with leaf
constraints at each refit. Scalar Kelly has no fair-action variables. Per-method wall
clock, refit counts, peak memory and the growth of a refit with tree size are measured
separately on one CPU thread per worker, and reported
in \cref{sec:cost}. Seeds are derived per cell from a master seed, so every cell is
reproducible independently.

\vskip 0.2in
\bibliography{references}

\end{document}